%% file: Arxiv.tex
\documentclass[11pt]{article}
\usepackage[top=2.54cm,left=2.54cm,right=2.54cm,bottom=2.54cm]{geometry}

\input{math_commands.tex}

\usepackage{natbib}
\usepackage[hidelinks]{hyperref} 
\usepackage{url}
\usepackage{amsmath,amssymb,amsthm,mathtools}
\usepackage{algorithm}
\usepackage{algpseudocode}
\usepackage{booktabs}
\usepackage{graphicx}
\usepackage{float}
\usepackage{cleveref}
\usepackage{caption}
\usepackage{microtype}
\usepackage{placeins}

\newtheorem{theorem}{Theorem}[section]
\newtheorem{proposition}[theorem]{Proposition}
\newtheorem{corollary}[theorem]{Corollary}
\newtheorem{definition}[theorem]{Definition}
\newtheorem{remark}[theorem]{Remark}
\newtheorem{lemma}[theorem]{Lemma}

\newcommand{\cS}{\mathcal{S}}
\newcommand{\cA}{\mathcal{A}}
\newcommand{\ind}{\mathbb{I}}
\newcommand{\Mult}{\operatorname{Multinomial}}
\newcommand{\MHG}{\operatorname{MHG}}
\newcommand{\iid}{\mathrm{iid}}
\newcommand{\fc}{\mathrm{fc}}
\newcommand{\splflow}{\mathrm{split}}
\newcommand{\trueflow}{\mathrm{true}}

\title{Exact Fast Batch Simulation for Tabular Reinforcement Learning}

\author{
Haochen Zhang$^{1}$, Lingzhou Xue$^{1}$, and Zhong Zheng$^{2}$\\[0.5em]
$^{1}$Department of Statistics, The Pennsylvania State University\\
$^{2}$Department of Applied and Computational Mathematics and Statistics,\\ University of Notre Dame
}
\date{}

\begin{document}

\maketitle
\begin{abstract}
Simulation is a fundamental computational primitive in reinforcement learning (RL), yet conventional simulation explicitly generates individual trajectories even when downstream procedures use only aggregate statistics. To address this, we develop an exact fast-simulation framework for finite-horizon tabular Markov decision processes. Our framework has two complementary modes. In direct batch simulation, a batch is represented by its aggregate Markov flow. With sufficient parallel simulation resources, this flow can be obtained by trajectory aggregation; when such simulation is unavailable or costly but the initial state and transition distributions are directly accessible, we instead generate an identically distributed flow through forward Markov-flow sampling without materializing individual trajectories. The latter reduces the simulator-side computational dependence on batch size $m$ from $O(m)$ to $O(1)$. In adaptive batch simulation, when batch length is determined by a data-dependent condition, exact multivariate-hypergeometric splitting recursively refines a candidate Markov flow while preserving the conditional law, reducing the cost dependence on $m$ from $O(m)$ to $O(\log m)$. Together, these modes accelerate simulation by keeping trajectories aggregated whenever possible and refining flows only when required to locate data-dependent boundaries. The framework applies broadly across simulator-based, offline, and online batch or stage-based RL, as illustrated with representative algorithms from each setting.
\end{abstract}

\section{Introduction}
\label{sec:intro}

Reinforcement learning (RL)~\citep{1998Reinforcement} has become a useful framework for sequential decision-making, with applications ranging from games~\citep{vinyals2019grandmaster} and robotics~\citep{gu2017deep} to autonomous driving~\citep{yurtsever2020survey}. A canonical theoretical setting is the finite-horizon tabular Markov decision process (MDP), where an agent generates trajectories of $H$ consecutive state-action transitions according to a policy and the underlying transition kernel.

Simulation is a fundamental computational primitive in RL. It is used to generate training data, construct offline datasets, evaluate learned policies, and empirically study algorithmic properties such as regret scaling, sample complexity, visitation frequencies, and switching behavior. The conventional simulation interface operates at the level of individual trajectories: it samples an initial state, repeatedly selects an action and draws a next state, materializes the resulting trajectory, and then repeats this procedure episode by episode.

This episode-level interface can be unnecessarily restrictive for RL procedures that naturally operate in batches. In particular, the sampling policy often remains unchanged over multiple consecutive episodes before the learner performs its next update. We formalize this structure as follows.
\begin{definition}[Batch and batch-implementable update]
\label{def:batch}
A batch is a consecutive collection of episodes generated under the same Markov policy $\pi$. The policy may depend on the pre-batch learner state, which collects all information maintained by the learning algorithm before the batch begins, but remains unchanged within the batch. The batch endpoint may be predetermined or be the first episode at which a specified data-dependent boundary condition is met.

A learner update is batch-implementable if it can be computed exactly from the pre-batch learner state, including any accumulators retained from previous batches, and aggregate information from the current batch, without explicitly processing its episodes one by one.
\end{definition}
Many RL procedures naturally admit this form of batching. Simulator-based methods may repeatedly query the same transition distributions only to construct empirical models or statistics~\citep{gheshlaghi2013minimax,sidford2018near,li2024breaking}; in this case, the raw next-state samples are ultimately used only through empirical transition counts. Offline RL often operates on trajectories collected under a fixed behavior policy~\citep{yin2021near,yin2021towards,shi2022pessimistic}; here, many downstream computations depend only on aggregate occupancy, visitation, and transition statistics. Batch or stage-based online algorithms may retain the same decision rule while accumulating samples until an update is triggered~\citep{bai2019provably,zhang2020almost,qiao2022sample,zhang2022near,zheng2023federated,zheng2024federated,zhang2025regret}; between updates, the trajectories are often used only to update visit counts, empirical moments, or other batch-wise accumulators. In all three settings, explicitly generating every individual sample before forming the required aggregate summaries can introduce unnecessary simulator-side computation.

This issue becomes particularly important when RL simulation is performed at scale. In theoretical studies, large numbers of episodes, repeated random seeds, multiple MDP instances, and extensive parameter sweeps are often required to reveal finite-sample scaling laws or distinguish closely related algorithmic behaviors. For example, recent tabular RL experiments \citep{zhang2025gapdependent} simulate millions of episodes in order to examine regret and switching behavior. Large simulation workloads also arise in practical RL systems: modern simulators such as Isaac Gym train policies using thousands of parallel environments and tens to hundreds of millions of environment steps~\citep{makoviychuk2021isaac}. To handle such workloads, systems including Isaac Gym~\citep{makoviychuk2021isaac}, Brax~\citep{freeman2021brax}, and EnvPool~\citep{weng2022envpool} improve throughput by parallelizing conventional environment stepping. These approaches substantially accelerate trajectory-level simulation, but still explicitly generate individual transitions. This motivates a complementary direction for finite-horizon tabular MDPs:
\begin{center}
\emph{Can we exploit aggregate batch structure to accelerate simulation while exactly preserving the distribution needed by the learning procedure?}
\end{center}
We answer this question affirmatively through two simulation modes, depending on how the batch endpoint is determined. When the endpoint is known in advance, we use \emph{direct batch simulation}. A predetermined batch is represented by its \emph{Markov-flow counts}: state counts, state-action counts, and state-action-next-state transition counts. There are two ways to obtain this flow. When sufficient parallel simulation resources are available, the trajectories can be generated concurrently and then aggregated into the Markov flow. When such parallel simulation is unavailable or costly, but the initial state and transition distributions are accessible, we instead generate an identically distributed Markov flow directly, without materializing individual trajectories, using forward Markov-flow sampling. This reduces the simulator-side computational dependence on the batch size from $O(m)$ to $O(1)$.

When the batch endpoint is determined by a data-dependent boundary condition and is therefore unknown in advance, we use \emph{adaptive batch simulation}. This situation arises in online algorithms whose next update is triggered, for example, when an accumulated visit count reaches a stage boundary or a reference statistic is refreshed. Under the current policy, we first simulate a large candidate batch in aggregate form and then use exact multivariate-hypergeometric splitting to recursively locate the first boundary episode. The aggregate flow up to this boundary is retained, the corresponding learner updates are applied, and simulation restarts from the updated learner state. Because each split preserves the correct conditional law, the resulting procedure is distributionally identical to conventional sequential simulation, while locating the boundary within a candidate batch of $m$ episodes requires only $O(\log m)$ refinement levels instead of $O(m)$ episode-wise inspections.

The resulting principle is simple: \emph{represent trajectory batches by aggregate Markov flows whenever possible, and apply exact hypergeometric splitting only when a data-dependent boundary must be located}. Although our main results are stated for finite-horizon tabular MDPs, the construction is not tied to episodic dynamics and also applies to independent finite-length rollouts of tabular MDPs.

\textbf{Our contributions.}
Our main contributions are summarized as follows.

\textbf{Exact direct batch simulation.}
We formulate the Markov-flow representation of a batch in Section~\ref{sec:direct}. When sufficient parallel simulation resources are available, the batch can be generated concurrently and aggregated directly into its Markov flow. When such parallel simulation is unavailable or costly, but the initial state and transition distributions are accessible, we instead show that the aggregate state, state-action, and transition counts of $m$ episodes can be sampled directly through forward Markov-flow sampling without materializing individual trajectories, while exactly preserving the joint flow distribution of conventional episode-by-episode simulation (\Cref{thm:forward}). Any flow-measurable batch accumulator can then be computed from the resulting aggregate counts. Under the computational convention of Section~\ref{sec:complexity}, this direct flow-sampling implementation reduces the simulator-side dependence on $m$ from $O(m)$ to $O(1)$ (see Proposition~\ref{prop:direct-complexity}).

\textbf{Exact adaptive batch simulation.} For a prefix-monotone data-dependent boundary, we develop a multivariate-hypergeometric splitting procedure in \Cref{sec:adaptive} that samples exact conditional sub-flows of a generated aggregate flow. Recursive binary splitting locates the first boundary in at most $\lceil\log_2 m\rceil$ refinements and recovers the aggregate flow up to that boundary (\Cref{thm:adaptive}). Under the boundary compatibility condition specified in Section~\ref{sec:problem}, restarting after each update exactly reproduces the joint distribution of the post-boundary learner states. When candidate flows are generated by forward Markov-flow sampling, the simulator-side dependence on $m$ is reduced from $O(m)$ to $O(\log m)$ (see Proposition~\ref{prop:adaptive-complexity} for details).

\textbf{Applications across representative RL settings.}
In \Cref{sec:examples}, we demonstrate the framework in simulator-based, offline, and online RL through model-based planning with a generative model~\citep{li2024breaking}, VPVI~\citep{yin2021towards}, and stage-based UCB-Advantage~\citep{zhang2020almost}, respectively. Model-based planning uses direct transition-count sampling, VPVI uses direct Markov-flow generation for a fixed offline dataset, and UCB-Advantage uses adaptive batch simulation to locate data-dependent update boundaries. Our experiments further demonstrate the practical computational efficiency of our framework compared with conventional episode-by-episode simulation.

\textbf{Technical novelty.} A key technical ingredient of our analysis is a conditional labeling representation of an aggregate Markov flow. Conditional on the full flow, all compatible labeled trajectory sequences are uniformly distributed and admit a sequential combinatorial representation (Lemmas \ref{lem:conditional-uniformity} and \ref{lem:conditional-labeling}). Building on this representation, we derive multivariate hypergeometric decompositions that recover the exact conditional law of a contiguous subbatch despite dependencies across successive MDP steps (Lemmas \ref{lem:initial-split}--\ref{lem:split-equivalence}). We further show that this conditional structure is preserved under recursive refinement, enabling repeated splitting and establishing distributional equivalence with sequential simulation across successive data-dependent boundaries.

\section{Problem Formulation}
\label{sec:problem}
\textbf{Finite-horizon tabular Markov decision process.}
We consider a finite-horizon tabular Markov decision process (MDP) $\mathcal{M}=(\cS,\cA,H,P,r,\nu)$, where $\cS$ is the state space with $|\cS|=S$, $\cA$ is the action space with $|\cA|=A$, and $H$ is the horizon. The transition kernel is $P=\{P_h\}_{h=1}^H$, where $P_h(\cdot\mid s,a)$ specifies the distribution of the next state given $(s,a)$ at step $h$. The reward function is $r=\{r_h\}_{h=1}^H$, where each $r_h:\cS\times\cA\rightarrow[0,1]$ is deterministic, and $\nu\in\Delta(\cS)$ denotes the initial-state distribution.

At the beginning of each episode, $X_1\sim\nu$. At step $h\in[H]$, the agent observes $X_h\in\cS$, selects $A_h\in\cA$, receives $r_h(X_h,A_h)$, and transitions according to $X_{h+1}\sim P_h(\cdot\mid X_h,A_h)$. A Markov policy $\pi=\{\pi_h\}_{h=1}^H$ satisfies $\pi_h(\cdot\mid s)\in\Delta(\cA)$. Whenever a policy is defined greedily from Q-values, ties are broken according to a fixed deterministic rule. For a fixed policy $\pi$, one episode generates $\omega=(X_1,A_1,\ldots,X_H,A_H,X_{H+1})$, and for any realization $\omega=(x_1,a_1,\ldots,x_H,a_H,x_{H+1})$,
\begin{equation}
\label{eq:traj-prob}
\mathbb{P}_{\pi}(\omega)=\nu(x_1)\prod\nolimits_{h=1}^H\pi_h(a_h\mid x_h)P_h(x_{h+1}\mid x_h,a_h).
\end{equation}
\textbf{Markov-flow counts.}
Fix the learner state $Z$ at the beginning of a batch as in Definition \ref{def:batch}, and let $\pi=\pi(Z)$ be the policy used throughout the batch. Conditional on $Z$, consider $m$ labeled trajectories $\omega_1,\ldots,\omega_m\stackrel{\mathrm{iid}}{\sim}\mathbb{P}_{\pi}$. Writing $\omega_i=(X_{i,1},A_{i,1},\ldots,X_{i,H},A_{i,H},X_{i,H+1})$, define
\begin{align}
N_h(s)&=\sum\nolimits_{i=1}^m\ind\{X_{i,h}=s\},\qquad
C_h(s,a)=\sum\nolimits_{i=1}^m\ind\{X_{i,h}=s,A_{i,h}=a\},\notag\\
E_h(s,a,s')&=\sum\nolimits_{i=1}^m\ind\{X_{i,h}=s,A_{i,h}=a,X_{i,h+1}=s'\}.\nonumber
\end{align}
We call $F_m=(N_1,C,E)$ the \emph{Markov flow} of the batch, where $C=\{C_h(s,a)\}_{h,s,a}$ and $E=\{E_h(s,a,s')\}_{h,s,a,s'}$. For two Markov flows $F=(N_1,C,E)$ and $F'=(N_1',C',E')$, addition and subtraction are defined as $F\pm F'=(N_1\pm N_1',\,C\pm C',\,E\pm E')$, where $(N_1\pm N_1')(s)=N_1(s)\pm N_1'(s),(C\pm C')_h(s,a)=C_h(s,a)\pm C_h'(s,a), (E\pm E')_h(s,a,s')=E_h(s,a,s')\pm E_h'(s,a,s') $.

\textbf{Flow-measurable accumulators.}
Given the pre-batch learner state $Z$, we call a batch statistic or accumulator \emph{flow measurable} if it can be written as a deterministic measurable function $\phi(Z,F_m)$. Examples include state-action visitation counts, empirical transition kernels, and sums of a fixed function of the next state, such as
\begin{equation}
\label{eq:flow-function}
\sum\nolimits_{i=1}^m\ind\{X_{i,h}=s,A_{i,h}=a\}f_{h+1}(X_{i,h+1})
=\sum\nolimits_{s'}E_h(s,a,s')f_{h+1}(s'),
\end{equation}
whenever $f_{h+1}$ is fixed by $Z$ throughout the batch. Thus, whenever the aggregate information required by a batch-implementable update is flow measurable, the update can be carried out directly from $F_m$ without reconstructing individual trajectories.

\textbf{Adaptive boundaries and boundary compatibility condition.}
For adaptive simulation, let $G_j$ denote the Markov flow of the first $j$ episodes generated under the current policy. A trigger $g(Z,G_j)\in\{0,1\}$ is \emph{prefix monotone} if $g(Z,G_j)=1$ implies $g(Z,G_{j'})=1$ for every $j'\ge j$. Cumulative threshold rules, including staggered visit-count or stage boundaries associated with different state-action-step triples, satisfy this property. Define the first boundary by
\begin{equation}
\label{eq:tau}
\tau=\inf\{j\ge1:g(Z,G_j)=1\},
\end{equation}
with $\inf\emptyset=\infty$. Once the boundary is identified, adaptive simulation retains the aggregate flow $G_\tau$ of all episodes up to the boundary. We require that all learner updates triggered at this boundary can be computed exactly from $(Z,G_\tau)$ and that postponing these updates until $G_\tau$ has been recovered does not alter any subsequent action selection or accumulator contribution within the boundary episode. We refer to these requirements collectively as the \emph{boundary compatibility condition}. After applying the triggered updates, simulation restarts from the resulting learner state.

\section{Exact Fast Simulation Framework}
\label{sec:method}
\subsection{Direct Batch Simulation}
\label{sec:direct}
We first consider a predetermined batch of size $m$, with the goal of obtaining its Markov flow $F_m$. The same aggregate-count perspective also applies to repeated generative-model queries. We consider two complementary ways of obtaining $F_m$.

\textbf{Parallel aggregation.} When sufficient parallel simulation resources are available, the $m$ episodes of the batch can be generated concurrently and aggregated into $F_m$.

\textbf{Forward Markov-flow sampling.} When such parallel simulation is unavailable or costly, but the transition kernels and initial-state distribution are directly accessible, we can instead generate an identically distributed flow without materializing individual trajectories using the forward Markov-flow sampling procedure developed below.

Sample the initial-state counts according to
\begin{equation}
\label{eq:forward1}
N_1(\cdot)\sim\Mult(m,\nu).
\end{equation}
Then, perform the sampling and propagation steps in~\Cref{eq:forward2,eq:forward3,eq:forward4} for $h=1,\ldots,H$ to produce the Markov flow $F_m=(N_1,C,E)$.

Conditional on $N_h$, sample
\begin{equation}
\label{eq:forward2}
C_h(s,\cdot)\mid N_h\sim\Mult\left(N_h(s),\pi_h(\cdot\mid s)\right),\qquad s\in\cS,
\end{equation}
independently across $s$. Conditional on $C_h$, sample
\begin{equation}
\label{eq:forward3}
E_h(s,a,\cdot)\mid C_h\sim\Mult\left(C_h(s,a),P_h(\cdot\mid s,a)\right),\qquad (s,a)\in\cS\times\cA,
\end{equation}
independently across $(s,a)$, and propagate
\begin{equation}
\label{eq:forward4}
N_{h+1}(s')=\sum\nolimits_{s,a}E_h(s,a,s'),\qquad s'\in\cS.
\end{equation}
\begin{theorem}[Exactness of direct Markov-flow sampling]
\label{thm:forward}
Let $F_m^{\iid}=(N_1^{\iid},C^{\iid},E^{\iid})$ be the Markov flow obtained from $m$ independent trajectories $\omega_1,\ldots,\omega_m\stackrel{\mathrm{iid}}{\sim}\mathbb{P}_{\pi}$. Let $F_m^{\fc}=(N_1^{\fc},C^{\fc},E^{\fc})$ be generated by~\Cref{eq:forward1,eq:forward2,eq:forward3,eq:forward4}. Then $F_m^{\fc}\stackrel{d}{=}F_m^{\iid}$.

Equivalently, the joint law of $(N_1,C_1,E_1,\ldots,C_H,E_H)$ is identical under direct Markov-flow sampling and conventional episode-by-episode simulation.
\end{theorem}
A complete proof is deferred to Appendix~\ref{app:proof-forward}. The exactness of the Markov flow immediately extends to any computation that depends on the batch only through its flow.
\begin{corollary}
\label{cor:flow}
Fix the pre-batch learner state $Z$. For any measurable map $\phi$ of $Z$ and the Markov flow, $\phi(Z,F_m^{\fc})\stackrel{d}{=}\phi(Z,F_m^{\iid}).$ Therefore, any flow-measurable accumulator or batch-implementable update that depends on the batch only through $F_m$ is distributionally identical under direct batch simulation and conventional episode-by-episode simulation.
\end{corollary}

\subsection{Adaptive Batch Simulation}
\label{sec:adaptive}

When the batch endpoint is determined by a data-dependent boundary, its location is unknown in advance. We therefore first obtain the Markov flow of a fixed-size candidate batch under the current policy using direct batch simulation and, if the candidate batch triggers the boundary, recursively refine its flow through exact multivariate-hypergeometric splitting to locate the earliest triggering episode.

\textbf{Exact multivariate-hypergeometric splitting.}
Let a parent flow $F_m=(N_1,C,E)$ represent $m$ trajectories, and choose any $m_L\in\{0,\ldots,m\}$ with $m_R=m-m_L$. For a count vector $b$, let $Y\sim\MHG(b,n)$ denote a multivariate hypergeometric random vector obtained by sampling $n$ objects without replacement from categories with counts $b$.

Define the left child flow recursively by first sampling
\begin{equation}
N_1^L(\cdot)
\sim \MHG\!\left(N_1(\cdot),m_L\right).
\label{eq:split1}
\end{equation}
Then, repeat the following procedure for $h=1,\ldots,H$. Conditional on $N_h^L$, sample
\begin{equation}
C_h^L(s,\cdot)
\sim \MHG\!\left(C_h(s,\cdot),N_h^L(s)\right),
\qquad s\in\cS,
\label{eq:split2}
\end{equation}
independently across $s$ and, conditional on $C_h^L$, sample
\begin{equation}
E_h^L(s,a,\cdot)
\sim \MHG\!\left(E_h(s,a,\cdot),C_h^L(s,a)\right),
\qquad (s,a)\in\cS\times\cA,
\label{eq:split3}
\end{equation}
independently across $(s,a)$. Finally, propagate
\begin{equation*}
N_{h+1}^L(s')=\sum\nolimits_{s,a}E_h^L(s,a,s'),\qquad s'\in\cS.
\end{equation*}
The right child flow is obtained by subtraction, $F_{m_R}^R=F_m-F_{m_L}^L$.

\begin{theorem}[Conditional correctness of a split]
\label{thm:split}
Let $\omega_1,\ldots,\omega_m\stackrel{\iid}{\sim}\mathbb{P}_\pi$, and let $F_m$ be their aggregate Markov flow. Fix any $m_L\in\{0,\ldots,m\}$ and let $m_R=m-m_L$. Denote by $F_{m_L}^{\trueflow,L}$ and $F_{m_R}^{\trueflow,R}$ the Markov flows of the first $m_L$ and remaining $m_R$ trajectories, respectively, under conventional episode-by-episode simulation. Let $F_{m_L}^{\splflow,L}$ be the left child generated by~\Cref{eq:split1,eq:split2,eq:split3}, and define the right child by $F_{m_R}^{\splflow,R}=F_m-F_{m_L}^{\splflow,L}.$ Then, conditional on $F_m$,
\begin{equation*}
\left(F_{m_L}^{\splflow,L},F_{m_R}^{\splflow,R}\right)\mid F_m
\stackrel{d}{=}
\left(F_{m_L}^{\trueflow,L},F_{m_R}^{\trueflow,R}\right)\mid F_m.
\end{equation*}
In particular, each child has the correct conditional distribution corresponding to its true sub-batch.
\end{theorem}
A complete proof is deferred to Appendix~\ref{app:proof-split}.

\begin{remark}[Deterministic-policy specialization]
\label{rem:deterministic-policy}
Suppose $\pi$ is deterministic, and write $\pi_h(s)\in\cA$ for the unique action selected at state $s$ and step $h$. Then~\Cref{eq:forward2} is deterministic: $C_h(s,\pi_h(s))=N_h(s)$ and $C_h(s,a)=0$ for $a\neq\pi_h(s)$, so~\Cref{eq:forward3} need only be sampled for the active row $(s,\pi_h(s))$. Likewise, in multivariate-hypergeometric splitting,~\Cref{eq:split2} is deterministic and~\Cref{eq:split3} need only be sampled for the active action. In these deterministic-policy specializations, we represent the inactive action rows implicitly as zero rather than materializing them. This changes neither the Markov flow represented nor the exactness statements above.
\end{remark}

\textbf{Binary search for the first boundary.}
For adaptive simulation, we specialize the general splitting rule above to approximately equal child sizes. Algorithm~\ref{alg:fastforward} maintains an accepted prefix flow $F^{\mathrm{pre}}$ that is known not to have triggered the boundary. At each iteration, the current batch is split at its midpoint. If appending the left child triggers $g$, the first boundary lies in the left child; otherwise, the left child is accepted, and the search continues in the right child. Thus, the unresolved batch size decreases by approximately one half at each iteration. The exactness of \Cref{alg:fastforward} is stated in the following theorem and proved in Appendix~\ref{app:proof-adaptive}.
\begin{theorem}[Exact adaptive batch simulation]
\label{thm:adaptive}
Fix $Z$ and suppose the policy remains equal to $\pi(Z)$ until the first boundary defined in~\Cref{eq:tau}. Let $\widetilde G$ denote the flow returned by Algorithm~\ref{alg:fastforward}. Then
\[
\bigl(g(Z,F_m),\widetilde G\bigr)
\overset{d}{=}
\bigl(\mathbb{I}\{\tau\le m\},G_{\tau\wedge m}\bigr).
\]
Moreover, the algorithm uses at most $\lceil\log_2 m\rceil$ conditional splits. If the boundary compatibility condition holds, repeatedly applying the procedure at successive boundaries produces the same joint distribution over post-boundary learner states as conventional simulation.
\end{theorem}
\begin{algorithm}[H]
\caption{Exact adaptive batch simulation by binary refinement}
\label{alg:fastforward}
\begin{algorithmic}[1]
\Require Learner state $Z$, candidate size $m$, policy $\pi(Z)$, prefix-monotone trigger $g$
\State Sample candidate flow $F_m$ by direct batch simulation
\If{$g(Z,F_m)=0$}
    \State \Return $F_m$
\EndIf
\State $(F,n)\gets(F_m,m)$ and $F^{\mathrm{pre}}\gets0$
\While{$n>1$}
    \State $n_L\gets\lfloor n/2\rfloor$ and $n_R\gets n-n_L$
    \State Split $(F,n)$ into exact prefix and suffix flows $(F^L,n_L)$ and $(F^R,n_R)$
    \If{$g(Z,F^{\mathrm{pre}}+F^L)=1$}
        \State $(F,n)\gets(F^L,n_L)$
    \Else
        \State $F^{\mathrm{pre}}\gets F^{\mathrm{pre}}+F^L$, $(F,n)\gets(F^R,n_R)$
    \EndIf
\EndWhile
\State \Return boundary flow $\widetilde G_\tau=F^{\mathrm{pre}}+F$
\end{algorithmic}
\end{algorithm}

\section{Computational Cost}
\label{sec:complexity}

The results in~\Cref{thm:forward,thm:split,thm:adaptive} establish that the proposed procedures exactly preserve the distribution induced by conventional episode-by-episode simulation. We now compare the simulator-side computational cost required to generate these equivalent outputs. For direct batch simulation, our cost analysis focuses on the forward Markov-flow sampling implementation.

\textbf{Computational convention.}
Throughout this section, we assume that basic arithmetic on the counts, sampling from binomial and hypergeometric distributions, and sampling a single environment transition each have $O(1)$ cost. Consequently, a $q$-category multinomial or multivariate-hypergeometric draw costs $O(q)$, independent of the total count.

For a general Markov policy, a forward Markov-flow sample requires one $S$-category multinomial draw, $HS$ draws over $A$ action categories, and $HSA$ draws over $S$ next-state categories. Thus, its simulator-side computational cost is $C_{\mathrm{flow}}=O(S+HSA+HS^2A)=O(HS^2A)$, independent of the batch size $m$. The multivariate-hypergeometric splitting procedure in Section~\ref{sec:adaptive} has the same computational cost. Under the deterministic-policy specialization of Remark~\ref{rem:deterministic-policy}, the action-count draws disappear, and only one $S$-category transition row is sampled for each $(s,h)$. With inactive action rows represented implicitly as zero, both forward flow sampling and multivariate-hypergeometric splitting therefore have simulator-side cost $C_{\mathrm{flow}}^{\mathrm{det}}=O(S+HS^2)=O(HS^2)$.

For the applications in \Cref{sec:examples}, learner updates are computed from aggregate visitation, transition, or moment statistics and introduce no additional dependence on $m$. In all three applications, the learner-side cost per update is at most $O(HS^2A)$ and is accounted for separately when relevant.

One apparent exception is the broad class of Q-learning methods based on the visit-dependent learning rate $\alpha_t=(H+1)/(H+t)$~\citep{jin2018q}. A direct implementation applies the learning-rate update after individual visits and therefore appears to retain learner-side computation linear in the batch size $m$. However, suitably designed round-based variants can aggregate not only the cumulative learning-rate effect but also the trajectory-dependent statistics and confidence bonuses required by the update. As shown in Appendix~\ref{app:further-extensions} using the single-agent specialization of FedQ-Bernstein~\citep{zheng2023federated} as a representative example, these quantities can be computed from aggregate round statistics and pre- and post-round visit counts with cost independent of $m$. Thus, visit-dependent learning rates do not inherently preclude batch-size-independent learner updates.

\subsection{Cost of Direct Batch Simulation}
Direct batch simulation applies whenever the sampling rule remains fixed over a predetermined batch, and the required output is flow measurable. This includes repeated simulator-based queries used to construct empirical transition models~\citep{gheshlaghi2013minimax,sidford2018near,li2024breaking}, synthetic offline data generated under a fixed behavior policy~\citep{yin2021near,yin2021towards,shi2022pessimistic,li2024settling}, and online algorithms with predetermined batch or policy-update schedules~\citep{qiao2022sample,zhang2022near}. In each case, the resulting samples can be represented and processed through aggregate flow statistics, while forward Markov-flow sampling can generate these statistics directly without materializing individual trajectories or transitions.
\begin{proposition}[Computational cost of direct batch simulation]
\label{prop:direct-complexity}
For a predetermined batch of $m$ episodes, conventional episode-by-episode simulation has simulator-side computational cost $\Theta(mH)$, whereas direct Markov-flow sampling has cost $C_{\mathrm{flow}}=O(HS^2A)$. If the sampling policy is deterministic, the specialization in Remark~\ref{rem:deterministic-policy} has cost $C_{\mathrm{flow}}^{\mathrm{det}}=O(HS^2)$.
\end{proposition}
Thus, the simulator-side computational cost improves from $O(m)$ to $O(1)$ in its dependence on the batch size $m$. By~\Cref{thm:forward} and Corollary~\ref{cor:flow}, this reduction preserves the distribution of the Markov flow and any flow-measurable computation without generating the individual trajectories.

\subsection{Cost of Adaptive Batch Simulation}

Adaptive batch simulation is relevant to online algorithms whose policy updates are triggered by data-dependent visitation thresholds~\citep{bai2019provably,zhang2020almost}. Suppose a run of $K\geq 2$ episodes contains a random number $B$ of adaptive batches, possibly including a terminal partial batch truncated at episode $K$, with $B\leq B_{\max}$ almost surely. Conventional episode-by-episode simulation generates trajectories sequentially and therefore has simulator-side computational cost $\Theta(HK)$.

Our adaptive procedure instead locates each data-dependent boundary through binary refinement. If $k$ episodes have already been accepted, we choose the candidate size as $m=K-k$. Thus, the initial forward Markov-flow sample has cost $C_{\mathrm{flow}}$, and locating the boundary requires at most $\lceil\log_2 K\rceil$ conditional splits, each with the same computational cost. 
\begin{proposition}[Computational cost of adaptive batch simulation]
\label{prop:adaptive-complexity}
Suppose each candidate flow in Algorithm~\ref{alg:fastforward} is generated by forward Markov-flow sampling, and the number of data-dependent batches satisfies $B\leq B_{\max}$ almost surely for some deterministic $B_{\max}$. Conventional episode-by-episode simulation incurs simulator-side computational cost $\Theta(HK)$, whereas Algorithm~\ref{alg:fastforward} has simulator-side computational cost $O(B_{\max}HS^2A\log K)$. If the policy in every adaptive batch is deterministic, the corresponding cost is $O(B_{\max}HS^2\log K)$.
\end{proposition}

By~\Cref{thm:split,thm:adaptive}, this reduction preserves the exact distribution of every boundary flow. In stage-based tabular algorithms with geometrically increasing visitation thresholds~\citep{bai2019provably,zhang2020almost}, the number of updates typically admits a deterministic upper bound logarithmic in $K$, up to factors in $(H,S,A)$. Thus, when $B_{\max}=O(\log K)$ in its dependence on $K$, the simulator-side dependence on the number of episodes becomes $O(\log^2 K)$ rather than $O(K)$.

If evaluating the prefix-monotone trigger costs $C_g$ per call, Algorithm~\ref{alg:fastforward} makes $O(\log K)$ trigger evaluations per adaptive batch, giving total cost $O\!\left(B_{\max}(C_{\mathrm{flow}}+C_g)\log K\right)$. For UCB-Advantage in Section~\ref{sec:examples}, each adaptive-batch policy is deterministic. Hence, for every $(s,h)$, only the active action $a=\pi_h(s)$ can receive additional visits as the candidate prefix grows, while all inactive-action counts remain zero. The stage-threshold trigger therefore inspects only the $HS$ active state-action-step tuples, while the reference-value part contains only $HS$ state-step conditions. Thus, $C_g=O(HS)$, and together with $C_{\mathrm{flow}}^{\mathrm{det}}=O(HS^2)$, the total deterministic-policy cost remains $O(B_{\max}HS^2\log K)$.

\section{Representative Applications}
\label{sec:examples}

We now illustrate how the fast-simulation framework applies to three RL settings: simulator-based RL, offline RL, and online RL. The role of simulation differs across these settings. In simulator-based RL, repeated generative-model queries can be sampled directly in aggregate form. In offline RL, datasets generated under a fixed behavior policy can be sampled directly as aggregate Markov flows. In online RL, where the sampling policy evolves with the learner, adaptive batch simulation can fast-forward through intervals between data-dependent update boundaries. Together, these examples demonstrate the applicability of our framework across these three RL settings.

\subsection{Simulator-Based RL: Model-Based Planning with a Generative Model}
Consider model-based planning with a generative model~\citep{li2024breaking}. Suppose each $(s,a,h)$ is queried exactly $J$ times. Under conventional simulation, the simulator independently draws $J$ next states from $P_h(\cdot\mid s,a)$ one by one and uses them to construct the empirical transition model. Repeating this procedure for all $(s,a,h)$ requires $\Theta(SAHJ)$ simulator-side computational cost.

For this application, we assume simulator-side access to the transition probability vectors $P_h(\cdot\mid s,a)$, which need not be available to the learning algorithm. Under this access model, direct transition-count sampling, corresponding to \Cref{eq:forward3}, produces an empirical transition model with the same distribution as conventional generative-model sampling without generating the $J$ individual next-state samples. Specifically, with $C_h(s,a)=J$, \Cref{eq:forward3} gives
\begin{equation*}
E_h(s,a,\cdot)\sim\Mult\!\left(J,P_h(\cdot\mid s,a)\right),\text{ and thus }\widehat P_h(s'\mid s,a)=E_h(s,a,s')/J.
\end{equation*}
Thus, the resulting $\widehat P_h(\cdot\mid s,a)$, and hence the full empirical model $\widehat P$, has the same distribution as under conventional generative-model sampling. Furthermore, by Proposition~\ref{prop:direct-complexity}, direct batch simulation has simulator-side cost $O(S^2AH)$, eliminating the dependence on the prescribed query count $J$ while leaving the model-based planning procedure unchanged.

\subsection{Offline RL: Vanilla Pessimistic Value Iteration (VPVI)}

Consider the offline RL algorithm VPVI~\citep{yin2021towards}. Suppose an offline dataset consists of $K$ trajectories collected under a fixed behavior policy $\mu$. Under conventional simulation, these trajectories are generated one episode at a time, incurring $\Theta(HK)$ simulator-side computational cost.

VPVI builds the empirical transition model from the dataset as $\widehat P_h(s'\mid s,a)=E_h(s,a,s')/C_h(s,a)$ whenever $C_h(s,a)>0$. Under our deterministic-reward setting, the empirical rewards are also determined by the state-action counts. The subsequent pessimistic value iteration depends on the dataset only through these empirical quantities and the visitation counts $C_h(s,a)$.

Since the behavior policy $\mu$ is fixed and the dataset size $K$ is predetermined, the aggregate dataset statistics required by VPVI can instead be generated by direct batch simulation. Applying~\Cref{eq:forward1,eq:forward2,eq:forward3,eq:forward4} with policy $\mu$ and batch size $K$ directly generates the Markov flow $F_K=(N_1,C,E)$, from which all quantities required by VPVI can be computed without materializing individual trajectories. By~\Cref{thm:forward} and Corollary~\ref{cor:flow}, the resulting value estimates and learned policy have the same distribution as under conventional episode-by-episode data generation.

By Proposition \ref{prop:direct-complexity}, the simulator-side computational cost is reduced from $\Theta(HK)$ to $O(HS^2A)$, eliminating its dependence on the dataset size $K$, while leaving the VPVI computation unchanged. If the behavior policy $\mu$ is deterministic, the cost further reduces to $O(HS^2)$ by Remark~\ref{rem:deterministic-policy}.
\subsection{Online RL: UCB-Advantage}

Consider the online RL algorithm UCB-Advantage~\citep{zhang2020almost}. Under conventional simulation, the algorithm generates $K$ trajectories sequentially, incurring $\Theta(HK)$ simulator-side computational cost. Unlike the previous two settings, the policy is updated adaptively based on observed data, so the entire interaction cannot be generated as a single predetermined batch. Between updates, however, UCB-Advantage follows the fixed deterministic greedy policy $\pi_h(s)\in\argmax_{a\in\cA}Q_h(s,a)$.

UCB-Advantage updates its value estimates only when prescribed visitation thresholds are reached. For each $(s,a,h)$, visits are partitioned into geometrically growing stages with lengths satisfying $e_1=H$ and $e_{j+1}=\left\lfloor(1+1/H)e_j\right\rfloor$. At each stage end, the estimates are updated, and the intra-stage accumulators are reset. The reference value $V_h^{\mathrm{ref}}(s)$ is also updated at prescribed state-level visitation thresholds. We treat these stage and reference-value updates as adaptive boundaries.

Let $n_h(s,a)$ and $n_h(s)=\sum_{a\in\mathcal A} n_h(s,a)$ denote the cumulative
state-action and state visit counts before the current batch, respectively, and let $b_h^{\mathrm{stage}}(s,a)$ and $b_h^{\mathrm{ref}}(s)$ denote the next active stage and reference thresholds. For the prefix flow $G_j$, let $C_h^{(j)}(s,a)$ and $N_h^{(j)}(s)$ denote its state-action and state counts, respectively, and define $g_{\mathrm{UCB}}(Z,G_j)$ as
\begin{equation*}
\begin{split}
\ind\bigl\{&\exists(s,a,h):n_h(s,a)+C_h^{(j)}(s,a)\ge b_h^{\mathrm{stage}}(s,a)\text{ or }\exists(s,h):n_h(s)+N_h^{(j)}(s)\ge b_h^{\mathrm{ref}}(s)\bigr\}.
\end{split}
\end{equation*}
Since these visitation counts are nondecreasing in $j$, the trigger is prefix monotone. Moreover, within the current batch only $C_h^{(j)}(s,\pi_h(s))$ can change for each $(s,h)$; all $C_h^{(j)}(s,a)$ with $a\neq\pi_h(s)$ remain zero. Hence the trigger can be evaluated in $O(HS)$ time as described in Section~\ref{sec:complexity}.

Between consecutive boundaries, the policy and the value quantities entering the relevant accumulators remain fixed. Hence the visitation counts and first- and second-moment accumulators are flow measurable through~\Cref{eq:flow-function}. Together with the active thresholds stored in the pre-boundary learner state $Z$, the flow $G_\tau$ therefore determines all updates triggered at the boundary.

UCB-Advantage performs triggered updates immediately at step $h$ of the boundary episode, so the second requirement must be verified. Such an update modifies only $Q_h$, $V_h$, and possibly $V_h^{\mathrm{ref}}$, while decisions in the same episode use quantities indexed by $h+1,h+2,\ldots$. The step-$h$ accumulators depend only on $V_{h+1}$ and $V_{h+1}^{\mathrm{ref}}$, which are likewise unaffected. Hence postponing these updates does not alter any subsequent action or accumulator contribution within the boundary episode. Replaying the triggered updates in increasing order of $h$, with $Q_h$ and $V_h$ updated before $V_h^{\mathrm{ref}}(s)$ when both updates coincide, therefore yields the same post-episode learner state.

Thus, the boundary compatibility condition holds, and adaptive batch simulation can be applied between successive boundaries. By~\Cref{thm:adaptive}, Algorithm~\ref{alg:fastforward} reproduces the joint distribution of the post-boundary learner-state sequence under conventional episode-by-episode simulation.

Finally, each $(s,a,h)$ has at most $O(H\log K)$ stages, so the number of adaptive boundaries is $O(H^2SA\log K)$ including reference-value updates, and hence $B_{\max}=O(H^2SA\log K)$. Since each batch policy is deterministic, the flow-sampling cost is $O(B_{\max}HS^2\log K)$ by Proposition~\ref{prop:adaptive-complexity}; the $O(HS)$ trigger evaluations are dominated by this cost. Therefore, the simulator-side cost is $O(H^3S^3A\log^2 K)$, compared with $\Theta(HK)$ for conventional simulation. Thus, the dependence on $K$ is reduced to $O(\log^2 K)$ while preserving the joint distribution of the post-boundary learner states.

\section{Empirical Evaluation}
\label{sec:empirical}

We evaluate our framework on the three applications in \Cref{sec:examples}, comparing conventional episode-by-episode simulation with our exact fast simulator across four tabular MDP scales. We examine whether the two simulators produce consistent learner-relevant outcomes and compare their computational costs at different budgets. Full experimental settings and results are deferred to Appendix~\ref{sec:experiments}.

For model-based planning and VPVI, direct batch simulation produces learning outcomes closely consistent with conventional simulation, including value errors, policy gaps, and empirical transition errors, in agreement with the distributional equivalence established in Corollary~\ref{cor:flow}. At the same time, conventional simulation time grows approximately linearly with the number of queries or trajectories, whereas the runtime of direct Markov-flow sampling remains nearly constant with respect to the batch size, consistent with Proposition~\ref{prop:direct-complexity}. See Appendices~\ref{exp-planning} and~\ref{exp-vpvi}.

For UCB-Advantage, adaptive batch simulation closely reproduces both cumulative regret and the data-dependent batch-update behavior of conventional simulation, consistent with the distributional equivalence in \Cref{thm:adaptive}. Although adaptive splitting introduces additional overhead at small budgets, its runtime grows much more slowly with the number of episodes, consistent with Proposition~\ref{prop:adaptive-complexity}. In particular, once the algorithm reaches the accuracy regime where the average regret falls below $0.01$, the fast simulator is consistently faster across all tested scales, with the computational advantage increasing at larger budgets. Details are provided in Appendix~\ref{exp-ucb}. Overall, the experiments support both the distributional equivalence and computational improvements predicted by our theory.

\section{Conclusion}

We present an exact fast simulation framework for tabular RL that preserves the relevant distributions of the learning procedure. For predetermined batches, Markov flows can be obtained through parallel trajectory aggregation or, with direct access to the initial-state and transition distributions, through forward Markov-flow sampling without materializing individual trajectories. Exact hypergeometric splitting extends the framework to data-dependent boundaries through adaptive refinement. Under forward Markov-flow sampling, the simulator-side dependence on batch size is reduced from linear to constant for direct simulation and logarithmic for adaptive simulation. Applications to simulator-based, offline, and online RL demonstrate the applicability of our framework.

\bibliography{main}
\bibliographystyle{plainnat}

\appendix
\newpage
The appendices are organized as follows. Appendix~\ref{sec:related} discusses related work. Appendix~\ref{sec:experiments} provides the complete experimental setup and results. Appendix~\ref{app:proof-forward} proves the exactness of direct Markov-flow sampling in \Cref{thm:forward}, Appendix~\ref{app:proof-split} establishes the conditional correctness of exact flow splitting in \Cref{thm:split}, and Appendix~\ref{app:proof-adaptive} proves the exactness of adaptive batch simulation in \Cref{thm:adaptive}. Appendix~\ref{app:algorithm-review} reviews the learning algorithms used in \Cref{sec:examples}, and Appendix~\ref{app:further-extensions} discusses further applications to Q-learning with visit-dependent learning rates and federated RL.
\section{Related Work}
\label{sec:related}

\textbf{Parallel and batched RL simulation.}
Efficient RL systems improve data-generation throughput by executing many environments concurrently. Sample Factory~\citep{pmlr-v119-petrenko20a} combines asynchronous sampling with GPU-based inference. Large Batch Simulation~\citep{shacklett2021large} jointly designs rendering and simulation around batches of requests, sharing scene resources and amortizing rendering, data-loading, and synchronization costs. Galactic~\citep{Berges_2023_CVPR} further optimizes physics, rendering, inference, and learning together for embodied RL. These systems generate environment transitions efficiently through parallelism and batching. Our framework addresses a complementary opportunity in finite tabular MDPs: directly sampling the aggregate statistics of many episodes when the learning procedure does not require their individual trajectories.

\textbf{Simulation with aggregate counts.}
Collective graphical models perform probabilistic inference on aggregate sufficient statistics of an individual-level graphical model~\citep{NIPS2011_fccb3cdc}. In collective multiagent RL, aggregate state, state-action, and transition counts are used to represent population-level dynamics and perform count-based learner updates~\citep{NEURIPS2018_94bb077f}. These count-based perspectives are related to our Markov-flow representation. Our framework instead studies exact batch simulation across independent episodes of a learning procedure, formalized through a flow-measurable interface. Beyond direct generation of aggregate statistics, we introduce exact conditional splitting of a previously sampled multi-step flow to locate the first data-dependent update boundary. Under the stated compatibility conditions, updating the learner and restarting simulation preserves the joint distribution of its boundary learner states.

The preceding discussion positions our framework relative to existing simulation methods. We next discuss its relationship to simulator-based, offline, and online RL, together with extensions to federated RL.

\textbf{Simulator-based RL.}
Generative models allow transition samples to be queried from specified state-action pairs or state-action-step triples. Prior work studies the sample and computational complexity of learning under this interface~\citep{sidford2018near}. Model-based methods estimate transition kernels from repeated queries and plan in the resulting empirical MDP; \citet{li2024breaking} establish guarantees covering finite-horizon settings, and \citet{kalyanakrishnan2025viewcertaintyequivalencemethodpac} revisit certainty-equivalence planning through the trajectory tree method. We study the simulator-side cost of realizing the statistics required by a prescribed learning procedure. With simulator-side access to transition probability vectors, the empirical transition counts associated with repeated queries can be sampled directly. This preserves their joint distribution, with computational costs measured under the convention in Section~\ref{sec:complexity}. 

\textbf{Offline RL.}
Offline RL learns from a fixed dataset without further environment interaction. In finite-horizon tabular MDPs, pessimistic value-iteration methods use empirical visitation and transition statistics to construct policies~\citep{yin2021towards,li2024settling}. The tabular implementation of selective uncertainty propagation also uses empirical transition models and count-based bonuses~\citep{krishnamurthy2025selective}. These methods motivate direct generation of the statistics needed from synthetic offline datasets. In our VPVI example, the behavior policy and dataset size are fixed, and, under deterministic rewards, the dataset enters the learner only through visitation and transition counts. Direct Markov-flow sampling therefore generates required data statistics without materializing individual trajectories, preserving the distribution of the value estimates and policy.

\textbf{Batch and low-switching online RL.}
Another line of work reduces online adaptivity by limiting policy updates. \citet{bai2019provably} study tabular Q-learning with low switching cost, while UCB-Advantage~\citep{zhang2020almost} combines reference-advantage decomposition with stage-wise updates. Subsequent work further reduces switching complexity~\citep{qiao2022sample}, studies multi-batch RL with predetermined update schedules~\citep{zhang2022near}, and develops a general framework for rare policy switching and batch learning~\citep{pmlr-v235-xiong24d}. Recent model-free algorithms improve the tradeoffs among regret, burn-in, and switching costs~\citep{zheng2024gap,zhang2025regret,zhang2025q}. These methods motivate simulation over intervals with a fixed policy. Low switching alone, however, does not ensure flow measurability: visit-index-dependent weights or sequential updates may require additional trajectory-order information. Our adaptive method applies when the required statistics are flow measurable, the trigger is prefix monotone, and the boundary compatibility condition holds. UCB-Advantage provides a concrete example in Section~\ref{sec:examples}.

\textbf{Federated RL.}
Event-triggered federated Q-learning achieves logarithmic communication cost in tabular episodic MDPs~\citep{zheng2023federated,zheng2024federated}, and subsequent work establishes gap-dependent regret guarantees~\citep{zhang2025gapdependent}. These settings provide natural applications of our exact batch simulation framework. In Appendix~\ref{app:further-extensions}, we verify the required conditions for FedQ-Bernstein and FedQ-Advantage, illustrating both round-based aggregate updates and distributed local exploration.

\section{Numerical Experiments}
\label{sec:experiments}

We empirically evaluate our fast simulation framework on model-based planning with a generative model, offline Vanilla Pessimistic Value Iteration (VPVI), and online UCB-Advantage.\footnote{The experiments were split by algorithm, scale, and either simulation replicate (fixed-MDP comparison) or MDP instance (cross-MDP timing). Each task ran on one core of an Intel Xeon Gold 6430 with 4~GB of allocated memory and finished within 15 hours.} In particular, we examine whether the learner-relevant quantities produced by the two simulation methods are consistent with the theoretical distributional equivalence and whether the reduced simulation complexity yields substantial computational gains. We consider four tabular MDP sizes:
\[
 (H,S,A)\in\{(2,3,3),(5,5,5),(7,8,6),(10,15,10)\}.
\]
For each size, the fixed-MDP experiment holds the underlying MDP fixed and uses ten independent simulation replicates per method to assess the distributional equivalence of learner-relevant quantities. The separate cross-MDP timing experiment uses ten independently generated MDP instances, with the conventional episode-by-episode simulation and exact fast simulation methods evaluated on the same instance and simulation seed, to assess computational efficiency and scaling across problem instances. For each $(h,s,a)$, the transition probability vector $P_h(\cdot \mid s,a)$ is sampled independently from a symmetric Dirichlet distribution with unit concentration; deterministic rewards are drawn independently from $\operatorname{Unif}[0,1]$, and the initial-state distribution is uniform.

\textbf{Measurement and uncertainty.}
Both implementations use the same programming environment, learner implementation, and boundary checks, differing only in the simulation procedure. The simulator time covers only the generation of transition samples or episode batches and the construction of the aggregate statistics passed to the learner. The learner time covers the subsequent algorithmic update, including policy-value computation for performance tracking whenever the policy changes. MDP construction, file I/O, and plotting are excluded from both timings.

For short-running configurations, identified by a pilot wall time below 0.5 seconds, we use two warm-up runs followed by five timed blocks, with each block repeating fresh runs until approximately 0.5 seconds of total runtime is accumulated; the reported time is the median of the per-run runtimes across the five blocks.

Tables report sample means with 95\% percentile bootstrap confidence intervals (CIs) for learning statistics and sample medians with 95\% percentile bootstrap CIs for timing results, based on 10,000 resamples. Numerical entries in the tables are reported to three significant digits, with scientific notation used where appropriate. For learning statistics, the conventional and fast methods use independent simulation replicates on the same fixed MDP. Confidence intervals for each mean and for the difference in means are therefore constructed by resampling the two groups independently. For timing, the ten MDP instances serve as the sampling units; median runtimes and median same-instance runtime ratios are bootstrapped across instances.

\textbf{Budget grid.}
For the distributional comparisons, both the conventional and fast simulators are evaluated on the common grids $J=10^1,\ldots,10^6$ for model-based planning and $K=10^1,\ldots,10^8$ for VPVI and UCB-Advantage. The timing experiments use the same measured ranges for the conventional simulator, while the fast simulator is further evaluated at every power of ten up to $10^{15}$.

The extended timing range is motivated by sample-complexity guarantees for the corresponding learning problems. For model-based planning, $\widetilde O(H^3/\epsilon^2)$ queries per $(h,s,a)$ suffice to obtain a uniformly $\epsilon$-accurate value function. For VPVI, $\widetilde O(H^4 S C^\star/\epsilon^2)$ episodes suffice to learn a policy with value gap at most $\epsilon$ under the standard coverage condition, where $C^\star$ bounds the state-action occupancy ratio between an optimal policy and the behavior policy. For UCB-Advantage, $\widetilde O(H^3SA/\epsilon^2)$ episodes suffice to learn an $\epsilon$-optimal policy. For our largest problem scale, setting $\epsilon=10^{-4}$ yields aggregate query or episode counts on the order of $10^{13}$--$10^{14}$ for planning and UCB-Advantage, before logarithmic and confidence factors, while the VPVI scale additionally depends on the coverage coefficient $C^\star\geq 1$. These estimates motivate probing substantially larger simulation budgets. Accordingly, we extend the fast-simulator timing experiments up to $10^{15}$ as a large-count stress test.

Because running the conventional simulator beyond the measured range would incur prohibitive computational cost, we extrapolate its runtime under a linear-scaling assumption with respect to the number of queries or episodes. Specifically, starting from the largest directly measured budget, we multiply the measured runtime by ten for each additional order-of-magnitude increase in the budget; all extrapolated entries are marked by $\approx$. This assumption is supported empirically by the approximately linear growth of the measured conventional runtimes over the observed range, and theoretically by the fact that conventional simulation performs work proportional to the number of generated queries or episodes. In contrast, all reported fast-simulation runtimes are obtained from direct measurements. We define speedup as the ratio of the conventional runtime to the fast-simulation runtime, so a value greater than one indicates that the fast simulator is faster.

\subsection{Model-Based Planning with a Generative Model}
\label{exp-planning}
The planning budget $J$ is applied to every $(h,s,a)$, resulting in a total of $HSAJ$ generative-model queries. The conventional and fast methods differ only in how the transition counts are generated; the same planning routine is applied afterward. We use the average optimal-value error
\[
E_V(J)=\frac{1}{HS}\sum\nolimits_{h=1}^{H}\sum\nolimits_{s=1}^{S}
\left|\widehat V_h^\star(s)-V_h^\star(s)\right|,
\]
where $\widehat V^\star$ and $V^\star$ are the optimal value functions of the empirical and true MDPs, respectively. This continuous statistic directly measures the accuracy of the empirical planning solution. Since the empirical transition counts have the same joint distribution under conventional and fast simulation, the resulting empirical MDP, $\widehat V^\star$, and hence $E_V(J)$ are also equal in distribution.

Tables~\ref{tab:planning-average_v_error-s1}--\ref{tab:planning-average_v_error-s4} show that the conventional and fast simulators achieve very similar value errors across all four problem scales and budgets. Most bootstrap confidence intervals for the difference in mean error contain zero, with 21 of the 24 intervals including zero. These results are consistent with the distributional equivalence predicted by our theory.

The timing results in Tables~\ref{tab:planning-timing-s1}--\ref{tab:planning-timing-s4} show that conventional simulation is substantially slower and its runtime grows approximately linearly with the query budget $J$. In contrast, our fast method uses direct batch simulation to sample the transition counts directly, so its simulator runtime remains nearly unchanged as $J$ increases, consistent with our complexity result in Proposition \ref{prop:direct-complexity} that removes the dependence on $J$. Moreover, the learner time remains below $0.007$ seconds in all tested configurations and is substantially smaller than the conventional simulation time at large scales and budgets. This shows that simulation is the dominant computational bottleneck in this regime, making the reduction in simulator-side complexity practically significant.

\begin{table}[htbp]
\centering
\tiny
\caption{Average optimal-value error for $(H,S,A)=(2,3,3)$. Entries report the mean and 95\% percentile bootstrap confidence interval over ten independent simulation seeds on a fixed MDP. The error difference is defined as fast minus conventional.}
\label{tab:planning-average_v_error-s1}
\setlength{\tabcolsep}{1.5pt}
\begin{tabular}{rlll}
\toprule
Budget & Conventional & Fast & Difference \\
\midrule
$10^{1}$ & $1.41\mathord{\times}10^{-2}\,[1.08\mathord{\times}10^{-2},1.79\mathord{\times}10^{-2}]$ & $1.08\mathord{\times}10^{-2}\,[8.27\mathord{\times}10^{-3},1.35\mathord{\times}10^{-2}]$ & $-3.29\mathord{\times}10^{-3}\,[-7.79\mathord{\times}10^{-3},1.01\mathord{\times}10^{-3}]$ \\
$10^{2}$ & $3.58\mathord{\times}10^{-3}\,[2.72\mathord{\times}10^{-3},4.42\mathord{\times}10^{-3}]$ & $3.21\mathord{\times}10^{-3}\,[2.37\mathord{\times}10^{-3},4.03\mathord{\times}10^{-3}]$ & $-3.67\mathord{\times}10^{-4}\,[-1.59\mathord{\times}10^{-3},8.24\mathord{\times}10^{-4}]$ \\
$10^{3}$ & $1.30\mathord{\times}10^{-3}\,[8.96\mathord{\times}10^{-4},1.72\mathord{\times}10^{-3}]$ & $8.97\mathord{\times}10^{-4}\,[6.69\mathord{\times}10^{-4},1.14\mathord{\times}10^{-3}]$ & $-4.06\mathord{\times}10^{-4}\,[-8.84\mathord{\times}10^{-4},6.27\mathord{\times}10^{-5}]$ \\
$10^{4}$ & $3.79\mathord{\times}10^{-4}\,[2.91\mathord{\times}10^{-4},4.69\mathord{\times}10^{-4}]$ & $2.90\mathord{\times}10^{-4}\,[2.34\mathord{\times}10^{-4},3.37\mathord{\times}10^{-4}]$ & $-8.91\mathord{\times}10^{-5}\,[-1.96\mathord{\times}10^{-4},1.17\mathord{\times}10^{-5}]$ \\
$10^{5}$ & $1.25\mathord{\times}10^{-4}\,[9.90\mathord{\times}10^{-5},1.52\mathord{\times}10^{-4}]$ & $8.20\mathord{\times}10^{-5}\,[6.27\mathord{\times}10^{-5},9.99\mathord{\times}10^{-5}]$ & $-4.28\mathord{\times}10^{-5}\,[-7.67\mathord{\times}10^{-5},-1.16\mathord{\times}10^{-5}]$ \\
$10^{6}$ & $3.51\mathord{\times}10^{-5}\,[2.38\mathord{\times}10^{-5},4.70\mathord{\times}10^{-5}]$ & $2.49\mathord{\times}10^{-5}\,[1.80\mathord{\times}10^{-5},3.14\mathord{\times}10^{-5}]$ & $-1.02\mathord{\times}10^{-5}\,[-2.39\mathord{\times}10^{-5},2.96\mathord{\times}10^{-6}]$ \\
\bottomrule
\end{tabular}
\end{table}

\begin{table}[htbp]
\centering
\tiny
\caption{Average optimal-value error for $(H,S,A)=(5,5,5)$. Entries report the mean and 95\% percentile bootstrap confidence interval over ten independent simulation seeds on a fixed MDP. The error difference is defined as fast minus conventional.}
\label{tab:planning-average_v_error-s2}
\setlength{\tabcolsep}{1.5pt}
\begin{tabular}{rlll}
\toprule
Budget & Conventional & Fast & Difference \\
\midrule
$10^{1}$ & $2.48\mathord{\times}10^{-2}\,[2.06\mathord{\times}10^{-2},2.92\mathord{\times}10^{-2}]$ & $1.98\mathord{\times}10^{-2}\,[1.74\mathord{\times}10^{-2},2.21\mathord{\times}10^{-2}]$ & $-4.94\mathord{\times}10^{-3}\,[-9.89\mathord{\times}10^{-3},-2.74\mathord{\times}10^{-4}]$ \\
$10^{2}$ & $6.99\mathord{\times}10^{-3}\,[6.15\mathord{\times}10^{-3},7.77\mathord{\times}10^{-3}]$ & $8.55\mathord{\times}10^{-3}\,[6.91\mathord{\times}10^{-3},1.04\mathord{\times}10^{-2}]$ & $1.56\mathord{\times}10^{-3}\,[-2.88\mathord{\times}10^{-4},3.62\mathord{\times}10^{-3}]$ \\
$10^{3}$ & $1.91\mathord{\times}10^{-3}\,[1.57\mathord{\times}10^{-3},2.30\mathord{\times}10^{-3}]$ & $2.07\mathord{\times}10^{-3}\,[1.71\mathord{\times}10^{-3},2.54\mathord{\times}10^{-3}]$ & $1.55\mathord{\times}10^{-4}\,[-3.81\mathord{\times}10^{-4},7.39\mathord{\times}10^{-4}]$ \\
$10^{4}$ & $7.60\mathord{\times}10^{-4}\,[6.21\mathord{\times}10^{-4},8.89\mathord{\times}10^{-4}]$ & $6.34\mathord{\times}10^{-4}\,[5.12\mathord{\times}10^{-4},8.00\mathord{\times}10^{-4}]$ & $-1.27\mathord{\times}10^{-4}\,[-3.13\mathord{\times}10^{-4},9.20\mathord{\times}10^{-5}]$ \\
$10^{5}$ & $2.08\mathord{\times}10^{-4}\,[1.77\mathord{\times}10^{-4},2.36\mathord{\times}10^{-4}]$ & $2.18\mathord{\times}10^{-4}\,[1.93\mathord{\times}10^{-4},2.45\mathord{\times}10^{-4}]$ & $1.04\mathord{\times}10^{-5}\,[-2.74\mathord{\times}10^{-5},5.09\mathord{\times}10^{-5}]$ \\
$10^{6}$ & $6.40\mathord{\times}10^{-5}\,[5.51\mathord{\times}10^{-5},7.41\mathord{\times}10^{-5}]$ & $6.45\mathord{\times}10^{-5}\,[5.68\mathord{\times}10^{-5},7.30\mathord{\times}10^{-5}]$ & $4.12\mathord{\times}10^{-7}\,[-1.23\mathord{\times}10^{-5},1.27\mathord{\times}10^{-5}]$ \\
\bottomrule
\end{tabular}
\end{table}

\begin{table}[htbp]
\centering
\tiny
\caption{Average optimal-value error for $(H,S,A)=(7, 8, 6)$. Entries report the mean and 95\% percentile bootstrap confidence interval over ten independent simulation seeds on a fixed MDP. The error difference is defined as fast minus conventional.}
\label{tab:planning-average_v_error-s3}
\setlength{\tabcolsep}{1.5pt}
\begin{tabular}{rlll}
\toprule
Budget & Conventional & Fast & Difference \\
\midrule
$10^{1}$ & $2.58\mathord{\times}10^{-2}\,[2.26\mathord{\times}10^{-2},2.90\mathord{\times}10^{-2}]$ & $2.91\mathord{\times}10^{-2}\,[2.49\mathord{\times}10^{-2},3.39\mathord{\times}10^{-2}]$ & $3.29\mathord{\times}10^{-3}\,[-2.15\mathord{\times}10^{-3},9.02\mathord{\times}10^{-3}]$ \\
$10^{2}$ & $8.29\mathord{\times}10^{-3}\,[7.45\mathord{\times}10^{-3},9.13\mathord{\times}10^{-3}]$ & $7.45\mathord{\times}10^{-3}\,[6.70\mathord{\times}10^{-3},8.44\mathord{\times}10^{-3}]$ & $-8.43\mathord{\times}10^{-4}\,[-1.97\mathord{\times}10^{-3},4.28\mathord{\times}10^{-4}]$ \\
$10^{3}$ & $2.53\mathord{\times}10^{-3}\,[2.31\mathord{\times}10^{-3},2.78\mathord{\times}10^{-3}]$ & $3.07\mathord{\times}10^{-3}\,[2.67\mathord{\times}10^{-3},3.47\mathord{\times}10^{-3}]$ & $5.48\mathord{\times}10^{-4}\,[8.15\mathord{\times}10^{-5},1.00\mathord{\times}10^{-3}]$ \\
$10^{4}$ & $8.01\mathord{\times}10^{-4}\,[6.89\mathord{\times}10^{-4},9.29\mathord{\times}10^{-4}]$ & $8.81\mathord{\times}10^{-4}\,[7.68\mathord{\times}10^{-4},1.03\mathord{\times}10^{-3}]$ & $8.01\mathord{\times}10^{-5}\,[-8.82\mathord{\times}10^{-5},2.64\mathord{\times}10^{-4}]$ \\
$10^{5}$ & $2.68\mathord{\times}10^{-4}\,[2.45\mathord{\times}10^{-4},2.92\mathord{\times}10^{-4}]$ & $2.58\mathord{\times}10^{-4}\,[2.31\mathord{\times}10^{-4},2.90\mathord{\times}10^{-4}]$ & $-1.01\mathord{\times}10^{-5}\,[-4.69\mathord{\times}10^{-5},2.88\mathord{\times}10^{-5}]$ \\
$10^{6}$ & $7.63\mathord{\times}10^{-5}\,[6.71\mathord{\times}10^{-5},8.64\mathord{\times}10^{-5}]$ & $8.77\mathord{\times}10^{-5}\,[7.68\mathord{\times}10^{-5},9.92\mathord{\times}10^{-5}]$ & $1.14\mathord{\times}10^{-5}\,[-3.45\mathord{\times}10^{-6},2.66\mathord{\times}10^{-5}]$ \\
\bottomrule
\end{tabular}
\end{table}

\begin{table}[htbp]
\centering
\tiny
\caption{Average optimal-value error for $(H,S,A)=(10, 15, 10)$. Entries report the mean and 95\% percentile bootstrap confidence interval over ten independent simulation seeds on a fixed MDP. The error difference is defined as fast minus conventional.}
\label{tab:planning-average_v_error-s4}
\setlength{\tabcolsep}{2pt}
\begin{tabular}{rlll}
\toprule
Budget & Conventional & Fast & Difference \\
\midrule
$10^{1}$ & $1.95\mathord{\times}10^{-2}\,[1.76\mathord{\times}10^{-2},2.23\mathord{\times}10^{-2}]$ & $1.70\mathord{\times}10^{-2}\,[1.53\mathord{\times}10^{-2},1.94\mathord{\times}10^{-2}]$ & $-2.52\mathord{\times}10^{-3}\,[-5.85\mathord{\times}10^{-3},6.62\mathord{\times}10^{-4}]$ \\
$10^{2}$ & $5.24\mathord{\times}10^{-3}\,[4.92\mathord{\times}10^{-3},5.64\mathord{\times}10^{-3}]$ & $5.73\mathord{\times}10^{-3}\,[5.34\mathord{\times}10^{-3},6.14\mathord{\times}10^{-3}]$ & $4.81\mathord{\times}10^{-4}\,[-6.87\mathord{\times}10^{-5},1.02\mathord{\times}10^{-3}]$ \\
$10^{3}$ & $1.85\mathord{\times}10^{-3}\,[1.63\mathord{\times}10^{-3},2.15\mathord{\times}10^{-3}]$ & $1.64\mathord{\times}10^{-3}\,[1.58\mathord{\times}10^{-3},1.71\mathord{\times}10^{-3}]$ & $-2.05\mathord{\times}10^{-4}\,[-5.07\mathord{\times}10^{-4},2.85\mathord{\times}10^{-5}]$ \\
$10^{4}$ & $5.33\mathord{\times}10^{-4}\,[5.09\mathord{\times}10^{-4},5.58\mathord{\times}10^{-4}]$ & $5.44\mathord{\times}10^{-4}\,[4.94\mathord{\times}10^{-4},6.04\mathord{\times}10^{-4}]$ & $1.11\mathord{\times}10^{-5}\,[-4.46\mathord{\times}10^{-5},7.48\mathord{\times}10^{-5}]$ \\
$10^{5}$ & $1.74\mathord{\times}10^{-4}\,[1.61\mathord{\times}10^{-4},1.92\mathord{\times}10^{-4}]$ & $1.78\mathord{\times}10^{-4}\,[1.69\mathord{\times}10^{-4},1.87\mathord{\times}10^{-4}]$ & $3.43\mathord{\times}10^{-6}\,[-1.63\mathord{\times}10^{-5},2.02\mathord{\times}10^{-5}]$ \\
$10^{6}$ & $5.40\mathord{\times}10^{-5}\,[5.07\mathord{\times}10^{-5},5.75\mathord{\times}10^{-5}]$ & $5.37\mathord{\times}10^{-5}\,[5.02\mathord{\times}10^{-5},5.71\mathord{\times}10^{-5}]$ & $-3.20\mathord{\times}10^{-7}\,[-5.20\mathord{\times}10^{-6},4.42\mathord{\times}10^{-6}]$ \\
\bottomrule
\end{tabular}
\end{table}

\begin{table}[htbp]
\centering
\tiny
\caption{Simulator time for $(H,S,A)=(2, 3, 3)$, in seconds. Entries are medians with 95\% percentile bootstrap CIs over ten MDP instances; speedup is the median of paired time ratios. Conventional entries beyond the last measured $10^{6}$ budget are marked $\approx$ and extrapolated by a factor of ten per additional power of ten.}
\label{tab:planning-timing-s1}
\setlength{\tabcolsep}{3pt}
\begin{tabular}{rlll}
\toprule
Budget & Conventional (s) & Fast (s) & Speedup \\
\midrule
$10^{1}$ & $3.01\mathord{\times}10^{-4}\,[2.61\mathord{\times}10^{-4},3.33\mathord{\times}10^{-4}]$ & $2.77\mathord{\times}10^{-5}\,[2.38\mathord{\times}10^{-5},3.04\mathord{\times}10^{-5}]$ & $1.09\mathord{\times}10^{1}\,[1.09\mathord{\times}10^{1},1.10\mathord{\times}10^{1}]$ \\
$10^{2}$ & $2.81\mathord{\times}10^{-3}\,[2.51\mathord{\times}10^{-3},3.18\mathord{\times}10^{-3}]$ & $2.93\mathord{\times}10^{-5}\,[2.56\mathord{\times}10^{-5},3.29\mathord{\times}10^{-5}]$ & $9.76\mathord{\times}10^{1}\,[9.58\mathord{\times}10^{1},9.88\mathord{\times}10^{1}]$ \\
$10^{3}$ & $2.91\mathord{\times}10^{-2}\,[2.51\mathord{\times}10^{-2},3.20\mathord{\times}10^{-2}]$ & $2.85\mathord{\times}10^{-5}\,[2.44\mathord{\times}10^{-5},3.11\mathord{\times}10^{-5}]$ & $1.03\mathord{\times}10^{3}\,[1.02\mathord{\times}10^{3},1.03\mathord{\times}10^{3}]$ \\
$10^{4}$ & $2.88\mathord{\times}10^{-1}\,[2.51\mathord{\times}10^{-1},3.21\mathord{\times}10^{-1}]$ & $2.82\mathord{\times}10^{-5}\,[2.40\mathord{\times}10^{-5},3.06\mathord{\times}10^{-5}]$ & $1.04\mathord{\times}10^{4}\,[1.03\mathord{\times}10^{4},1.07\mathord{\times}10^{4}]$ \\
$10^{5}$ & $2.90\,[2.52,3.24]$ & $2.72\mathord{\times}10^{-5}\,[2.37\mathord{\times}10^{-5},3.10\mathord{\times}10^{-5}]$ & $1.06\mathord{\times}10^{5}\,[1.05\mathord{\times}10^{5},1.08\mathord{\times}10^{5}]$ \\
$10^{6}$ & $2.89\mathord{\times}10^{1}\,[2.54\mathord{\times}10^{1},3.23\mathord{\times}10^{1}]$ & $2.77\mathord{\times}10^{-5}\,[2.38\mathord{\times}10^{-5},3.03\mathord{\times}10^{-5}]$ & $1.07\mathord{\times}10^{6}\,[1.05\mathord{\times}10^{6},1.07\mathord{\times}10^{6}]$ \\
$10^{7}$ & $\approx2.89\mathord{\times}10^{2}$ & $2.79\mathord{\times}10^{-5}\,[2.38\mathord{\times}10^{-5},3.08\mathord{\times}10^{-5}]$ & $\approx1.06\mathord{\times}10^{7}$ \\
\bottomrule
\end{tabular}
\end{table}

\begin{table}[htbp]
\centering
\scriptsize
\caption{Simulator time for $(H,S,A)=(5, 5, 5)$, in seconds. Entries are medians with 95\% percentile bootstrap CIs over ten MDP instances; speedup is the median of paired time ratios. Conventional entries beyond the last measured $10^{6}$ budget are marked $\approx$ and extrapolated by a factor of ten per additional power of ten.}
\label{tab:planning-timing-s2}
\setlength{\tabcolsep}{3pt}
\begin{tabular}{rlll}
\toprule
Budget & Conventional (s) & Fast (s) & Speedup \\
\midrule
$10^{1}$ & $2.25\mathord{\times}10^{-3}\,[1.81\mathord{\times}10^{-3},2.33\mathord{\times}10^{-3}]$ & $1.70\mathord{\times}10^{-4}\,[1.64\mathord{\times}10^{-4},1.93\mathord{\times}10^{-4}]$ & $1.14\mathord{\times}10^{1}\,[1.12\mathord{\times}10^{1},1.35\mathord{\times}10^{1}]$ \\
$10^{2}$ & $2.17\mathord{\times}10^{-2}\,[1.76\mathord{\times}10^{-2},2.24\mathord{\times}10^{-2}]$ & $1.90\mathord{\times}10^{-4}\,[1.84\mathord{\times}10^{-4},2.21\mathord{\times}10^{-4}]$ & $9.82\mathord{\times}10^{1}\,[9.50\mathord{\times}10^{1},1.17\mathord{\times}10^{2}]$ \\
$10^{3}$ & $2.18\mathord{\times}10^{-1}\,[1.75\mathord{\times}10^{-1},2.25\mathord{\times}10^{-1}]$ & $1.81\mathord{\times}10^{-4}\,[1.76\mathord{\times}10^{-4},2.16\mathord{\times}10^{-4}]$ & $1.03\mathord{\times}10^{3}\,[1.00\mathord{\times}10^{3},1.22\mathord{\times}10^{3}]$ \\
$10^{4}$ & $2.20\,[1.77,2.26]$ & $1.76\mathord{\times}10^{-4}\,[1.70\mathord{\times}10^{-4},2.05\mathord{\times}10^{-4}]$ & $1.07\mathord{\times}10^{4}\,[1.05\mathord{\times}10^{4},1.28\mathord{\times}10^{4}]$ \\
$10^{5}$ & $2.19\mathord{\times}10^{1}\,[1.77\mathord{\times}10^{1},2.26\mathord{\times}10^{1}]$ & $1.73\mathord{\times}10^{-4}\,[1.68\mathord{\times}10^{-4},1.99\mathord{\times}10^{-4}]$ & $1.07\mathord{\times}10^{5}\,[1.06\mathord{\times}10^{5},1.29\mathord{\times}10^{5}]$ \\
$10^{6}$ & $2.12\mathord{\times}10^{2}\,[1.79\mathord{\times}10^{2},2.22\mathord{\times}10^{2}]$ & $1.73\mathord{\times}10^{-4}\,[1.67\mathord{\times}10^{-4},1.98\mathord{\times}10^{-4}]$ & $1.07\mathord{\times}10^{6}\,[1.06\mathord{\times}10^{6},1.24\mathord{\times}10^{6}]$ \\
$10^{7}$ & $\approx2.12\mathord{\times}10^{3}$ & $1.74\mathord{\times}10^{-4}\,[1.68\mathord{\times}10^{-4},1.96\mathord{\times}10^{-4}]$ & $\approx1.08\mathord{\times}10^{7}$ \\
\bottomrule
\end{tabular}
\end{table}

\begin{table}[htbp]
\centering
\scriptsize
\caption{Simulator time for $(H,S,A)=(7, 8, 6)$, in seconds. Entries are medians with 95\% percentile bootstrap CIs over ten MDP instances; speedup is the median of paired time ratios. Conventional entries beyond the last measured $10^{6}$ budget are marked $\approx$ and extrapolated by a factor of ten per additional power of ten.}
\label{tab:planning-timing-s3}
\setlength{\tabcolsep}{3pt}
\begin{tabular}{rlll}
\toprule
Budget & Conventional (s) & Fast (s) & Speedup \\
\midrule
$10^{1}$ & $5.44\mathord{\times}10^{-3}\,[5.29\mathord{\times}10^{-3},5.72\mathord{\times}10^{-3}]$ & $5.04\mathord{\times}10^{-4}\,[4.92\mathord{\times}10^{-4},5.31\mathord{\times}10^{-4}]$ & $1.08\mathord{\times}10^{1}\,[1.05\mathord{\times}10^{1},1.12\mathord{\times}10^{1}]$ \\
$10^{2}$ & $5.10\mathord{\times}10^{-2}\,[5.01\mathord{\times}10^{-2},5.34\mathord{\times}10^{-2}]$ & $6.40\mathord{\times}10^{-4}\,[6.19\mathord{\times}10^{-4},6.70\mathord{\times}10^{-4}]$ & $7.94\mathord{\times}10^{1}\,[7.77\mathord{\times}10^{1},8.49\mathord{\times}10^{1}]$ \\
$10^{3}$ & $5.16\mathord{\times}10^{-1}\,[5.10\mathord{\times}10^{-1},5.35\mathord{\times}10^{-1}]$ & $6.10\mathord{\times}10^{-4}\,[5.93\mathord{\times}10^{-4},6.27\mathord{\times}10^{-4}]$ & $8.64\mathord{\times}10^{2}\,[8.33\mathord{\times}10^{2},8.87\mathord{\times}10^{2}]$ \\
$10^{4}$ & $5.28\,[5.26,5.51]$ & $5.66\mathord{\times}10^{-4}\,[5.47\mathord{\times}10^{-4},5.88\mathord{\times}10^{-4}]$ & $9.44\mathord{\times}10^{3}\,[9.17\mathord{\times}10^{3},9.91\mathord{\times}10^{3}]$ \\
$10^{5}$ & $5.27\mathord{\times}10^{1}\,[5.18\mathord{\times}10^{1},5.48\mathord{\times}10^{1}]$ & $5.50\mathord{\times}10^{-4}\,[5.37\mathord{\times}10^{-4},5.67\mathord{\times}10^{-4}]$ & $9.61\mathord{\times}10^{4}\,[9.41\mathord{\times}10^{4},1.00\mathord{\times}10^{5}]$ \\
$10^{6}$ & $5.27\mathord{\times}10^{2}\,[5.17\mathord{\times}10^{2},5.51\mathord{\times}10^{2}]$ & $5.42\mathord{\times}10^{-4}\,[5.33\mathord{\times}10^{-4},5.73\mathord{\times}10^{-4}]$ & $9.83\mathord{\times}10^{5}\,[9.38\mathord{\times}10^{5},9.94\mathord{\times}10^{5}]$ \\
$10^{7}$ & $\approx5.27\mathord{\times}10^{3}$ & $5.53\mathord{\times}10^{-4}\,[5.36\mathord{\times}10^{-4},5.70\mathord{\times}10^{-4}]$ & $\approx9.58\mathord{\times}10^{6}$ \\
\bottomrule
\end{tabular}
\end{table}

\begin{table}[htbp]
\centering
\scriptsize
\caption{Simulator time for $(H,S,A)=(10, 15, 10)$, in seconds. Entries are medians with 95\% percentile bootstrap CIs over ten MDP instances; speedup is the median of paired time ratios. Conventional entries beyond the last measured $10^{6}$ budget are marked $\approx$ and extrapolated by a factor of ten per additional power of ten.}
\label{tab:planning-timing-s4}
\setlength{\tabcolsep}{3pt}
\begin{tabular}{rlll}
\toprule
Budget & Conventional (s) & Fast (s) & Speedup \\
\midrule
$10^{1}$ & $2.69\mathord{\times}10^{-2}\,[2.26\mathord{\times}10^{-2},2.76\mathord{\times}10^{-2}]$ & $3.17\mathord{\times}10^{-3}\,[2.62\mathord{\times}10^{-3},3.29\mathord{\times}10^{-3}]$ & $8.53\,[8.36,8.58]$ \\
$10^{2}$ & $2.57\mathord{\times}10^{-1}\,[2.21\mathord{\times}10^{-1},2.67\mathord{\times}10^{-1}]$ & $4.58\mathord{\times}10^{-3}\,[3.74\mathord{\times}10^{-3},4.75\mathord{\times}10^{-3}]$ & $5.70\mathord{\times}10^{1}\,[5.63\mathord{\times}10^{1},5.85\mathord{\times}10^{1}]$ \\
$10^{3}$ & $2.57\,[2.27,2.72]$ & $5.19\mathord{\times}10^{-3}\,[4.23\mathord{\times}10^{-3},5.35\mathord{\times}10^{-3}]$ & $5.13\mathord{\times}10^{2}\,[5.02\mathord{\times}10^{2},5.25\mathord{\times}10^{2}]$ \\
$10^{4}$ & $2.57\mathord{\times}10^{1}\,[2.19\mathord{\times}10^{1},2.71\mathord{\times}10^{1}]$ & $4.26\mathord{\times}10^{-3}\,[3.47\mathord{\times}10^{-3},4.40\mathord{\times}10^{-3}]$ & $6.27\mathord{\times}10^{3}\,[6.13\mathord{\times}10^{3},6.30\mathord{\times}10^{3}]$ \\
$10^{5}$ & $2.64\mathord{\times}10^{2}\,[2.23\mathord{\times}10^{2},2.71\mathord{\times}10^{2}]$ & $3.85\mathord{\times}10^{-3}\,[3.17\mathord{\times}10^{-3},4.00\mathord{\times}10^{-3}]$ & $6.90\mathord{\times}10^{4}\,[6.77\mathord{\times}10^{4},7.01\mathord{\times}10^{4}]$ \\
$10^{6}$ & $2.61\mathord{\times}10^{3}\,[2.16\mathord{\times}10^{3},2.69\mathord{\times}10^{3}]$ & $3.71\mathord{\times}10^{-3}\,[3.03\mathord{\times}10^{-3},3.86\mathord{\times}10^{-3}]$ & $6.98\mathord{\times}10^{5}\,[6.94\mathord{\times}10^{5},7.13\mathord{\times}10^{5}]$ \\
$10^{7}$ & $\approx2.61\mathord{\times}10^{4}$ & $3.68\mathord{\times}10^{-3}\,[3.02\mathord{\times}10^{-3},3.83\mathord{\times}10^{-3}]$ & $\approx7.04\mathord{\times}10^{6}$ \\
\bottomrule
\end{tabular}
\end{table}

\newpage
\subsection{Offline Vanilla Pessimistic Value Iteration}
\label{exp-vpvi}
We use a uniform behavior policy over actions to generate an offline dataset. Under conventional simulation, the offline dataset is generated trajectory by trajectory. In contrast, the direct batch simulation mode of our framework samples the aggregate Markov flow for the full batch of $K$ trajectories and directly recovers the visitation and transition counts required by VPVI. The same VPVI update is then applied to the resulting empirical statistics. The pessimism multiplier $c$ and the log-confidence factor $\iota$ in \Cref{alg:review-vpvi} are both set to one.

Learning performance is evaluated by the policy gap in the true MDP, averaged over the uniform initial-state distribution:
\[
G_\pi(K)=\mathbb E_{s\sim\nu}\big[V_1^\star(s)-V_1^{\widehat\pi}(s)\big].
\]
The policy-gap results are reported in Tables \ref{tab:vpvi-final_policy_gap-s1}--\ref{tab:vpvi-final_policy_gap-s4}. Across the reported budgets, the conventional and fast simulators produce nearly identical policy gaps at the same $K$: for most settings, the confidence intervals for their differences either contain zero or collapse to zero (31 out of 32 settings). This behavior is consistent with the distributional equivalence established in Corollary~\ref{cor:flow}. Indeed, since VPVI depends on the offline dataset only through the visitation and transition statistics contained in the Markov flow, equality in distribution of the generated flows induces equality in distribution of the resulting learned policies and hence of their policy gaps.

The policy gap, however, can become exactly zero at relatively modest budgets, particularly for smaller MDPs, once the learned policy is already optimal. Beyond this point, it no longer distinguishes between empirical models of different accuracies. We therefore additionally examine the visitation-weighted transition error
\[
E_P(K)=\frac{\sum_{h,s,a}C_h(s,a)\left\|\widehat P_h(\cdot\mid s,a)-P_h(\cdot\mid s,a)\right\|_1}{\sum_{h,s,a}C_h(s,a)}.\]
This quantity measures the discrepancy between the empirical transition model passed to VPVI and the true transition model, weighted by how frequently each row is represented in the sampled dataset. It therefore provides a more fine-grained diagnostic of the sampled empirical model, complementing the policy gap, which measures the quality of the final learned policy.

The transition-error results are reported in Tables \ref{tab:vpvi-weighted_transition_l1_error-s1}--\ref{tab:vpvi-weighted_transition_l1_error-s4}. The conventional and fast simulators exhibit very similar weighted transition errors across budgets, with almost all confidence intervals for their differences containing zero. Since $E_P(K)$ is itself a function of the sampled visitation and transition counts, this agreement is also consistent with the distributional equivalence in Corollary~\ref{cor:flow}.

The timing results in Tables~\ref{tab:vpvi-timing-s1}--\ref{tab:vpvi-timing-s4} show different computational scaling. Conventional simulation generates the $K$ trajectories individually, and its runtime grows approximately linearly with $K$. In contrast, our fast method uses direct batch simulation to generate the aggregate Markov flow, eliminating the simulator-side dependence on $K$. Beyond the small-budget regime, its runtime remains approximately flat over moderate-to-large budgets, consistent with Proposition~\ref{prop:direct-complexity}. Moreover, the learner time remains below $0.002$ seconds in all tested configurations and is substantially smaller than the conventional simulation time at large scales and budgets. This shows that simulation is the dominant computational bottleneck in this regime, making the reduction in simulator-side complexity practically significant.

\begin{table}[htbp]
\centering
\tiny
\caption{Expected policy gap for $(H,S,A)=(2, 3, 3)$. Entries report the mean and 95\% percentile bootstrap confidence interval over ten independent simulation seeds on a fixed MDP. The difference is defined as fast minus conventional.}
\label{tab:vpvi-final_policy_gap-s1}
\setlength{\tabcolsep}{3pt}
\begin{tabular}{rlll}
\toprule
Budget & Conventional & Fast & Difference \\
\midrule
$10^{1}$ & $9.20\mathord{\times}10^{-1}\,[9.20\mathord{\times}10^{-1},9.20\mathord{\times}10^{-1}]$ & $9.20\mathord{\times}10^{-1}\,[9.20\mathord{\times}10^{-1},9.20\mathord{\times}10^{-1}]$ & $0\,[0,0]$ \\
$10^{2}$ & $2.89\mathord{\times}10^{-2}\,[9.48\mathord{\times}10^{-3},5.29\mathord{\times}10^{-2}]$ & $6.39\mathord{\times}10^{-2}\,[4.98\mathord{\times}10^{-2},7.89\mathord{\times}10^{-2}]$ & $3.50\mathord{\times}10^{-2}\,[6.42\mathord{\times}10^{-3},6.15\mathord{\times}10^{-2}]$ \\
$10^{3}$ & $0\,[0,0]$ & $0\,[0,0]$ & $0\,[0,0]$ \\
$10^{4}$ & $0\,[0,0]$ & $0\,[0,0]$ & $0\,[0,0]$ \\
$10^{5}$ & $0\,[0,0]$ & $0\,[0,0]$ & $0\,[0,0]$ \\
$10^{6}$ & $0\,[0,0]$ & $0\,[0,0]$ & $0\,[0,0]$ \\
$10^{7}$ & $0\,[0,0]$ & $0\,[0,0]$ & $0\,[0,0]$ \\
$10^{8}$ & $0\,[0,0]$ & $0\,[0,0]$ & $0\,[0,0]$ \\
\bottomrule
\end{tabular}
\end{table}

\begin{table}[htbp]
\centering
\tiny
\caption{Expected policy gap for $(H,S,A)=(5, 5, 5)$. Entries report the mean and 95\% percentile bootstrap confidence interval over ten independent simulation seeds on a fixed MDP. The difference is defined as fast minus conventional.}
\label{tab:vpvi-final_policy_gap-s2}
\setlength{\tabcolsep}{2pt}
\begin{tabular}{rlll}
\toprule
Budget & Conventional & Fast & Difference \\
\midrule
$10^{1}$ & $1.79\,[1.79,1.79]$ & $1.79\,[1.79,1.79]$ & $0\,[0,0]$ \\
$10^{2}$ & $1.79\,[1.79,1.79]$ & $1.79\,[1.79,1.79]$ & $0\,[0,0]$ \\
$10^{3}$ & $1.21\mathord{\times}10^{-1}\,[9.11\mathord{\times}10^{-2},1.49\mathord{\times}10^{-1}]$ & $9.35\mathord{\times}10^{-2}\,[5.11\mathord{\times}10^{-2},1.39\mathord{\times}10^{-1}]$ & $-2.72\mathord{\times}10^{-2}\,[-7.89\mathord{\times}10^{-2},2.71\mathord{\times}10^{-2}]$ \\
$10^{4}$ & $4.14\mathord{\times}10^{-4}\,[1.04\mathord{\times}10^{-4},7.25\mathord{\times}10^{-4}]$ & $3.11\mathord{\times}10^{-4}\,[0,6.21\mathord{\times}10^{-4}]$ & $-1.04\mathord{\times}10^{-4}\,[-5.18\mathord{\times}10^{-4},3.11\mathord{\times}10^{-4}]$ \\
$10^{5}$ & $0\,[0,0]$ & $0\,[0,0]$ & $0\,[0,0]$ \\
$10^{6}$ & $0\,[0,0]$ & $0\,[0,0]$ & $0\,[0,0]$ \\
$10^{7}$ & $0\,[0,0]$ & $0\,[0,0]$ & $0\,[0,0]$ \\
$10^{8}$ & $0\,[0,0]$ & $0\,[0,0]$ & $0\,[0,0]$ \\
\bottomrule
\end{tabular}
\end{table}

\begin{table}[htbp]
\centering
\tiny
\caption{Expected policy gap for $(H,S,A)=(7, 8, 6)$. Entries report the mean and 95\% percentile bootstrap confidence interval over ten independent simulation seeds on a fixed MDP. The difference is defined as fast minus conventional.}
\label{tab:vpvi-final_policy_gap-s3}
\setlength{\tabcolsep}{2pt}
\begin{tabular}{rlll}
\toprule
Budget & Conventional & Fast & Difference \\
\midrule
$10^{1}$ & $2.60\,[2.60,2.60]$ & $2.60\,[2.60,2.60]$ & $0\,[0,0]$ \\
$10^{2}$ & $2.60\,[2.60,2.60]$ & $2.60\,[2.60,2.60]$ & $0\,[0,0]$ \\
$10^{3}$ & $2.60\,[2.60,2.60]$ & $2.60\,[2.60,2.60]$ & $0\,[0,0]$ \\
$10^{4}$ & $8.57\mathord{\times}10^{-3}\,[5.52\mathord{\times}10^{-3},1.18\mathord{\times}10^{-2}]$ & $8.36\mathord{\times}10^{-3}\,[5.73\mathord{\times}10^{-3},1.11\mathord{\times}10^{-2}]$ & $-2.10\mathord{\times}10^{-4}\,[-4.35\mathord{\times}10^{-3},3.82\mathord{\times}10^{-3}]$ \\
$10^{5}$ & $0\,[0,0]$ & $0\,[0,0]$ & $0\,[0,0]$ \\
$10^{6}$ & $0\,[0,0]$ & $0\,[0,0]$ & $0\,[0,0]$ \\
$10^{7}$ & $0\,[0,0]$ & $0\,[0,0]$ & $0\,[0,0]$ \\
$10^{8}$ & $0\,[0,0]$ & $0\,[0,0]$ & $0\,[0,0]$ \\
\bottomrule
\end{tabular}
\end{table}

\begin{table}[htbp]
\centering
\tiny
\caption{Expected policy gap for $(H,S,A)=(10, 15, 10)$. Entries report the mean and 95\% percentile bootstrap confidence interval over ten independent simulation seeds on a fixed MDP. The difference is defined as fast minus conventional.}
\label{tab:vpvi-final_policy_gap-s4}
\setlength{\tabcolsep}{2pt}
\begin{tabular}{rlll}
\toprule
Budget & Conventional & Fast & Difference \\
\midrule
$10^{1}$ & $4.06\,[4.06,4.06]$ & $4.06\,[4.06,4.06]$ & $0\,[0,0]$ \\
$10^{2}$ & $4.06\,[4.06,4.06]$ & $4.06\,[4.06,4.06]$ & $0\,[0,0]$ \\
$10^{3}$ & $4.06\,[4.06,4.06]$ & $4.06\,[4.06,4.06]$ & $0\,[0,0]$ \\
$10^{4}$ & $4.06\,[4.06,4.06]$ & $4.06\,[4.06,4.06]$ & $0\,[0,0]$ \\
$10^{5}$ & $4.01\mathord{\times}10^{-3}\,[2.93\mathord{\times}10^{-3},5.13\mathord{\times}10^{-3}]$ & $4.25\mathord{\times}10^{-3}\,[3.26\mathord{\times}10^{-3},5.27\mathord{\times}10^{-3}]$ & $2.32\mathord{\times}10^{-4}\,[-1.19\mathord{\times}10^{-3},1.69\mathord{\times}10^{-3}]$ \\
$10^{6}$ & $1.62\mathord{\times}10^{-4}\,[7.37\mathord{\times}10^{-5},2.54\mathord{\times}10^{-4}]$ & $6.09\mathord{\times}10^{-5}\,[2.39\mathord{\times}10^{-5},1.17\mathord{\times}10^{-4}]$ & $-1.01\mathord{\times}10^{-4}\,[-2.06\mathord{\times}10^{-4},1.16\mathord{\times}10^{-6}]$ \\
$10^{7}$ & $6.82\mathord{\times}10^{-6}\,[0,1.70\mathord{\times}10^{-5}]$ & $3.41\mathord{\times}10^{-6}\,[0,1.02\mathord{\times}10^{-5}]$ & $-3.41\mathord{\times}10^{-6}\,[-1.36\mathord{\times}10^{-5},6.82\mathord{\times}10^{-6}]$ \\
$10^{8}$ & $0\,[0,0]$ & $0\,[0,0]$ & $0\,[0,0]$ \\
\bottomrule
\end{tabular}
\end{table}

\begin{table}[htbp]
\centering
\tiny
\caption{Visit-weighted transition $L_1$ error for $(H,S,A)=(2, 3, 3)$. Entries report the mean and 95\% percentile bootstrap confidence interval over ten independent simulation seeds on a fixed MDP. The error difference is defined as fast minus conventional.}
\label{tab:vpvi-weighted_transition_l1_error-s1}
\setlength{\tabcolsep}{2pt}
\begin{tabular}{rlll}
\toprule
Budget & Conventional & Fast & Difference \\
\midrule
$10^{1}$ & $7.64\mathord{\times}10^{-1}\,[6.87\mathord{\times}10^{-1},8.42\mathord{\times}10^{-1}]$ & $8.20\mathord{\times}10^{-1}\,[7.24\mathord{\times}10^{-1},9.13\mathord{\times}10^{-1}]$ & $5.62\mathord{\times}10^{-2}\,[-6.30\mathord{\times}10^{-2},1.77\mathord{\times}10^{-1}]$ \\
$10^{2}$ & $2.64\mathord{\times}10^{-1}\,[2.40\mathord{\times}10^{-1},2.90\mathord{\times}10^{-1}]$ & $2.91\mathord{\times}10^{-1}\,[2.68\mathord{\times}10^{-1},3.12\mathord{\times}10^{-1}]$ & $2.65\mathord{\times}10^{-2}\,[-7.67\mathord{\times}10^{-3},5.93\mathord{\times}10^{-2}]$ \\
$10^{3}$ & $8.54\mathord{\times}10^{-2}\,[7.81\mathord{\times}10^{-2},9.21\mathord{\times}10^{-2}]$ & $9.52\mathord{\times}10^{-2}\,[9.01\mathord{\times}10^{-2},1.00\mathord{\times}10^{-1}]$ & $9.84\mathord{\times}10^{-3}\,[1.32\mathord{\times}10^{-3},1.85\mathord{\times}10^{-2}]$ \\
$10^{4}$ & $3.00\mathord{\times}10^{-2}\,[2.84\mathord{\times}10^{-2},3.15\mathord{\times}10^{-2}]$ & $2.94\mathord{\times}10^{-2}\,[2.73\mathord{\times}10^{-2},3.15\mathord{\times}10^{-2}]$ & $-5.82\mathord{\times}10^{-4}\,[-3.22\mathord{\times}10^{-3},2.06\mathord{\times}10^{-3}]$ \\
$10^{5}$ & $9.12\mathord{\times}10^{-3}\,[8.31\mathord{\times}10^{-3},9.96\mathord{\times}10^{-3}]$ & $8.42\mathord{\times}10^{-3}\,[7.79\mathord{\times}10^{-3},9.10\mathord{\times}10^{-3}]$ & $-6.99\mathord{\times}10^{-4}\,[-1.73\mathord{\times}10^{-3},3.56\mathord{\times}10^{-4}]$ \\
$10^{6}$ & $2.76\mathord{\times}10^{-3}\,[2.57\mathord{\times}10^{-3},2.96\mathord{\times}10^{-3}]$ & $2.56\mathord{\times}10^{-3}\,[2.31\mathord{\times}10^{-3},2.84\mathord{\times}10^{-3}]$ & $-2.07\mathord{\times}10^{-4}\,[-5.31\mathord{\times}10^{-4},1.33\mathord{\times}10^{-4}]$ \\
\bottomrule
\end{tabular}
\end{table}

\begin{table}[htbp]
\centering
\tiny
\caption{Visit-weighted transition $L_1$ error for $(H,S,A)=(5, 5, 5)$. Entries report the mean and 95\% percentile bootstrap confidence interval over ten independent simulation seeds on a fixed MDP. The error difference is defined as fast minus conventional.}
\label{tab:vpvi-weighted_transition_l1_error-s2}
\setlength{\tabcolsep}{2pt}
\begin{tabular}{rlll}
\toprule
Budget & Conventional & Fast & Difference \\
\midrule
$10^{1}$ & $1.23\,[1.19,1.27]$ & $1.24\,[1.21,1.26]$ & $3.90\mathord{\times}10^{-3}\,[-4.22\mathord{\times}10^{-2},5.25\mathord{\times}10^{-2}]$ \\
$10^{2}$ & $6.72\mathord{\times}10^{-1}\,[6.65\mathord{\times}10^{-1},6.78\mathord{\times}10^{-1}]$ & $6.90\mathord{\times}10^{-1}\,[6.76\mathord{\times}10^{-1},7.05\mathord{\times}10^{-1}]$ & $1.85\mathord{\times}10^{-2}\,[2.22\mathord{\times}10^{-3},3.42\mathord{\times}10^{-2}]$ \\
$10^{3}$ & $2.15\mathord{\times}10^{-1}\,[2.10\mathord{\times}10^{-1},2.20\mathord{\times}10^{-1}]$ & $2.20\mathord{\times}10^{-1}\,[2.17\mathord{\times}10^{-1},2.24\mathord{\times}10^{-1}]$ & $4.95\mathord{\times}10^{-3}\,[-1.14\mathord{\times}10^{-3},1.09\mathord{\times}10^{-2}]$ \\
$10^{4}$ & $7.10\mathord{\times}10^{-2}\,[6.96\mathord{\times}10^{-2},7.25\mathord{\times}10^{-2}]$ & $6.92\mathord{\times}10^{-2}\,[6.73\mathord{\times}10^{-2},7.11\mathord{\times}10^{-2}]$ & $-1.89\mathord{\times}10^{-3}\,[-4.21\mathord{\times}10^{-3},5.84\mathord{\times}10^{-4}]$ \\
$10^{5}$ & $2.18\mathord{\times}10^{-2}\,[2.13\mathord{\times}10^{-2},2.25\mathord{\times}10^{-2}]$ & $2.20\mathord{\times}10^{-2}\,[2.15\mathord{\times}10^{-2},2.25\mathord{\times}10^{-2}]$ & $1.42\mathord{\times}10^{-4}\,[-6.62\mathord{\times}10^{-4},9.08\mathord{\times}10^{-4}]$ \\
$10^{6}$ & $6.81\mathord{\times}10^{-3}\,[6.63\mathord{\times}10^{-3},6.99\mathord{\times}10^{-3}]$ & $6.75\mathord{\times}10^{-3}\,[6.61\mathord{\times}10^{-3},6.91\mathord{\times}10^{-3}]$ & $-5.27\mathord{\times}10^{-5}\,[-2.96\mathord{\times}10^{-4},1.80\mathord{\times}10^{-4}]$ \\
$10^{7}$ & $2.17\mathord{\times}10^{-3}\,[2.09\mathord{\times}10^{-3},2.24\mathord{\times}10^{-3}]$ & $2.18\mathord{\times}10^{-3}\,[2.15\mathord{\times}10^{-3},2.21\mathord{\times}10^{-3}]$ & $1.22\mathord{\times}10^{-5}\,[-6.61\mathord{\times}10^{-5},9.26\mathord{\times}10^{-5}]$ \\
\bottomrule
\end{tabular}
\end{table}

\begin{table}[htbp]
\centering
\tiny
\caption{Visit-weighted transition $L_1$ error for $(H,S,A)=(7, 8, 6)$. Entries report the mean and 95\% percentile bootstrap confidence interval over ten independent simulation seeds on a fixed MDP. The error difference is defined as fast minus conventional.}
\label{tab:vpvi-weighted_transition_l1_error-s3}
\setlength{\tabcolsep}{1.5pt}
\begin{tabular}{rlll}
\toprule
Budget & Conventional & Fast & Difference \\
\midrule
$10^{1}$ & $1.51\,[1.48,1.54]$ & $1.50\,[1.48,1.52]$ & $-1.14\mathord{\times}10^{-2}\,[-5.21\mathord{\times}10^{-2},2.70\mathord{\times}10^{-2}]$ \\
$10^{2}$ & $1.10\,[1.08,1.11]$ & $1.11\,[1.10,1.12]$ & $9.43\mathord{\times}10^{-3}\,[-9.93\mathord{\times}10^{-3},2.79\mathord{\times}10^{-2}]$ \\
$10^{3}$ & $4.03\mathord{\times}10^{-1}\,[3.98\mathord{\times}10^{-1},4.09\mathord{\times}10^{-1}]$ & $4.06\mathord{\times}10^{-1}\,[4.02\mathord{\times}10^{-1},4.11\mathord{\times}10^{-1}]$ & $2.97\mathord{\times}10^{-3}\,[-4.40\mathord{\times}10^{-3},1.02\mathord{\times}10^{-2}]$ \\
$10^{4}$ & $1.28\mathord{\times}10^{-1}\,[1.27\mathord{\times}10^{-1},1.29\mathord{\times}10^{-1}]$ & $1.29\mathord{\times}10^{-1}\,[1.28\mathord{\times}10^{-1},1.30\mathord{\times}10^{-1}]$ & $3.99\mathord{\times}10^{-4}\,[-8.50\mathord{\times}10^{-4},1.58\mathord{\times}10^{-3}]$ \\
$10^{5}$ & $4.12\mathord{\times}10^{-2}\,[4.07\mathord{\times}10^{-2},4.17\mathord{\times}10^{-2}]$ & $4.04\mathord{\times}10^{-2}\,[3.98\mathord{\times}10^{-2},4.09\mathord{\times}10^{-2}]$ & $-7.26\mathord{\times}10^{-4}\,[-1.54\mathord{\times}10^{-3},-1.61\mathord{\times}10^{-5}]$ \\
$10^{6}$ & $1.29\mathord{\times}10^{-2}\,[1.28\mathord{\times}10^{-2},1.30\mathord{\times}10^{-2}]$ & $1.28\mathord{\times}10^{-2}\,[1.27\mathord{\times}10^{-2},1.30\mathord{\times}10^{-2}]$ & $-2.16\mathord{\times}10^{-5}\,[-2.17\mathord{\times}10^{-4},1.64\mathord{\times}10^{-4}]$ \\
$10^{7}$ & $4.03\mathord{\times}10^{-3}\,[3.99\mathord{\times}10^{-3},4.07\mathord{\times}10^{-3}]$ & $4.06\mathord{\times}10^{-3}\,[4.02\mathord{\times}10^{-3},4.09\mathord{\times}10^{-3}]$ & $3.19\mathord{\times}10^{-5}\,[-1.88\mathord{\times}10^{-5},8.16\mathord{\times}10^{-5}]$ \\
\bottomrule
\end{tabular}
\end{table}

\begin{table}[htbp]
\centering
\tiny
\caption{Visit-weighted transition $L_1$ error for $(H,S,A)=(10, 15, 10)$. Entries report the mean and 95\% percentile bootstrap confidence interval over ten independent simulation seeds on a fixed MDP. The error difference is defined as fast minus conventional.}
\label{tab:vpvi-weighted_transition_l1_error-s4}
\setlength{\tabcolsep}{2pt}
\begin{tabular}{rlll}
\toprule
Budget & Conventional & Fast & Difference \\
\midrule
$10^{1}$ & $1.74\,[1.74,1.75]$ & $1.73\,[1.72,1.74]$ & $-9.56\mathord{\times}10^{-3}\,[-2.11\mathord{\times}10^{-2},1.65\mathord{\times}10^{-3}]$ \\
$10^{2}$ & $1.62\,[1.62,1.62]$ & $1.62\,[1.61,1.62]$ & $-4.61\mathord{\times}10^{-3}\,[-9.54\mathord{\times}10^{-3},1.63\mathord{\times}10^{-4}]$ \\
$10^{3}$ & $9.70\mathord{\times}10^{-1}\,[9.67\mathord{\times}10^{-1},9.73\mathord{\times}10^{-1}]$ & $9.72\mathord{\times}10^{-1}\,[9.70\mathord{\times}10^{-1},9.75\mathord{\times}10^{-1}]$ & $2.89\mathord{\times}10^{-3}\,[-1.85\mathord{\times}10^{-3},6.95\mathord{\times}10^{-3}]$ \\
$10^{4}$ & $3.20\mathord{\times}10^{-1}\,[3.19\mathord{\times}10^{-1},3.21\mathord{\times}10^{-1}]$ & $3.21\mathord{\times}10^{-1}\,[3.20\mathord{\times}10^{-1},3.22\mathord{\times}10^{-1}]$ & $7.83\mathord{\times}10^{-4}\,[-4.90\mathord{\times}10^{-4},2.21\mathord{\times}10^{-3}]$ \\
$10^{5}$ & $1.02\mathord{\times}10^{-1}\,[1.01\mathord{\times}10^{-1},1.02\mathord{\times}10^{-1}]$ & $1.02\mathord{\times}10^{-1}\,[1.01\mathord{\times}10^{-1},1.02\mathord{\times}10^{-1}]$ & $1.04\mathord{\times}10^{-4}\,[-4.42\mathord{\times}10^{-4},6.80\mathord{\times}10^{-4}]$ \\
$10^{6}$ & $3.23\mathord{\times}10^{-2}\,[3.22\mathord{\times}10^{-2},3.24\mathord{\times}10^{-2}]$ & $3.21\mathord{\times}10^{-2}\,[3.20\mathord{\times}10^{-2},3.22\mathord{\times}10^{-2}]$ & $-1.44\mathord{\times}10^{-4}\,[-2.89\mathord{\times}10^{-4},2.50\mathord{\times}10^{-6}]$ \\
$10^{7}$ & $1.02\mathord{\times}10^{-2}\,[1.01\mathord{\times}10^{-2},1.02\mathord{\times}10^{-2}]$ & $1.01\mathord{\times}10^{-2}\,[1.01\mathord{\times}10^{-2},1.02\mathord{\times}10^{-2}]$ & $-1.36\mathord{\times}10^{-5}\,[-6.97\mathord{\times}10^{-5},4.30\mathord{\times}10^{-5}]$ \\
\bottomrule
\end{tabular}
\end{table}

\begin{table}[htbp]
\centering
\scriptsize
\caption{Simulator time for $(H,S,A)=(2, 3, 3)$, in seconds. Entries are medians with 95\% percentile bootstrap CIs over ten MDP instances; speedup is the median of paired time ratios. Conventional entries beyond the last measured $10^{8}$ budget are marked $\approx$ and extrapolated by a factor of ten per additional power of ten.}
\label{tab:vpvi-timing-s1}
\setlength{\tabcolsep}{2pt}
\begin{tabular}{rlll}
\toprule
Budget & Conventional (s) & Fast (s) & Speedup \\
\midrule
$10^{1}$ & $1.05\mathord{\times}10^{-4}\,[9.41\mathord{\times}10^{-5},1.07\mathord{\times}10^{-4}]$ & $4.71\mathord{\times}10^{-5}\,[4.15\mathord{\times}10^{-5},4.88\mathord{\times}10^{-5}]$ & $2.21\,[2.15,2.25]$ \\
$10^{2}$ & $9.80\mathord{\times}10^{-4}\,[8.74\mathord{\times}10^{-4},1.00\mathord{\times}10^{-3}]$ & $5.95\mathord{\times}10^{-5}\,[5.13\mathord{\times}10^{-5},5.98\mathord{\times}10^{-5}]$ & $1.68\mathord{\times}10^{1}\,[1.64\mathord{\times}10^{1},1.70\mathord{\times}10^{1}]$ \\
$10^{3}$ & $9.93\mathord{\times}10^{-3}\,[8.72\mathord{\times}10^{-3},9.98\mathord{\times}10^{-3}]$ & $6.11\mathord{\times}10^{-5}\,[5.33\mathord{\times}10^{-5},6.20\mathord{\times}10^{-5}]$ & $1.63\mathord{\times}10^{2}\,[1.59\mathord{\times}10^{2},1.64\mathord{\times}10^{2}]$ \\
$10^{4}$ & $1.00\mathord{\times}10^{-1}\,[8.77\mathord{\times}10^{-2},1.01\mathord{\times}10^{-1}]$ & $5.95\mathord{\times}10^{-5}\,[5.17\mathord{\times}10^{-5},6.06\mathord{\times}10^{-5}]$ & $1.67\mathord{\times}10^{3}\,[1.64\mathord{\times}10^{3},1.71\mathord{\times}10^{3}]$ \\
$10^{5}$ & $9.74\mathord{\times}10^{-1}\,[8.81\mathord{\times}10^{-1},9.97\mathord{\times}10^{-1}]$ & $5.88\mathord{\times}10^{-5}\,[5.08\mathord{\times}10^{-5},5.95\mathord{\times}10^{-5}]$ & $1.67\mathord{\times}10^{4}\,[1.64\mathord{\times}10^{4},1.74\mathord{\times}10^{4}]$ \\
$10^{6}$ & $9.82\,[8.92,1.01\mathord{\times}10^{1}]$ & $5.89\mathord{\times}10^{-5}\,[5.08\mathord{\times}10^{-5},5.94\mathord{\times}10^{-5}]$ & $1.70\mathord{\times}10^{5}\,[1.68\mathord{\times}10^{5},1.74\mathord{\times}10^{5}]$ \\
$10^{7}$ & $9.75\mathord{\times}10^{1}\,[8.76\mathord{\times}10^{1},9.91\mathord{\times}10^{1}]$ & $5.83\mathord{\times}10^{-5}\,[5.09\mathord{\times}10^{-5},5.90\mathord{\times}10^{-5}]$ & $1.67\mathord{\times}10^{6}\,[1.66\mathord{\times}10^{6},1.74\mathord{\times}10^{6}]$ \\
\bottomrule
\end{tabular}
\end{table}

\begin{table}[htbp]
\centering
\scriptsize
\caption{Simulator time for $(H,S,A)=(5, 5, 5)$, in seconds. Entries are medians with 95\% percentile bootstrap CIs over ten MDP instances; speedup is the median of paired time ratios. Conventional entries beyond the last measured $10^{8}$ budget are marked $\approx$ and extrapolated by a factor of ten per additional power of ten.}
\label{tab:vpvi-timing-s2}
\setlength{\tabcolsep}{3pt}
\begin{tabular}{rlll}
\toprule
Budget & Conventional (s) & Fast (s) & Speedup \\
\midrule
$10^{1}$ & $1.94\mathord{\times}10^{-4}\,[1.92\mathord{\times}10^{-4},2.09\mathord{\times}10^{-4}]$ & $1.29\mathord{\times}10^{-4}\,[1.20\mathord{\times}10^{-4},1.45\mathord{\times}10^{-4}]$ & $1.49\,[1.42,1.63]$ \\
$10^{2}$ & $1.88\mathord{\times}10^{-3}\,[1.86\mathord{\times}10^{-3},2.05\mathord{\times}10^{-3}]$ & $2.70\mathord{\times}10^{-4}\,[2.62\mathord{\times}10^{-4},3.06\mathord{\times}10^{-4}]$ & $6.81\,[6.62,7.12]$ \\
$10^{3}$ & $1.86\mathord{\times}10^{-2}\,[1.84\mathord{\times}10^{-2},2.03\mathord{\times}10^{-2}]$ & $3.00\mathord{\times}10^{-4}\,[2.83\mathord{\times}10^{-4},3.31\mathord{\times}10^{-4}]$ & $6.18\mathord{\times}10^{1}\,[6.05\mathord{\times}10^{1},6.51\mathord{\times}10^{1}]$ \\
$10^{4}$ & $1.87\mathord{\times}10^{-1}\,[1.83\mathord{\times}10^{-1},2.04\mathord{\times}10^{-1}]$ & $3.09\mathord{\times}10^{-4}\,[2.89\mathord{\times}10^{-4},3.41\mathord{\times}10^{-4}]$ & $6.23\mathord{\times}10^{2}\,[5.95\mathord{\times}10^{2},6.32\mathord{\times}10^{2}]$ \\
$10^{5}$ & $1.89\,[1.85,2.01]$ & $2.90\mathord{\times}10^{-4}\,[2.78\mathord{\times}10^{-4},3.34\mathord{\times}10^{-4}]$ & $6.43\mathord{\times}10^{3}\,[6.04\mathord{\times}10^{3},6.79\mathord{\times}10^{3}]$ \\
$10^{6}$ & $1.91\mathord{\times}10^{1}\,[1.86\mathord{\times}10^{1},2.06\mathord{\times}10^{1}]$ & $2.87\mathord{\times}10^{-4}\,[2.71\mathord{\times}10^{-4},3.27\mathord{\times}10^{-4}]$ & $6.60\mathord{\times}10^{4}\,[6.25\mathord{\times}10^{4},6.79\mathord{\times}10^{4}]$ \\
$10^{7}$ & $1.86\mathord{\times}10^{2}\,[1.82\mathord{\times}10^{2},2.03\mathord{\times}10^{2}]$ & $2.83\mathord{\times}10^{-4}\,[2.73\mathord{\times}10^{-4},3.21\mathord{\times}10^{-4}]$ & $6.49\mathord{\times}10^{5}\,[6.18\mathord{\times}10^{5},6.75\mathord{\times}10^{5}]$ \\
\bottomrule
\end{tabular}
\end{table}

\begin{table}[htbp]
\centering
\scriptsize
\caption{Simulator time for $(H,S,A)=(7, 8, 6)$, in seconds. Entries are medians with 95\% percentile bootstrap CIs over ten MDP instances; speedup is the median of paired time ratios. Conventional entries beyond the last measured $10^{8}$ budget are marked $\approx$ and extrapolated by a factor of ten per additional power of ten.}
\label{tab:vpvi-timing-s3}
\setlength{\tabcolsep}{3pt}
\begin{tabular}{rlll}
\toprule
Budget & Conventional (s) & Fast (s) & Speedup \\
\midrule
$10^{1}$ & $3.07\mathord{\times}10^{-4}\,[2.77\mathord{\times}10^{-4},3.08\mathord{\times}10^{-4}]$ & $2.72\mathord{\times}10^{-4}\,[2.54\mathord{\times}10^{-4},2.82\mathord{\times}10^{-4}]$ & $1.10\,[1.07,1.13]$ \\
$10^{2}$ & $3.01\mathord{\times}10^{-3}\,[2.70\mathord{\times}10^{-3},3.03\mathord{\times}10^{-3}]$ & $7.98\mathord{\times}10^{-4}\,[7.44\mathord{\times}10^{-4},8.06\mathord{\times}10^{-4}]$ & $3.73\,[3.66,3.78]$ \\
$10^{3}$ & $3.01\mathord{\times}10^{-2}\,[2.69\mathord{\times}10^{-2},3.02\mathord{\times}10^{-2}]$ & $1.04\mathord{\times}10^{-3}\,[9.73\mathord{\times}10^{-4},1.05\mathord{\times}10^{-3}]$ & $2.87\mathord{\times}10^{1}\,[2.77\mathord{\times}10^{1},2.90\mathord{\times}10^{1}]$ \\
$10^{4}$ & $3.00\mathord{\times}10^{-1}\,[2.67\mathord{\times}10^{-1},3.01\mathord{\times}10^{-1}]$ & $1.17\mathord{\times}10^{-3}\,[1.10\mathord{\times}10^{-3},1.19\mathord{\times}10^{-3}]$ & $2.52\mathord{\times}10^{2}\,[2.43\mathord{\times}10^{2},2.55\mathord{\times}10^{2}]$ \\
$10^{5}$ & $3.05\,[2.68,3.06]$ & $1.07\mathord{\times}10^{-3}\,[1.01\mathord{\times}10^{-3},1.08\mathord{\times}10^{-3}]$ & $2.79\mathord{\times}10^{3}\,[2.66\mathord{\times}10^{3},2.85\mathord{\times}10^{3}]$ \\
$10^{6}$ & $3.05\mathord{\times}10^{1}\,[2.72\mathord{\times}10^{1},3.06\mathord{\times}10^{1}]$ & $1.01\mathord{\times}10^{-3}\,[9.46\mathord{\times}10^{-4},1.01\mathord{\times}10^{-3}]$ & $3.01\mathord{\times}10^{4}\,[2.94\mathord{\times}10^{4},3.03\mathord{\times}10^{4}]$ \\
$10^{7}$ & $3.04\mathord{\times}10^{2}\,[2.64\mathord{\times}10^{2},3.06\mathord{\times}10^{2}]$ & $9.94\mathord{\times}10^{-4}\,[9.33\mathord{\times}10^{-4},9.99\mathord{\times}10^{-4}]$ & $3.02\mathord{\times}10^{5}\,[2.83\mathord{\times}10^{5},3.07\mathord{\times}10^{5}]$ \\
\bottomrule
\end{tabular}
\end{table}

\begin{table}[htbp]
\centering
\tiny
\caption{Simulator time for $(H,S,A)=(10, 15, 10)$, in seconds. Entries are medians with 95\% percentile bootstrap CIs over ten MDP instances; speedup is the median of paired time ratios. Conventional entries beyond the last measured $10^{8}$ budget are marked $\approx$ and extrapolated by a factor of ten per additional power of ten.}
\label{tab:vpvi-timing-s4}
\setlength{\tabcolsep}{3pt}
\begin{tabular}{rlll}
\toprule
Budget & Conventional (s) & Fast (s) & Speedup \\
\midrule
$10^{1}$ & $3.43\mathord{\times}10^{-4}\,[3.36\mathord{\times}10^{-4},4.17\mathord{\times}10^{-4}]$ & $5.90\mathord{\times}10^{-4}\,[5.80\mathord{\times}10^{-4},6.22\mathord{\times}10^{-4}]$ & $5.98\mathord{\times}10^{-1}\,[5.72\mathord{\times}10^{-1},6.62\mathord{\times}10^{-1}]$ \\
$10^{2}$ & $3.34\mathord{\times}10^{-3}\,[3.26\mathord{\times}10^{-3},3.82\mathord{\times}10^{-3}]$ & $2.09\mathord{\times}10^{-3}\,[2.05\mathord{\times}10^{-3},2.18\mathord{\times}10^{-3}]$ & $1.64\,[1.59,1.77]$ \\
$10^{3}$ & $3.32\mathord{\times}10^{-2}\,[3.25\mathord{\times}10^{-2},3.93\mathord{\times}10^{-2}]$ & $4.27\mathord{\times}10^{-3}\,[4.17\mathord{\times}10^{-3},4.40\mathord{\times}10^{-3}]$ & $8.02\,[7.76,8.75]$ \\
$10^{4}$ & $3.33\mathord{\times}10^{-1}\,[3.27\mathord{\times}10^{-1},3.93\mathord{\times}10^{-1}]$ & $5.53\mathord{\times}10^{-3}\,[5.37\mathord{\times}10^{-3},5.65\mathord{\times}10^{-3}]$ & $6.23\mathord{\times}10^{1}\,[6.02\mathord{\times}10^{1},6.75\mathord{\times}10^{1}]$ \\
$10^{5}$ & $3.38\,[3.29,3.84]$ & $6.59\mathord{\times}10^{-3}\,[6.34\mathord{\times}10^{-3},6.73\mathord{\times}10^{-3}]$ & $5.29\mathord{\times}10^{2}\,[5.05\mathord{\times}10^{2},5.71\mathord{\times}10^{2}]$ \\
$10^{6}$ & $3.43\mathord{\times}10^{1}\,[3.39\mathord{\times}10^{1},4.04\mathord{\times}10^{1}]$ & $5.53\mathord{\times}10^{-3}\,[5.19\mathord{\times}10^{-3},5.70\mathord{\times}10^{-3}]$ & $6.67\mathord{\times}10^{3}\,[6.22\mathord{\times}10^{3},7.03\mathord{\times}10^{3}]$ \\
$10^{7}$ & $3.49\mathord{\times}10^{2}\,[3.43\mathord{\times}10^{2},3.91\mathord{\times}10^{2}]$ & $4.97\mathord{\times}10^{-3}\,[4.71\mathord{\times}10^{-3},5.26\mathord{\times}10^{-3}]$ & $7.33\mathord{\times}10^{4}\,[6.95\mathord{\times}10^{4},7.62\mathord{\times}10^{4}]$ \\
\bottomrule
\end{tabular}
\end{table}

\subsection{Online UCB-Advantage}
\label{exp-ucb}
We next consider online UCB-Advantage, where the policy changes at data-dependent update boundaries. Under conventional simulation, episodes are generated sequentially, and learner updates are triggered as soon as the corresponding stage or reference-value thresholds are reached. In contrast, the adaptive mode of our fast simulation framework generates an aggregate Markov flow for a candidate batch of episodes under the current policy and uses hypergeometric splitting to locate the first update boundary. The aggregate statistics before the boundary are processed in batch, the boundary episode is handled according to the original learner update, and simulation then restarts under the updated policy. Between two consecutive update boundaries, the quantities entering a $Q$-update are determined by the learner state fixed at the beginning of the batch, including the relevant value functions and visit counts, together with the additive accumulators collected over the current stage. Hence, these quantities can be recovered from the aggregate flow without materializing the individual episodes. Both implementations exploit the deterministic action selection of UCB-Advantage when simulating transitions. The confidence factor $\iota$ and bonus confidence constants are fixed to one in \Cref{alg:review-ucb-advantage}, and the reference-update threshold is set to $HSA$.

Learning performance is evaluated with regret. For each batch, let $N_1(s)$ denote the number of episodes whose initial state is $s$. The regret contributed by this batch is $\sum_{s\in\cS} N_1(s)\bigl(V_1^\star(s)-V_1^\pi(s)\bigr)$, where $V_1^\star(s)$ is the optimal initial-step value in state $s$, and $V_1^\pi(s)$ is the value of policy $\pi$ used throughout that batch. At each batch endpoint, including a possible terminal partial batch, the cumulative regret is obtained by summing the regret contributions from all batches up to that point.

The regret results are shown in \Cref{fig:ucb-regret}. For each method, we use ten independent replicates on the same fixed MDP, each consisting of $K=10^8$ episodes, and record cumulative regret at every completed batch endpoint; straight segments are used only to interpolate between recorded endpoints for visualization. In each plot, the solid curve is the median over the ten replicates, with the shaded region showing the empirical 10th--90th percentile range. Across all four problem scales, the conventional and fast regret curves closely track each other over the common episode range. At the largest scale, $(H,S,A)=(10,15,10)$, the final median regrets at $K=10^8$ are approximately $511861$ and $511803$, respectively. This close agreement is consistent with the distributional equivalence of the adaptive fast simulation procedure and conventional online simulation established by \Cref{thm:adaptive}.

We further examine the number of processed adaptive batches in Figure 2. This statistic includes all processed batches, including a possible terminal partial batch at the prescribed episode budget, and provides a direct diagnostic of whether the fast simulator reproduces the batching behavior of the original algorithm. The conventional and fast methods again exhibit closely matched trajectories. At the largest scale, their final median batch counts are approximately 48773 and 48749, respectively.

The timing results in Tables~\ref{tab:ucb-timing-s1}--\ref{tab:ucb-timing-s4} illustrate the cost introduced by adaptive splitting. Unlike direct batch simulation, the adaptive method must repeatedly locate data-dependent boundaries and restart simulation after learner updates, so it is not uniformly faster at small budgets. Conventional simulation grows approximately linearly with $K$, whereas the fast method replaces long episode-by-episode segments between updates with aggregate flow generation and boundary refinement, resulting in slower runtime growth at large budgets. 

For each problem scale, the boldfaced budget in the corresponding timing table denotes the smallest episode count on the tested budget grid at which the mean value of $\operatorname{Regret}(K)/K$ for the conventional method falls below $0.01$. At these accuracy-relevant budgets and beyond, the fast simulator is consistently faster than conventional simulation. For the largest scale, conventional simulation remains faster from $K=10^2$ to $10^7$, but at the last measured budget $K=10^8$, its median simulator time is $3543$ seconds compared with $2032$ seconds for fast simulation, a speedup of about $1.72$. The crossover point and runtime growth of the adaptive fast method vary with the MDP size and budget. At $K=10^{15}$ for the largest scale $(10,15,10)$, the fast simulator remains directly measurable, with a median runtime of $7420$ seconds, or approximately $2.06$ hours. In contrast, the extrapolated conventional runtime at the same budget is approximately $3.54\times10^{10}$ seconds, or about $1120$ years, making direct conventional simulation computationally prohibitive. This simulator cost is far higher than the maximum learner time of about $38$ seconds across all tested configurations, showing that simulation overwhelmingly dominates the computational cost in the large-budget regime. Overall, these results demonstrate that, despite the overhead of adaptive boundary refinement, our method yields substantial computational gains at large budgets by avoiding episode-by-episode simulation.

\begin{figure}[htbp]
 \centering
 \includegraphics[width=0.993\textwidth]{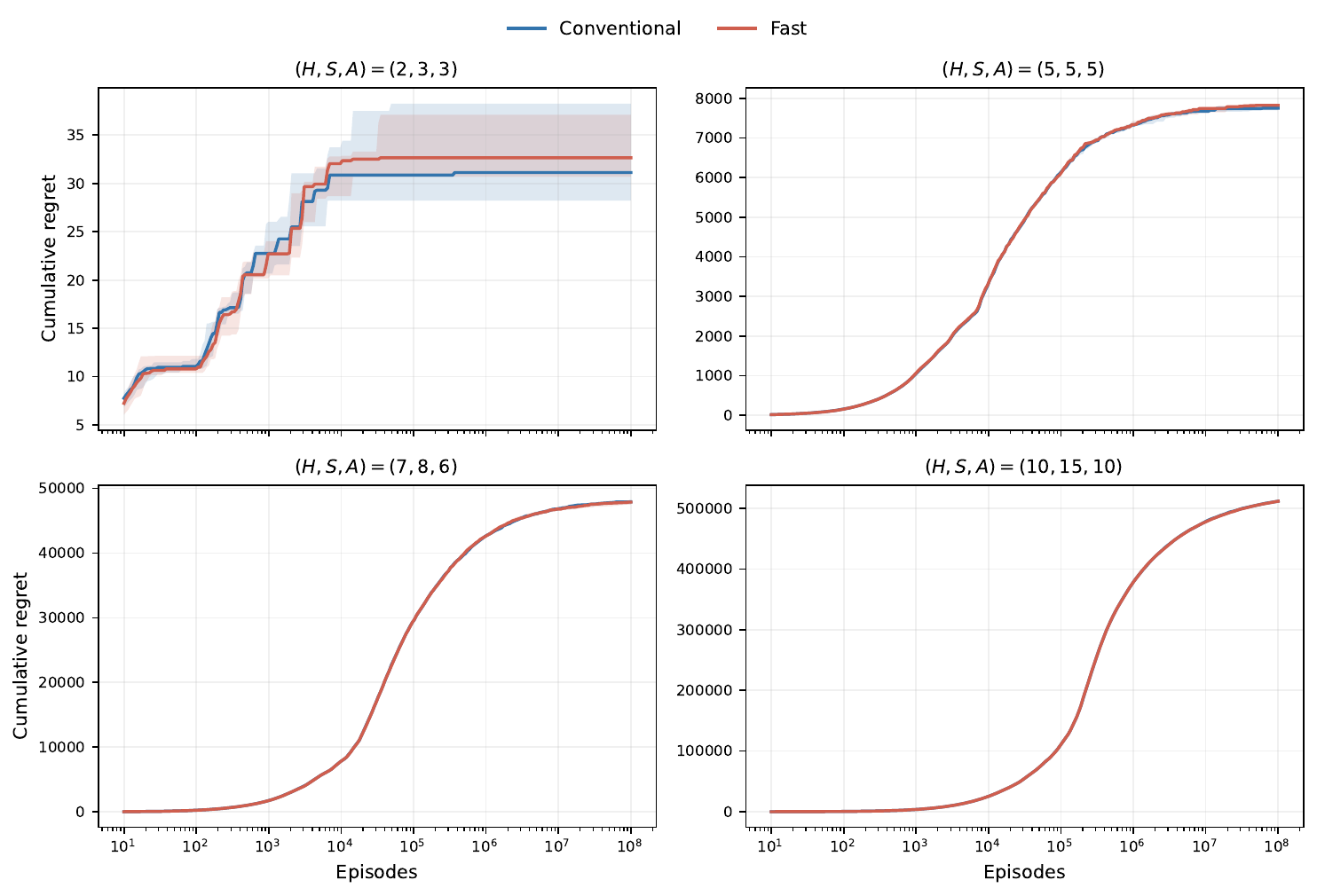}
 \caption{Cumulative UCB-Advantage regret on four fixed MDPs, one per scale. Conventional and fast simulation each use ten independent simulation seeds. Lines are seed medians and shaded bands are empirical 10th--90th percentiles. Intermediate regret values are visually interpolated between recorded batch endpoints.}
 \label{fig:ucb-regret}
\end{figure}

\begin{figure}[H]
 \centering
 \includegraphics[width=0.993\textwidth]{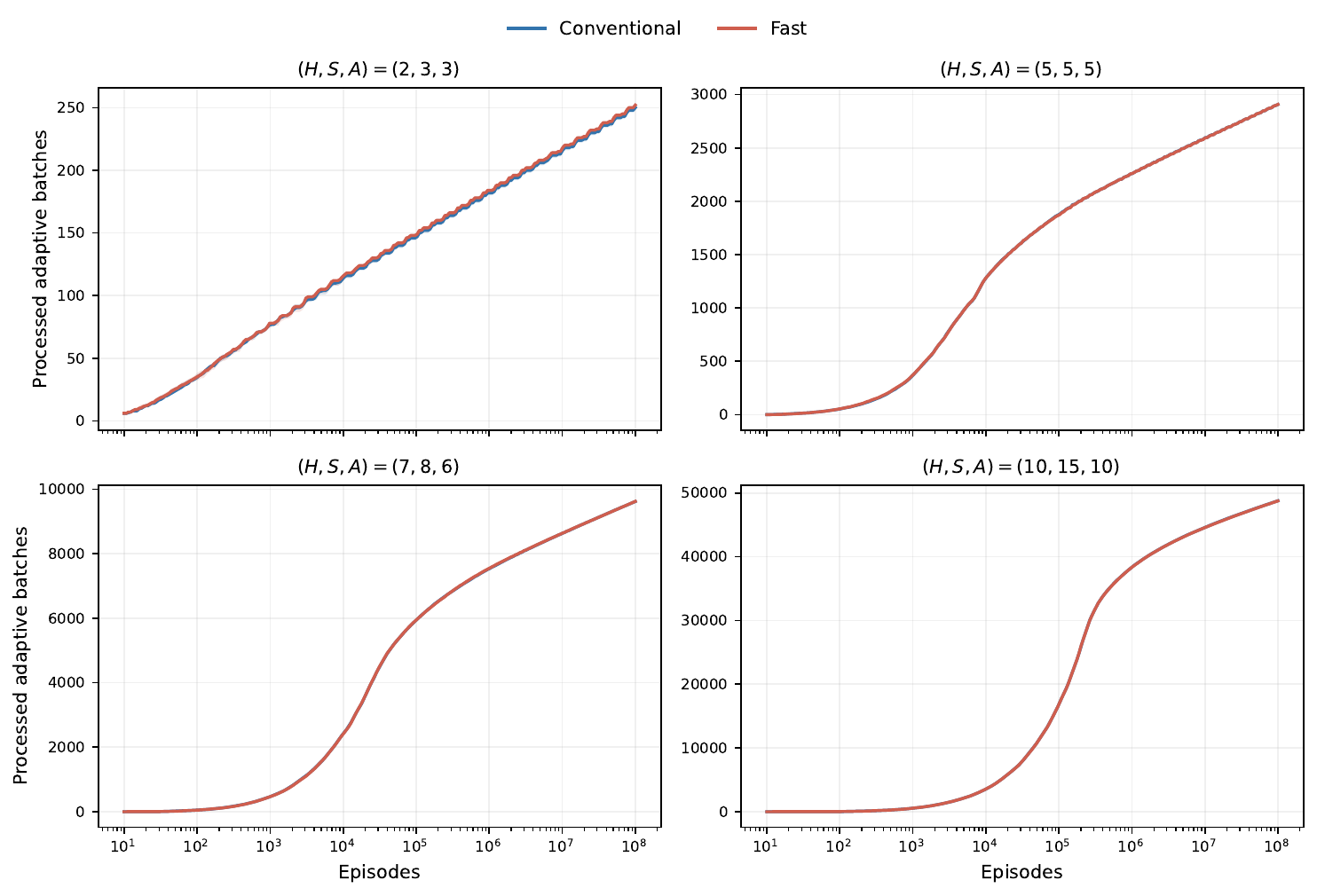}
 \caption{Processed adaptive batch count on the same four fixed MDPs. Lines are medians and shading is the empirical 10th--90th percentile range over ten simulation seeds per method. Counts change only at recorded batch endpoints.}
 \label{fig:ucb-batches}
\end{figure}

\begin{table}[htbp]
\centering
\tiny
\caption{Simulator time for $(H,S,A)=(2, 3, 3)$, in seconds. Entries are medians with 95\% percentile bootstrap CIs over ten MDP instances; speedup is the median of paired time ratios. Conventional entries beyond the last measured $10^{8}$ budget are marked $\approx$ and extrapolated by a factor of ten per additional power of ten.}
\label{tab:ucb-timing-s1}
\setlength{\tabcolsep}{0.5pt}
\begin{tabular}{rlll}
\toprule
Budget & Conventional (s) & Fast (s) & Speedup \\
\midrule
$10^{1}$ & $1.64\mathord{\times}10^{-4}\,[1.56\mathord{\times}10^{-4},1.81\mathord{\times}10^{-4}]$ & $8.18\mathord{\times}10^{-4}\,[7.76\mathord{\times}10^{-4},9.48\mathord{\times}10^{-4}]$ & $1.97\mathord{\times}10^{-1}\,[1.83\mathord{\times}10^{-1},2.27\mathord{\times}10^{-1}]$ \\
$10^{2}$ & $1.41\mathord{\times}10^{-3}\,[1.38\mathord{\times}10^{-3},1.58\mathord{\times}10^{-3}]$ & $1.41\mathord{\times}10^{-2}\,[1.25\mathord{\times}10^{-2},1.62\mathord{\times}10^{-2}]$ & $1.04\mathord{\times}10^{-1}\,[9.63\mathord{\times}10^{-2},1.09\mathord{\times}10^{-1}]$ \\
$10^{3}$ & $1.24\mathord{\times}10^{-2}\,[1.19\mathord{\times}10^{-2},1.38\mathord{\times}10^{-2}]$ & $5.60\mathord{\times}10^{-2}\,[5.25\mathord{\times}10^{-2},5.89\mathord{\times}10^{-2}]$ & $2.37\mathord{\times}10^{-1}\,[2.17\mathord{\times}10^{-1},2.47\mathord{\times}10^{-1}]$ \\
$\boldsymbol{10^{4}}$ & $\boldsymbol{1.20\mathord{\times}10^{-1}\,[1.16\mathord{\times}10^{-1},1.34\mathord{\times}10^{-1}]}$ & $\boldsymbol{1.16\mathord{\times}10^{-1}\,[1.07\mathord{\times}10^{-1},1.25\mathord{\times}10^{-1}]}$ & $\boldsymbol{1.10\,[1.01,1.14]}$ \\
$10^{5}$ & $1.20\,[1.16,1.33]$ & $1.96\mathord{\times}10^{-1}\,[1.79\mathord{\times}10^{-1},2.04\mathord{\times}10^{-1}]$ & $6.58\,[5.90,6.85]$ \\
$10^{6}$ & $1.19\mathord{\times}10^{1}\,[1.14\mathord{\times}10^{1},1.33\mathord{\times}10^{1}]$ & $2.82\mathord{\times}10^{-1}\,[2.65\mathord{\times}10^{-1},2.94\mathord{\times}10^{-1}]$ & $4.40\mathord{\times}10^{1}\,[4.03\mathord{\times}10^{1},4.65\mathord{\times}10^{1}]$ \\
$10^{7}$ & $1.20\mathord{\times}10^{2}\,[1.16\mathord{\times}10^{2},1.33\mathord{\times}10^{2}]$ & $3.89\mathord{\times}10^{-1}\,[3.69\mathord{\times}10^{-1},4.06\mathord{\times}10^{-1}]$ & $3.19\mathord{\times}10^{2}\,[3.02\mathord{\times}10^{2},3.36\mathord{\times}10^{2}]$ \\
$10^{8}$ & $1.19\mathord{\times}10^{3}\,[1.18\mathord{\times}10^{3},1.36\mathord{\times}10^{3}]$ & $5.11\mathord{\times}10^{-1}\,[4.82\mathord{\times}10^{-1},5.36\mathord{\times}10^{-1}]$ & $2.44\mathord{\times}10^{3}\,[2.33\mathord{\times}10^{3},2.55\mathord{\times}10^{3}]$ \\
$10^{9}$ & $\approx1.19\mathord{\times}10^{4}$ & $6.52\mathord{\times}10^{-1}\,[6.29\mathord{\times}10^{-1},6.88\mathord{\times}10^{-1}]$ & $\approx1.88\mathord{\times}10^{4}$ \\
$10^{10}$ & $\approx1.19\mathord{\times}10^{5}$ & $8.05\mathord{\times}10^{-1}\,[7.74\mathord{\times}10^{-1},8.48\mathord{\times}10^{-1}]$ & $\approx1.53\mathord{\times}10^{5}$ \\
$10^{11}$ & $\approx1.19\mathord{\times}10^{6}$ & $9.65\mathord{\times}10^{-1}\,[9.23\mathord{\times}10^{-1},1.02]$ & $\approx1.28\mathord{\times}10^{6}$ \\
$10^{12}$ & $\approx1.19\mathord{\times}10^{7}$ & $1.14\,[1.10,1.20]$ & $\approx1.07\mathord{\times}10^{7}$ \\
$10^{13}$ & $\approx1.19\mathord{\times}10^{8}$ & $1.32\,[1.30,1.40]$ & $\approx9.16\mathord{\times}10^{7}$ \\
$10^{14}$ & $\approx1.19\mathord{\times}10^{9}$ & $1.52\,[1.48,1.60]$ & $\approx7.99\mathord{\times}10^{8}$ \\
$10^{15}$ & $\approx1.19\mathord{\times}10^{10}$ & $1.74\,[1.71,1.85]$ & $\approx6.96\mathord{\times}10^{9}$ \\
\bottomrule
\end{tabular}
\end{table}

\begin{table}[htbp]
\centering
\tiny
\caption{Simulator time for $(H,S,A)=(5, 5, 5)$, in seconds. Entries are medians with 95\% percentile bootstrap CIs over ten MDP instances; speedup is the median of paired time ratios. Conventional entries beyond the last measured $10^{8}$ budget are marked $\approx$ and extrapolated by a factor of ten per additional power of ten.}
\label{tab:ucb-timing-s2}
\setlength{\tabcolsep}{3pt}
\begin{tabular}{rlll}
\toprule
Budget & Conventional (s) & Fast (s) & Speedup \\
\midrule
$10^{1}$ & $2.41\mathord{\times}10^{-4}\,[2.22\mathord{\times}10^{-4},2.52\mathord{\times}10^{-4}]$ & $5.98\mathord{\times}10^{-4}\,[2.98\mathord{\times}10^{-4},7.18\mathord{\times}10^{-4}]$ & $4.08\mathord{\times}10^{-1}\,[3.44\mathord{\times}10^{-1},1.96]$ \\
$10^{2}$ & $2.63\mathord{\times}10^{-3}\,[2.43\mathord{\times}10^{-3},2.72\mathord{\times}10^{-3}]$ & $5.11\mathord{\times}10^{-2}\,[4.82\mathord{\times}10^{-2},5.31\mathord{\times}10^{-2}]$ & $5.14\mathord{\times}10^{-2}\,[4.93\mathord{\times}10^{-2},5.36\mathord{\times}10^{-2}]$ \\
$10^{3}$ & $2.47\mathord{\times}10^{-2}\,[2.33\mathord{\times}10^{-2},2.63\mathord{\times}10^{-2}]$ & $6.60\mathord{\times}10^{-1}\,[6.31\mathord{\times}10^{-1},7.25\mathord{\times}10^{-1}]$ & $3.63\mathord{\times}10^{-2}\,[3.54\mathord{\times}10^{-2},3.87\mathord{\times}10^{-2}]$ \\
$10^{4}$ & $2.31\mathord{\times}10^{-1}\,[2.15\mathord{\times}10^{-1},2.44\mathord{\times}10^{-1}]$ & $3.65\,[3.52,3.85]$ & $6.24\mathord{\times}10^{-2}\,[6.15\mathord{\times}10^{-2},6.35\mathord{\times}10^{-2}]$ \\
$10^{5}$ & $2.22\,[2.16,2.35]$ & $7.45\,[7.15,7.84]$ & $2.98\mathord{\times}10^{-1}\,[2.93\mathord{\times}10^{-1},3.04\mathord{\times}10^{-1}]$ \\
$\boldsymbol{10^{6}}$ & $\boldsymbol{2.20\mathord{\times}10^{1}\,[2.13\mathord{\times}10^{1},2.33\mathord{\times}10^{1}]}$ & $\boldsymbol{1.11\mathord{\times}10^{1}\,[1.06\mathord{\times}10^{1},1.19\mathord{\times}10^{1}]}$ & $\boldsymbol{1.99\,[1.95,2.03]}$ \\
$10^{7}$ & $2.21\mathord{\times}10^{2}\,[2.19\mathord{\times}10^{2},2.32\mathord{\times}10^{2}]$ & $1.51\mathord{\times}10^{1}\,[1.46\mathord{\times}10^{1},1.62\mathord{\times}10^{1}]$ & $1.48\mathord{\times}10^{1}\,[1.44\mathord{\times}10^{1},1.51\mathord{\times}10^{1}]$ \\
$10^{8}$ & $2.16\mathord{\times}10^{3}\,[2.12\mathord{\times}10^{3},2.31\mathord{\times}10^{3}]$ & $1.98\mathord{\times}10^{1}\,[1.91\mathord{\times}10^{1},2.09\mathord{\times}10^{1}]$ & $1.10\mathord{\times}10^{2}\,[1.08\mathord{\times}10^{2},1.14\mathord{\times}10^{2}]$ \\
$10^{9}$ & $\approx2.16\mathord{\times}10^{4}$ & $2.51\mathord{\times}10^{1}\,[2.40\mathord{\times}10^{1},2.64\mathord{\times}10^{1}]$ & $\approx8.65\mathord{\times}10^{2}$ \\
$10^{10}$ & $\approx2.16\mathord{\times}10^{5}$ & $3.10\mathord{\times}10^{1}\,[2.96\mathord{\times}10^{1},3.22\mathord{\times}10^{1}]$ & $\approx7.08\mathord{\times}10^{3}$ \\
$10^{11}$ & $\approx2.16\mathord{\times}10^{6}$ & $3.70\mathord{\times}10^{1}\,[3.55\mathord{\times}10^{1},3.88\mathord{\times}10^{1}]$ & $\approx5.88\mathord{\times}10^{4}$ \\
$10^{12}$ & $\approx2.16\mathord{\times}10^{7}$ & $4.39\mathord{\times}10^{1}\,[4.23\mathord{\times}10^{1},4.59\mathord{\times}10^{1}]$ & $\approx4.98\mathord{\times}10^{5}$ \\
$10^{13}$ & $\approx2.16\mathord{\times}10^{8}$ & $5.20\mathord{\times}10^{1}\,[4.96\mathord{\times}10^{1},5.39\mathord{\times}10^{1}]$ & $\approx4.23\mathord{\times}10^{6}$ \\
$10^{14}$ & $\approx2.16\mathord{\times}10^{9}$ & $5.95\mathord{\times}10^{1}\,[5.72\mathord{\times}10^{1},6.19\mathord{\times}10^{1}]$ & $\approx3.70\mathord{\times}10^{7}$ \\
$10^{15}$ & $\approx2.16\mathord{\times}10^{10}$ & $6.79\mathord{\times}10^{1}\,[6.49\mathord{\times}10^{1},7.04\mathord{\times}10^{1}]$ & $\approx3.25\mathord{\times}10^{8}$ \\
\bottomrule
\end{tabular}
\end{table}

\begin{table}[H]
\centering
\tiny
\caption{Simulator time for $(H,S,A)=(7, 8, 6)$, in seconds. Entries are medians with 95\% percentile bootstrap CIs over ten MDP instances; speedup is the median of paired time ratios. Conventional entries beyond the last measured $10^{8}$ budget are marked $\approx$ and extrapolated by a factor of ten per additional power of ten.}
\label{tab:ucb-timing-s3}
\setlength{\tabcolsep}{2.5pt}
\begin{tabular}{rlll}
\toprule
Budget & Conventional (s) & Fast (s) & Speedup \\
\midrule
$10^{1}$ & $2.53\mathord{\times}10^{-4}\,[2.50\mathord{\times}10^{-4},2.76\mathord{\times}10^{-4}]$ & $9.84\mathord{\times}10^{-5}\,[9.43\mathord{\times}10^{-5},1.11\mathord{\times}10^{-4}]$ & $2.59\,[2.40,2.70]$ \\
$10^{2}$ & $2.78\mathord{\times}10^{-3}\,[2.74\mathord{\times}10^{-3},3.00\mathord{\times}10^{-3}]$ & $6.05\mathord{\times}10^{-2}\,[5.90\mathord{\times}10^{-2},7.26\mathord{\times}10^{-2}]$ & $4.53\mathord{\times}10^{-2}\,[4.22\mathord{\times}10^{-2},4.70\mathord{\times}10^{-2}]$ \\
$10^{3}$ & $2.74\mathord{\times}10^{-2}\,[2.71\mathord{\times}10^{-2},2.99\mathord{\times}10^{-2}]$ & $1.33\,[1.32,1.66]$ & $2.05\mathord{\times}10^{-2}\,[1.87\mathord{\times}10^{-2},2.06\mathord{\times}10^{-2}]$ \\
$10^{4}$ & $2.59\mathord{\times}10^{-1}\,[2.55\mathord{\times}10^{-1},2.82\mathord{\times}10^{-1}]$ & $1.10\mathord{\times}10^{1}\,[1.06\mathord{\times}10^{1},1.34\mathord{\times}10^{1}]$ & $2.33\mathord{\times}10^{-2}\,[2.18\mathord{\times}10^{-2},2.42\mathord{\times}10^{-2}]$ \\
$10^{5}$ & $2.47\,[2.43,2.67]$ & $4.02\mathord{\times}10^{1}\,[3.87\mathord{\times}10^{1},4.85\mathord{\times}10^{1}]$ & $6.07\mathord{\times}10^{-2}\,[5.62\mathord{\times}10^{-2},6.30\mathord{\times}10^{-2}]$ \\
$10^{6}$ & $2.42\mathord{\times}10^{1}\,[2.39\mathord{\times}10^{1},2.54\mathord{\times}10^{1}]$ & $6.53\mathord{\times}10^{1}\,[6.24\mathord{\times}10^{1},7.88\mathord{\times}10^{1}]$ & $3.67\mathord{\times}10^{-1}\,[3.37\mathord{\times}10^{-1},3.85\mathord{\times}10^{-1}]$ \\
$\boldsymbol{10^{7}}$ & $\boldsymbol{2.41\mathord{\times}10^{2}\,[2.39\mathord{\times}10^{2},2.82\mathord{\times}10^{2}]}$ & $\boldsymbol{9.07\mathord{\times}10^{1}\,[8.62\mathord{\times}10^{1},1.09\mathord{\times}10^{2}]}$ & $\boldsymbol{2.66\,[2.52,2.77]}$ \\
$10^{8}$ & $2.31\mathord{\times}10^{3}\,[2.29\mathord{\times}10^{3},2.79\mathord{\times}10^{3}]$ & $1.17\mathord{\times}10^{2}\,[1.12\mathord{\times}10^{2},1.43\mathord{\times}10^{2}]$ & $1.96\mathord{\times}10^{1}\,[1.91\mathord{\times}10^{1},2.06\mathord{\times}10^{1}]$ \\
$10^{9}$ & $\approx2.31\mathord{\times}10^{4}$ & $1.47\mathord{\times}10^{2}\,[1.40\mathord{\times}10^{2},1.81\mathord{\times}10^{2}]$ & $\approx1.55\mathord{\times}10^{2}$ \\
$10^{10}$ & $\approx2.31\mathord{\times}10^{5}$ & $1.80\mathord{\times}10^{2}\,[1.73\mathord{\times}10^{2},2.18\mathord{\times}10^{2}]$ & $\approx1.28\mathord{\times}10^{3}$ \\
$10^{11}$ & $\approx2.31\mathord{\times}10^{6}$ & $2.19\mathord{\times}10^{2}\,[2.09\mathord{\times}10^{2},2.69\mathord{\times}10^{2}]$ & $\approx1.05\mathord{\times}10^{4}$ \\
$10^{12}$ & $\approx2.31\mathord{\times}10^{7}$ & $2.60\mathord{\times}10^{2}\,[2.50\mathord{\times}10^{2},3.20\mathord{\times}10^{2}]$ & $\approx8.83\mathord{\times}10^{4}$ \\
$10^{13}$ & $\approx2.31\mathord{\times}10^{8}$ & $3.15\mathord{\times}10^{2}\,[2.95\mathord{\times}10^{2},3.78\mathord{\times}10^{2}]$ & $\approx7.43\mathord{\times}10^{5}$ \\
$10^{14}$ & $\approx2.31\mathord{\times}10^{9}$ & $3.82\mathord{\times}10^{2}\,[3.59\mathord{\times}10^{2},4.41\mathord{\times}10^{2}]$ & $\approx6.23\mathord{\times}10^{6}$ \\
$10^{15}$ & $\approx2.31\mathord{\times}10^{10}$ & $4.39\mathord{\times}10^{2}\,[4.17\mathord{\times}10^{2},5.09\mathord{\times}10^{2}]$ & $\approx5.40\mathord{\times}10^{7}$ \\
\bottomrule
\end{tabular}
\end{table}

\begin{table}[H]
\centering
\tiny
\caption{Simulator time for $(H,S,A)=(10, 15, 10)$, in seconds. Entries are medians with 95\% percentile bootstrap CIs over ten MDP instances; speedup is the median of paired time ratios. Conventional entries beyond the last measured $10^{8}$ budget are marked $\approx$ and extrapolated by a factor of ten per additional power of ten.}
\label{tab:ucb-timing-s4}
\setlength{\tabcolsep}{3pt}
\begin{tabular}{rlll}
\toprule
Budget & Conventional (s) & Fast (s) & Speedup \\
\midrule
$10^{1}$ & $3.68\mathord{\times}10^{-4}\,[3.42\mathord{\times}10^{-4},3.91\mathord{\times}10^{-4}]$ & $2.14\mathord{\times}10^{-4}\,[2.01\mathord{\times}10^{-4},2.23\mathord{\times}10^{-4}]$ & $1.75\,[1.63,1.81]$ \\
$10^{2}$ & $3.75\mathord{\times}10^{-3}\,[3.47\mathord{\times}10^{-3},3.89\mathord{\times}10^{-3}]$ & $3.98\mathord{\times}10^{-2}\,[3.41\mathord{\times}10^{-2},4.37\mathord{\times}10^{-2}]$ & $9.52\mathord{\times}10^{-2}\,[8.86\mathord{\times}10^{-2},1.02\mathord{\times}10^{-1}]$ \\
$10^{3}$ & $3.98\mathord{\times}10^{-2}\,[3.71\mathord{\times}10^{-2},4.22\mathord{\times}10^{-2}]$ & $4.39\,[4.03,4.47]$ & $9.28\mathord{\times}10^{-3}\,[9.00\mathord{\times}10^{-3},9.57\mathord{\times}10^{-3}]$ \\
$10^{4}$ & $3.86\mathord{\times}10^{-1}\,[3.56\mathord{\times}10^{-1},4.02\mathord{\times}10^{-1}]$ & $5.01\mathord{\times}10^{1}\,[4.62\mathord{\times}10^{1},5.11\mathord{\times}10^{1}]$ & $7.77\mathord{\times}10^{-3}\,[7.69\mathord{\times}10^{-3},7.94\mathord{\times}10^{-3}]$ \\
$10^{5}$ & $3.67\,[3.43,3.79]$ & $3.38\mathord{\times}10^{2}\,[3.09\mathord{\times}10^{2},3.44\mathord{\times}10^{2}]$ & $1.10\mathord{\times}10^{-2}\,[1.08\mathord{\times}10^{-2},1.12\mathord{\times}10^{-2}]$ \\
$10^{6}$ & $3.56\mathord{\times}10^{1}\,[3.33\mathord{\times}10^{1},3.67\mathord{\times}10^{1}]$ & $1.06\mathord{\times}10^{3}\,[9.92\mathord{\times}10^{2},1.08\mathord{\times}10^{3}]$ & $3.39\mathord{\times}10^{-2}\,[3.32\mathord{\times}10^{-2},3.46\mathord{\times}10^{-2}]$ \\
$10^{7}$ & $3.52\mathord{\times}10^{2}\,[3.27\mathord{\times}10^{2},3.64\mathord{\times}10^{2}]$ & $1.57\mathord{\times}10^{3}\,[1.48\mathord{\times}10^{3},1.63\mathord{\times}10^{3}]$ & $2.22\mathord{\times}10^{-1}\,[2.15\mathord{\times}10^{-1},2.31\mathord{\times}10^{-1}]$ \\
$\boldsymbol{10^{8}}$ & $\boldsymbol{3.54\mathord{\times}10^{3}\,[3.27\mathord{\times}10^{3},3.62\mathord{\times}10^{3}]}$ & $\boldsymbol{2.03\mathord{\times}10^{3}\,[1.91\mathord{\times}10^{3},2.11\mathord{\times}10^{3}]}$ & $\boldsymbol{1.72\,[1.66,1.78]}$ \\
$10^{9}$ & $\approx3.54\mathord{\times}10^{4}$ & $2.54\mathord{\times}10^{3}\,[2.37\mathord{\times}10^{3},2.63\mathord{\times}10^{3}]$ & $\approx1.38\mathord{\times}10^{1}$ \\
$10^{10}$ & $\approx3.54\mathord{\times}10^{5}$ & $3.07\mathord{\times}10^{3}\,[2.95\mathord{\times}10^{3},3.13\mathord{\times}10^{3}]$ & $\approx1.15\mathord{\times}10^{2}$ \\
$10^{11}$ & $\approx3.54\mathord{\times}10^{6}$ & $3.73\mathord{\times}10^{3}\,[3.61\mathord{\times}10^{3},3.91\mathord{\times}10^{3}]$ & $\approx9.35\mathord{\times}10^{2}$ \\
$10^{12}$ & $\approx3.54\mathord{\times}10^{7}$ & $4.56\mathord{\times}10^{3}\,[4.32\mathord{\times}10^{3},4.72\mathord{\times}10^{3}]$ & $\approx7.69\mathord{\times}10^{3}$ \\
$10^{13}$ & $\approx3.54\mathord{\times}10^{8}$ & $5.42\mathord{\times}10^{3}\,[5.16\mathord{\times}10^{3},5.60\mathord{\times}10^{3}]$ & $\approx6.48\mathord{\times}10^{4}$ \\
$10^{14}$ & $\approx3.54\mathord{\times}10^{9}$ & $6.37\mathord{\times}10^{3}\,[6.14\mathord{\times}10^{3},6.56\mathord{\times}10^{3}]$ & $\approx5.52\mathord{\times}10^{5}$ \\
$10^{15}$ & $\approx3.54\mathord{\times}10^{10}$ & $7.42\mathord{\times}10^{3}\,[7.25\mathord{\times}10^{3},7.62\mathord{\times}10^{3}]$ & $\approx4.74\mathord{\times}10^{6}$ \\
\bottomrule
\end{tabular}
\end{table}

\subsubsection{Large-Population Multivariate-Hypergeometric Sampling}

The adaptive fast simulator requires multivariate-hypergeometric (MHG) sampling when splitting an aggregate flow at a data-dependent boundary. An MHG draw can be generated through a sequence of conditional scalar hypergeometric draws. Specifically, let $b_1,\ldots,b_q$ denote the color counts, with $N_{\mathrm{pop}}=\sum_{i=1}^q b_i$, and suppose that $n$ items are drawn without replacement. Define
\[
R_i=\sum_{j=i}^q b_j,\qquad n_1=n,\qquad n_{i+1}=n_i-x_i.
\]
We sequentially sample
\[
x_i\sim\operatorname{Hypergeom}(R_i,b_i,n_i),\qquad i=1,\ldots,q-1,
\]
and set $x_q=n_q$. The resulting joint probability is
\[
\prod_{i=1}^{q-1}
\frac{\binom{b_i}{x_i}\binom{R_i-b_i}{n_i-x_i}}
{\binom{R_i}{n_i}}
=
\frac{\prod_{i=1}^{q}\binom{b_i}{x_i}}
{\binom{N_{\mathrm{pop}}}{n}},
\]
which is exactly the MHG probability mass function. Therefore, provided that each scalar hypergeometric draw is exact, this sequential construction produces the same MHG distribution as direct multivariate-hypergeometric sampling.

In our implementation, for total population size $N_{\mathrm{pop}}<10^9$, we use NumPy's MHG sampler with \texttt{method="marginals"}, which implements a sequential construction based on scalar hypergeometric draws. Since this method requires $N_{\mathrm{pop}}<10^9$, we use a custom large-population implementation when $N_{\mathrm{pop}}\geq 10^9$, retaining the same conditional factorization. Each scalar draw is generated using an HRUA-style ratio-of-uniforms rejection procedure. To reduce precision loss from subtracting large log-factorial terms, the compiled implementation evaluates the required log-probability ratios through stable log-factorial differences, with numerically ambiguous acceptance decisions recomputed using higher-precision arithmetic. The exactness result above concerns the probabilistic construction under exact scalar sampling; the custom sampler is its finite-precision numerical implementation. Thus, the large-population modification leaves the MHG factorization and distributional argument unchanged.
\section{Proof of \texorpdfstring{\Cref{thm:forward}}{Theorem 3.1}}
\label{app:proof-forward}
\begin{proof}
We prove equality of the joint flow law by induction over the step index $h$.

Under conventional episode-by-episode simulation, the initial states $X_{1,1},\ldots,X_{m,1}$ are i.i.d.\ from $\nu$. Hence
\[
N_1^{\iid}(\cdot)\sim \Mult(m,\nu),
\]
which is exactly the law used in~\Cref{eq:forward1}.

Now fix $h\in[H]$ and suppose that the joint law of the aggregate flow up to the beginning of step $h$, and in particular the law of $N_h$, agrees under the two constructions. We compare the conditional laws used to extend this flow through step $h$.

Let
\[
\mathcal H_h^{-}
=
\sigma(N_1,C_1,E_1,\ldots,C_{h-1},E_{h-1}),
\]
with $\mathcal H_1^{-}=\sigma(N_1)$, denote the aggregate-flow history up to the beginning of step $h$, and let
\[
\mathcal G_h^{-}
=
\sigma(X_{i,1},A_{i,1},X_{i,2},\ldots,A_{i,h-1},X_{i,h}:i\in[m])
\]
denote the corresponding finer labeled history. Then $\mathcal H_h^{-}\subseteq\mathcal G_h^{-}$, and $N_h$ is $\mathcal H_h^{-}$-measurable.

For conventional simulation, conditional on $\mathcal G_h^{-}$, for each $s\in\cS$, the labels satisfying $X_{i,h}=s$ choose their actions independently according to $\pi_h(\cdot\mid s)$. Therefore
\[
C_h(s,\cdot)\mid\mathcal G_h^{-}
\sim
\Mult\!\left(N_h(s),\pi_h(\cdot\mid s)\right),
\]
independently across $s$. Since this conditional law depends on $\mathcal G_h^{-}$ only through the $\mathcal H_h^{-}$-measurable count vector $N_h$, the tower property yields
\[
C_h(s,\cdot)\mid\mathcal H_h^{-}
\sim
\Mult\!\left(N_h(s),\pi_h(\cdot\mid s)\right),
\]
independently across $s$. This is exactly the conditional law in~\Cref{eq:forward2}.

Next, let
\[
\mathcal G_h
=
\sigma(\mathcal G_h^{-},A_{1,h},\ldots,A_{m,h})
\qquad\text{and}\qquad
\mathcal H_h
=
\sigma(\mathcal H_h^{-},C_h).
\]
Then $\mathcal H_h\subseteq\mathcal G_h$. Conditional on $\mathcal G_h$, for each $(s,a)\in\cS\times\cA$, the labels satisfying $(X_{i,h},A_{i,h})=(s,a)$ transition independently according to $P_h(\cdot\mid s,a)$. Thus
\[
E_h(s,a,\cdot)\mid\mathcal G_h
\sim
\Mult\!\left(C_h(s,a),P_h(\cdot\mid s,a)\right),
\]
independently across $(s,a)$. Since this conditional law depends on $\mathcal G_h$ only through the $\mathcal H_h$-measurable count array $C_h$, another application of the tower property gives
\[
E_h(s,a,\cdot)\mid\mathcal H_h
\sim
\Mult\!\left(C_h(s,a),P_h(\cdot\mid s,a)\right),
\]
independently across $(s,a)$, which is exactly the conditional law in~\Cref{eq:forward3}.

Finally, in both constructions the next state counts are the same deterministic function of the transition-flow counts,
\[
N_{h+1}(s')=\sum_{s,a}E_h(s,a,s'),\qquad s'\in\cS,
\]
as in~\Cref{eq:forward4}. Therefore, the conditional law of $(C_h,E_h,N_{h+1})$ given the previously generated aggregate flow is identical under conventional simulation and direct Markov-flow sampling. Induction over $h=1,\ldots,H$ gives equality of the joint law of $(N_1,C_1,E_1,\ldots,C_H,E_H)$, and hence
\[
F_m^{\fc}\stackrel{d}{=}F_m^{\iid}.
\]
\end{proof}

\section{Proof of \texorpdfstring{\Cref{thm:split}}{Theorem 3.3}}
\label{app:proof-split}

We first record the conditional combinatorial structure of a labeled batch given its aggregate flow. Throughout this section, let $\omega_1,\ldots,\omega_m\stackrel{\iid}{\sim}\mathbb P_\pi$ and let $F_m$ denote their aggregate Markov flow. Fix $I_L=\{1,\ldots,m_L\}$, and let $F_{m_L}^{\trueflow,L}$ be the flow of the trajectories whose labels belong to $I_L$.

\begin{lemma}[Conditional uniformity given the full flow]
\label{lem:conditional-uniformity}
Fix a flow realization $f=(N_1,C,E)$ such that $\mathbb P(F_m=f)>0$. Conditional on $F_m=f$, all ordered labeled trajectory sequences $(\omega_1,\ldots,\omega_m)$ compatible with $f$ are equally likely.
\end{lemma}

\begin{proof}
For any ordered labeled trajectory sequence, independence of the episodes and~\Cref{eq:traj-prob} give
\[
\prod_{i=1}^m \mathbb P_\pi(\omega_i)=\prod_{s\in\cS}\nu(s)^{N_1(s)}\prod_{h=1}^H\prod_{s,a}\pi_h(a\mid s)^{C_h(s,a)}\prod_{h=1}^H\prod_{s,a,s'}P_h(s'\mid s,a)^{E_h(s,a,s')}.
\]
Here and below, we use the convention $0^0=1$.

For every sequence compatible with $f$, the right-hand side is the same because it depends only on the aggregate counts in $f$. Since $\mathbb P(F_m=f)>0$, conditioning on $F_m=f$ therefore produces the uniform distribution over all compatible ordered labeled trajectory sequences.
\end{proof}

\begin{lemma}[Conditional labeling representation]
\label{lem:conditional-labeling}
Fix $f=(N_1,C,E)$ with $\mathbb P(F_m=f)>0$. Conditional on $F_m=f$, the
labeled trajectories can equivalently be generated by the following
combinatorial procedure:
\begin{enumerate}
    \item assign the $m$ labels to initial states uniformly among all assignments having counts $N_1(s)$;
    \item for each $h\in[H]$ and each state $s$, once the set of labels currently at $s$ is known, assign those labels uniformly to outgoing categories $(a,s')\in\cA\times\cS$ with prescribed counts $E_h(s,a,s')$.
\end{enumerate}
Conditional on the label sets entering the rows at step $h$, the outgoing assignments are independent across states $s$.
\end{lemma}

\begin{proof}
The procedure generates only labeled trajectory sequences compatible with $f$. For any fixed compatible sequence, its probability under this combinatorial procedure is
\[
\frac{1}{m!/\prod_s N_1(s)!}\prod_{h=1}^H\prod_s\frac{1}{N_h(s)!/\prod_{a,s'}E_h(s,a,s')!},
\]
which depends only on $f$ and not on the particular compatible sequence. Hence the procedure is uniform over all compatible sequences and is therefore identical to the conditional law in Lemma~\ref{lem:conditional-uniformity}. The row-wise conditional independence follows directly from the construction, because after the entering label sets are fixed, each state row is assigned separately subject only to its prescribed row counts.
\end{proof}

For the true left batch, write $N_1^L,C_h^L,E_h^L$ for its corresponding counts. Before processing step $h$, define the aggregate left-flow sigma-field
\[
\mathcal H_{h-1}^L=\sigma\!\left(N_1^L,E_1^L,\ldots,E_{h-1}^L\right),
\]
with $\mathcal H_0^L=\sigma(N_1^L)$. Thus $\mathcal H_{h-1}^L$ determines $N_h^L(s)$ for every $s$. We will also use the finer labeled-history sigma-field
\[
\mathcal G_h^{-}=\sigma\!\left(X_{i,1},A_{i,1},X_{i,2},\ldots,A_{i,h-1},X_{i,h}:\ i\in[m]\right),
\]
where for $h=1$ this means $\mathcal G_1^{-}=\sigma(X_{1,1},\ldots,X_{m,1})$.

\begin{lemma}[Initial-state hypergeometric split]
\label{lem:initial-split}
Conditional on $F_m$,
\[
N_1^L(\cdot)\sim\MHG\!\left(N_1(\cdot),m_L\right).
\]
\end{lemma}

\begin{proof}
Fix a realization $F_m=f$ and write $n_s=N_1(s)$. By Lemma~\ref{lem:conditional-labeling}, conditional on $F_m=f$, the $m$ labels are assigned uniformly to the initial states among all assignments having counts $(n_s)_{s\in\cS}$. The total number of such assignments is
\[
\frac{m!}{\prod_{s\in\cS}n_s!}.
\]
Fix a feasible vector $y=(y_s)_{s\in\cS}$ satisfying $0\le y_s\le n_s$ and $\sum_s y_s=m_L$, and let $m_R=m-m_L$. The number of initial-state assignments for which the first $m_L$ labels have category counts $y_s$ is
\[
\frac{m_L!}{\prod_s y_s!}\frac{m_R!}{\prod_s(n_s-y_s)!}.
\]
Consequently,
\[
\begin{aligned}
\mathbb P\!\left(N_1^L(s)=y_s,\ \forall s\mid F_m=f\right)
&=\frac{\frac{m_L!}{\prod_s y_s!}\frac{m_R!}{\prod_s(n_s-y_s)!}}{\frac{m!}{\prod_s n_s!}}=
\frac{m_L!m_R!}{m!}\prod_s\frac{n_s!}{y_s!(n_s-y_s)!}\\
&=\frac{\prod_{s\in\cS}\binom{n_s}{y_s}}{\binom{m}{m_L}}.
\end{aligned}
\]
Since $n_s=N_1(s)$ under $F_m=f$, this is exactly
$\MHG(N_1(\cdot),m_L)$.
\end{proof}

\begin{lemma}[Row-wise hypergeometric split]
\label{lem:row-split}
Fix $h\in[H]$. Conditional on $F_m$ and $\mathcal H_{h-1}^L$, for each
$s\in\cS$,
\[
E_h^L(s,\cdot,\cdot)\sim\MHG\!\left(\{E_h(s,a,s')\}_{a\in\cA,\ s'\in\cS},N_h^L(s)\right),
\]
and these arrays are conditionally independent across $s$.
\end{lemma}

\begin{proof}
Fix realizations $F_m=f$ and $\mathcal H_{h-1}^L=\eta$. These realizations determine the aggregate left state counts $N_h^L(s)$, but in general they do \emph{not} determine which particular labels in $I_L$ are currently at each state. This is the point at which it is useful to condition temporarily on the finer labeled history $\mathcal G_h^{-}$.

For each $s\in\cS$, define
\[
\mathcal I_h(s)=\{i\in[m]:X_{i,h}=s\},\quad L_h(s)=\mathcal I_h(s)\cap I_L.
\]
Conditional on $\mathcal G_h^{-}$, both sets are fixed, with
\[
|\mathcal I_h(s)|=N_h(s),\qquad|L_h(s)|=N_h^L(s).
\]
By Lemma~\ref{lem:conditional-labeling}, conditional on $F_m=f$ and the entering label sets $\{\mathcal I_h(s)\}_{s\in\cS}$, the labels in row $(h,s)$ are assigned uniformly to the outgoing categories $(a,s')$ with prescribed category counts $\{E_h(s,a,s')\}_{a,s'}$. The assignments for distinct states are conditionally independent.

Fix a state $s$ and a feasible array $y_s=\{y_s(a,s')\}_{a,s'}$ satisfying
\[
0\le y_s(a,s')\le E_h(s,a,s'),\qquad\sum_{a,s'}y_s(a,s')=N_h^L(s).
\]
The number of row assignments having the prescribed full row counts is
\[
\frac{N_h(s)!}{\prod_{a,s'}E_h(s,a,s')!}.
\]
Among them, the number for which the fixed left subset $L_h(s)$ receives category counts $y_s(a,s')$ is
\[
\frac{N_h^L(s)!}{\prod_{a,s'}y_s(a,s')!}\frac{(N_h(s)-N_h^L(s))!}{\prod_{a,s'}(E_h(s,a,s')-y_s(a,s'))!}.
\]
Therefore,
\[
\begin{aligned}
\mathbb P\!\left(E_h^L(s,a,s')=y_s(a,s')\ \forall a,s'\mid F_m=f,\mathcal G_h^{-}\right)
& =\frac{\frac{N_h^L(s)!}{\prod_{a,s'}y_s(a,s')!}\frac{(N_h(s)-N_h^L(s))!}{\prod_{a,s'}(E_h(s,a,s')-y_s(a,s'))!}}{\frac{N_h(s)!}{\prod_{a,s'}E_h(s,a,s')!}}\\
& =\frac{\prod_{a,s'}\binom{E_h(s,a,s')}{y_s(a,s')}}{\binom{N_h(s)}{N_h^L(s)}}.
\end{aligned}
\]
Thus, conditional on $F_m=f$ and $\mathcal G_h^{-}$,
\[
E_h^L(s,\cdot,\cdot)\sim\MHG\!\left(\{E_h(s,a,s')\}_{a,s'},N_h^L(s)\right).
\]
Because the current outgoing row assignments are product-form across states, for any collection of feasible arrays $\{y_s\}_{s\in\cS}$,
\[
\begin{aligned}
&\mathbb P\!\left(E_h^L(s,\cdot,\cdot)=y_s,\ \forall s\mid F_m=f,\mathcal G_h^{-}\right) =\prod_{s\in\cS}\frac{\prod_{a,s'}\binom{E_h(s,a,s')}{y_s(a,s')}}{\binom{N_h(s)}{N_h^L(s)}}.
\end{aligned}
\]
Crucially, the right-hand side does not depend on the identities of the entering labels or any other aspect of $\mathcal G_h^{-}$; it depends on the past only through $N_h^L(s)$, which is already measurable with respect to
$\mathcal H_{h-1}^L$. Since
$\mathcal H_{h-1}^L\subseteq\mathcal G_h^{-}$, the tower property gives
\[
\begin{aligned}
\mathbb P\!\left(E_h^L(s,\cdot,\cdot)=y_s,\ \forall s\mid F_m,\mathcal H_{h-1}^L\right)
&=\mathbb E\!\left[\mathbb P\!\left(E_h^L(s,\cdot,\cdot)=y_s,\ \forall s\mid F_m,\mathcal G_h^{-}\right)\middle|F_m,\mathcal H_{h-1}^L\right]\\
&=\prod_{s\in\cS}\frac{\prod_{a,s'}\binom{E_h(s,a,s')}{y_s(a,s')}}{\binom{N_h(s)}{N_h^L(s)}}.
\end{aligned}
\]
This is exactly the claimed row-wise multivariate-hypergeometric law and establishes the conditional independence across states.
\end{proof}

\begin{lemma}[Equivalence of row-wise and action-then-transition splitting]
\label{lem:split-equivalence}
Fix $h\in[H]$ and $s\in\cS$. Conditional on $F_m$ and $\mathcal H_{h-1}^L$, the row-wise draw in Lemma~\ref{lem:row-split} is equivalent in distribution to first sampling
\[
C_h^L(s,\cdot)\sim\MHG\!\left(C_h(s,\cdot),N_h^L(s)\right)
\]
and then, independently across $a\in\cA$, sampling
\[
E_h^L(s,a,\cdot)\sim\MHG\!\left(E_h(s,a,\cdot),C_h^L(s,a)\right).
\]
The two-stage implementation also preserves the conditional independence across states $s$.
\end{lemma}

\begin{proof}
Fix $h$ and $s$, and abbreviate $e_{a,s'}=E_h(s,a,s')$ and $c_a=C_h(s,a)=\sum_{s'}e_{a,s'}$. Under the row-wise split, for every feasible array $y=\{y(a,s')\}_{a,s'}$ with $\sum_{a,s'}y(a,s')=N_h^L(s)$,
\[
\mathbb P(Y=y)=\frac{\prod_{a,s'}\binom{e_{a,s'}}{y(a,s')}}{\binom{N_h(s)}{N_h^L(s)}}.
\]
Let $Y_A(a)=\sum_{s'}Y(a,s')$. For any feasible action-total vector $r=(r_a)_a$, Vandermonde's identity gives
\[
\begin{aligned}
\mathbb P\!\left(Y_A(a)=r_a,\ \forall a\right)&=
\frac{\prod_a\sum_{\{y(a,s'):\,\sum_{s'}y(a,s')=r_a\}}\prod_{s'}\binom{e_{a,s'}}{y(a,s')}}{\binom{N_h(s)}{N_h^L(s)}}\\
&=\frac{\prod_a\binom{c_a}{r_a}}{\binom{N_h(s)}{N_h^L(s)}}.
\end{aligned}
\]
Hence $Y_A\sim\MHG(C_h(s,\cdot),N_h^L(s))$. Moreover, conditional on $Y_A(a)=r_a$ for every $a$,
\[
\mathbb P\!\left(Y(a,s')=y(a,s'),\ \forall a,s'\mid Y_A=r\right)=\prod_a\frac{\prod_{s'}\binom{e_{a,s'}}{y(a,s')}}{\binom{c_a}{r_a}},
\]
which is exactly the product of the conditional draws $\MHG(E_h(s,a,\cdot),r_a)$ across $a$. Thus the row-wise and two-stage procedures have the same law. The independence across states follows from Lemma~\ref{lem:row-split}.
\end{proof}

\begin{proof}[Proof of \Cref{thm:split}]
We induct over $h$ on the conditional law of the left flow.

By Lemma~\ref{lem:initial-split}, the split rule~\Cref{eq:split1} generates exactly the conditional law of the true left initial-state counts $N_1^L$ given $F_m$. Suppose now that the split has generated the correct conditional law of $\mathcal H_{h-1}^L$. Conditional on $(F_m,\mathcal H_{h-1}^L)$, Lemma~\ref{lem:row-split} gives the correct row-wise law of $E_h^L$, and Lemma~\ref{lem:split-equivalence} shows that this law is exactly implemented by first drawing~\Cref{eq:split2} and then drawing~\Cref{eq:split3}. After $E_h^L$ is sampled, flow conservation gives
\[
N_{h+1}^L(s')=\sum_{s,a}E_h^L(s,a,s'),
\]
which is the same deterministic relation satisfied by the true left sub-batch. Therefore the induction continues through all $h\in[H]$, yielding
\[
F_{m_L}^{\splflow,L}\mid F_m\stackrel{d}{=}F_{m_L}^{\trueflow,L}\mid F_m.
\]

For the right child, conditional on the parent flow and the left child, all right counts are determined by subtraction:
\[
F_{m_R}^{\splflow,R}=F_m-F_{m_L}^{\splflow,L}.
\]
The same identity holds for the true partition of the first $m_L$ and last $m_R$ labeled trajectories. Therefore
\[
\left(F_{m_L}^{\splflow,L},F_{m_R}^{\splflow,R}\right)\mid F_m
\stackrel{d}{=}
\left(F_{m_L}^{\trueflow,L},F_{m_R}^{\trueflow,R}\right)\mid F_m,
\]
which proves the claimed joint conditional correctness.
\end{proof}

\section{Proof of \texorpdfstring{\Cref{thm:adaptive}}{Theorem 3.5}}
\label{app:proof-adaptive}

\begin{proof}
Fix $Z$ and consider the sequence $\omega_1,\omega_2,\ldots$ generated by conventional episode-by-episode simulation while the policy remains equal to $\pi(Z)$. Let $G_j$ be the Markov flow of its first $j$ episodes and let
\[
\tau=\inf\{j\ge1:g(Z,G_j)=1\}.
\]
We first consider the no-trigger case. By prefix monotonicity, $g(Z,G_m)=0$ if and only if $\tau>m$. By~\Cref{thm:forward}, the sampled candidate flow $F_m$ has exactly the same law as $G_m$. Therefore, the event $g(Z,F_m)=0$ and the returned flow have the same joint distribution as the corresponding no-trigger event and prefix flow under conventional simulation.

Now suppose $g(Z,F_m)=1$, corresponding under the common root-flow law to $\tau\le m$. We justify the recursive use of \Cref{thm:split} by maintaining the following invariant for the conventional i.i.d.\ episode sequence: conditional on the entire revealed-flow history, the labeled trajectories in the current retained node are uniform over all sequences compatible with that node's flow. At the root, this follows from Lemma~\ref{lem:conditional-uniformity}.

Suppose the invariant holds for a retained node with index interval $I$, and write $F_I$ for its flow. Conditional on the revealed history, $I$ and $F_I$ are fixed, and the trajectories in $I$ are uniform over all labeled sequences compatible with $F_I$. Thus \Cref{thm:split} gives the exact joint conditional law of the flows of its left and right children. Conditional on the revealed history and both child flows, the labeled trajectories in the two children are independent, and each child is uniform over the sequences compatible with its own flow. The choice of which child to retain is determined by the revealed flows and the accepted prefix flow; it uses no additional information about the trajectories' labels. Consequently, revealing the child flows and making this choice leaves the invariant valid for the retained child. Since the splitting procedure uses fresh randomness at each refinement, induction over the refinement depth shows that every adaptive split has the same conditional law as the corresponding split of the conventional episode sequence.

It remains to verify that the binary search retains exactly the batch containing the first trigger. Throughout the search, maintain the invariant that (i) $F^{\mathrm{pre}}$ is the flow of a contiguous prefix that has not triggered, and (ii) the current node immediately follows this prefix and contains the first triggering episode. The invariant holds initially because $F^{\mathrm{pre}}=0$ and the root batch contains a trigger. Suppose a current node is split into contiguous children $(F^L,n_L)$ and $(F^R,n_R)$. If $g(Z,F^{\mathrm{pre}}+F^L)=1$, prefix monotonicity implies that the first trigger lies in the left child, so the algorithm retains $F^L$. Otherwise, $F^{\mathrm{pre}}+F^L$ has not triggered, and because the current node contains a trigger, the first trigger must lie in the right child; the algorithm therefore accepts $F^L$ into $F^{\mathrm{pre}}$ and retains $F^R$. Thus the invariant is preserved.

When the current node has size one, it corresponds to the boundary episode at index $\tau$, while $F^{\mathrm{pre}}$ is the flow of the first $\tau-1$ episodes. Consequently, the returned boundary flow satisfies
\[
F^{\mathrm{pre}}+F\stackrel{d}{=}G_\tau.
\]
Thus Algorithm~\ref{alg:fastforward} returns the correct boundary-inclusive aggregate flow without reconstructing the boundary trajectory.

Each binary refinement replaces a node of size $n$ by a child of size at most $\lceil n/2\rceil$. After $q$ splits, the retained node therefore has size at most $\lceil m/2^q\rceil$, so at most $\lceil\log_2m\rceil$ splits are required to reach the boundary episode.

Finally, suppose the boundary compatibility condition in~\Cref{sec:problem} holds. Since the returned flow $\widetilde G_\tau$ has the same distribution as $G_\tau$ under conventional simulation, applying the same boundary update produces the same post-boundary learner state $Z'$ and hence the same next policy. The unaccepted suffix of the candidate batch was generated under the old policy $\pi(Z)$ and is discarded. Fresh episodes generated under the updated policy $\pi(Z')$ therefore have the same conditional law as the episodes generated after the boundary by conventional simulation. Repeating this argument over successive boundaries proves equality in distribution of the sequence of post-boundary learner states.
\end{proof}

\section{Review of Algorithms Used in \texorpdfstring{\Cref{sec:examples}}{Section 5}}
\label{app:algorithm-review}

In this section, we briefly review the three RL algorithms used in Section~\ref{sec:examples}. We rewrite their notation to match the notation used throughout the paper whenever possible.

\subsection{Model-Based Planning with a Generative Model}
\label{app:model-based-planning}

We first review the finite-horizon model-based planning with a generative model considered in~\citet{li2024breaking}. A generative model allows the learner to query any state-action-step triple $(s,a,h)$ and obtain an independent next-state sample from $P_h(\cdot\mid s,a)$. Suppose each triple is queried $J$ times. Let $Y_{h,j}(s,a)$ denote the $j$-th sampled next state for $(s,a,h)$. The empirical transition kernel is
\[
\widehat P_h(s'\mid s,a)=\frac{1}{J}\sum_{j=1}^{J}\ind\{Y_{h,j}(s,a)=s'\}.
\]
Equivalently, if $E_h(s,a,s')$ denotes the resulting transition count, then $$\widehat P_h(s'\mid s,a)=E_h(s,a,s')/J.$$
The algorithm constructs the empirical MDP using $\widehat P$ and the original reward function $r$, and then computes an optimal policy of this empirical MDP by backward dynamic programming.

\begin{algorithm}[ht]
\caption{Model-based planning with a generative model}
\label{alg:review-generative}
\begin{algorithmic}[1]
\Require Generative model, query count $J$
\For{$h=1,\ldots,H$, $s\in\cS$, $a\in\cA$}
    \State Draw $Y_{h,1}(s,a),\ldots,Y_{h,J}(s,a)\stackrel{\mathrm{iid}}{\sim}P_h(\cdot\mid s,a)$
    \State Set $\widehat P_h(s'\mid s,a)\gets J^{-1}\sum_{j=1}^{J}\ind\{Y_{h,j}(s,a)=s'\}$ for all $s'\in\cS$
\EndFor
\State $\widehat V_{H+1}(s)\gets0$ for all $s\in\cS$
\For{$h=H,H-1,\ldots,1$}
    \State $\widehat Q_h(s,a)\gets r_h(s,a)+\sum_{s'}\widehat P_h(s'\mid s,a)\widehat V_{h+1}(s')$
    \State $\widehat V_h(s)\gets\max_{a\in\cA}\widehat Q_h(s,a)$
    \State $\widehat\pi_h(s)\gets\argmax_{a\in\cA}\widehat Q_h(s,a)$
\EndFor
\State \Return $\widehat\pi=\{\widehat\pi_h\}_{h=1}^H$
\end{algorithmic}
\end{algorithm}

Thus, after sampling, the planning stage depends on the generative-model observations only through the empirical transition counts $E_h(s,a,s')$.

\subsection{Vanilla Pessimistic Value Iteration}
\label{app:vpvi}

We next review Vanilla Pessimistic Value Iteration (VPVI)~\citep{yin2021towards}. Let the offline dataset consist of $K$ trajectories generated by a fixed behavior policy $\mu$. Using the notation of Section~\ref{sec:problem}, $C_h(s,a)$ is the number of visits to $(s,a)$ at step $h$, and $E_h(s,a,s')$ is the corresponding transition count. For $C_h(s,a)>0$, VPVI forms the empirical transition kernel
\[
\widehat P_h(s'\mid s,a)=\frac{E_h(s,a,s')}{C_h(s,a)}.
\]
When $C_h(s,a)=0$, the original algorithm sets $\widehat P_h(\cdot\mid s,a)$ to the uniform distribution. Under the deterministic-reward setting considered in this paper, the empirical reward equals $r_h(s,a)$ whenever $C_h(s,a)>0$; for an unvisited pair, the original algorithm sets the empirical reward to zero.

Let $\iota=\log(HSA/\delta)$. VPVI uses a Hoeffding-type pessimistic penalty of the form
\[
\Gamma_h(s,a)
=
\begin{cases}
cH\iota/\sqrt{C_h(s,a)}, & C_h(s,a)>0,\\
cH\iota, & C_h(s,a)=0,
\end{cases}
\]
where $c>0$ is an absolute constant. Starting from $\widehat V_{H+1}=0$, it performs backward value iteration, subtracting $\Gamma_h$ from the empirical Bellman update.

\begin{algorithm}[ht]
\caption{Vanilla Pessimistic Value Iteration}
\label{alg:review-vpvi}
\begin{algorithmic}[1]
\Require Offline dataset of $K$ trajectories, failure probability $\delta$
\State Construct $C_h(s,a)$ and $E_h(s,a,s')$ from the dataset
\For{$h\in[H]$, $(s,a)\in\cS\times\cA$}
    \If{$C_h(s,a)>0$}
        \State $\widehat P_h(s'\mid s,a)\gets E_h(s,a,s')/C_h(s,a)$
        \State $\widehat r_h(s,a)\gets r_h(s,a)$
    \Else
        \State $\widehat P_h(s'\mid s,a)\gets1/S$ and $\widehat r_h(s,a)\gets0$
    \EndIf
\EndFor
\State $\widehat V_{H+1}(s)\gets0$ for all $s\in\cS$
\State $\iota\gets\log(HSA/\delta)$
\For{$h=H,H-1,\ldots,1$}
    \State $\overline Q_h(s,a)\gets\widehat r_h(s,a)+\sum_{s'}\widehat P_h(s'\mid s,a)\widehat V_{h+1}(s')$
    \State Set $\Gamma_h(s,a)\gets cH\iota/\sqrt{C_h(s,a)}$ if $C_h(s,a)>0$, and $\Gamma_h(s,a)\gets cH\iota$ otherwise
    \State $Q_h(s,a)\gets\max\!\left\{0,\min\!\left\{\overline Q_h(s,a)-\Gamma_h(s,a),H-h+1\right\}\right\}$
    \State $\widehat\pi_h(s)\gets\argmax_{a\in\cA}Q_h(s,a)$
    \State $\widehat V_h(s)\gets Q_h(s,\widehat\pi_h(s))$
\EndFor
\State \Return $\widehat\pi=\{\widehat\pi_h\}_{h=1}^H$
\end{algorithmic}
\end{algorithm}

In particular, under our deterministic-reward setting, the dataset enters VPVI only through the visitation and transition counts $(C,E)$.

\subsection{UCB-Advantage}
\label{app:ucb-advantage}

Finally, we review UCB-Advantage~\citep{zhang2020almost}, a model-free online algorithm based on reference-advantage decomposition and stage-wise updates. For each $(s,a,h)$, let $n_h(s,a)$ denote its cumulative number of visits and $\check n_h(s,a)$ the number of visits in the current stage. The stage lengths satisfy
\[
e_1=H,\quad e_{j+1}=\left\lfloor\left(1+\frac{1}{H}\right)e_j\right\rfloor,
\]
and we write
\[
\mathcal L=\left\{\sum_{i=1}^j e_i:j=1,2,\ldots\right\}
\]
for the cumulative visit counts at which stage updates are triggered.

Besides $Q_h$ and $V_h$, UCB-Advantage maintains a reference value $V_h^{\mathrm{ref}}$. For a transition $(s,a,s')$ at step $h$, the current-stage accumulators are updated as
\begin{align*}
\check\mu_h(s,a)\overset{+}{=}V_{h+1}(s')-V_{h+1}^{\mathrm{ref}}(s'),\ 
\check\nu_h(s,a) \overset{+}{=}V_{h+1}(s'),\ 
\check\sigma_h(s,a)\overset{+}{=}\left(V_{h+1}(s')-V_{h+1}^{\mathrm{ref}}(s')\right)^2,
\end{align*}
while the global reference accumulators satisfy
\begin{align*}
\mu_h^{\mathrm{ref}}(s,a)
&\overset{+}{=}
V_{h+1}^{\mathrm{ref}}(s'),
&
\sigma_h^{\mathrm{ref}}(s,a)
&\overset{+}{=}
\left(V_{h+1}^{\mathrm{ref}}(s')\right)^2.
\end{align*}
The checked quantities are reset after each stage, whereas the reference accumulators use all samples collected so far.

At the end of a stage, writing $n=n_h(s,a)$ and $\check n=\check n_h(s,a)$, the algorithm constructs a Bernstein-type reference-advantage bonus
\begin{align*}
b_h(s,a)
={}&
c_1\sqrt{\frac{\sigma_h^{\mathrm{ref}}(s,a)/n-\left(\mu_h^{\mathrm{ref}}(s,a)/n\right)^2}{n}\,\iota}+
c_2\sqrt{\frac{\check\sigma_h(s,a)/\check n-\left(\check\mu_h(s,a)/\check n\right)^2}{\check n}\,\iota}
\\
&+c_3\left(\frac{H\iota}{n}+\frac{H\iota}{\check n}+\frac{H\iota^{3/4}}{n^{3/4}}+\frac{H\iota^{3/4}}{\check n^{3/4}}\right),
\end{align*}
together with the Hoeffding-type bonus
\[
\overline b_h(s,a)=2\sqrt{\frac{H^2\iota}{\check n}}.
\]
Here $\iota$ is a logarithmic confidence factor and $c_1,c_2,c_3>0$ are universal constants.

The $Q$-value is updated by taking the minimum of the standard stage-wise estimate, the reference-advantage estimate, and its previous value:
\[
\begin{split}
Q_h(s,a)\leftarrow\min\Bigg\{&r_h(s,a)+\frac{\check\nu_h(s,a)}{\check n}+\overline b_h(s,a),\\
&r_h(s,a)+\frac{\mu_h^{\mathrm{ref}}(s,a)}{n}+\frac{\check\mu_h(s,a)}{\check n}+b_h(s,a),\;Q_h(s,a)\Bigg\}.
\end{split}
\]
Afterward, $V_h(s)=\max_a Q_h(s,a)$ and the current-stage accumulators are reset. When the state-level visitation count reaches the threshold $N_0$, the reference value is refreshed as $V_h^{\mathrm{ref}}(s)\leftarrow V_h(s)$.

Two features are particularly relevant to Section~\ref{sec:examples}. First, between stage or reference-value updates, the functions appearing in the accumulated target and moment statistics remain fixed. Second, within a boundary episode, a step-$h$ update modifies only the step-$h$ value tables, whereas subsequent decisions use tables indexed by $h+1,h+2,\ldots$. Hence, the stage-wise statistics can be accumulated using the corresponding pre-boundary functions while preserving the update order of the original algorithm.
\begin{algorithm}[ht]
\caption{UCB-Advantage}
\label{alg:review-ucb-advantage}
\begin{algorithmic}[1]
\Require Number of episodes $K$, stage schedule $\mathcal L$, reference threshold $N_0$
\State Initialize all accumulators to zero
\State For $h\in[H]$, set $Q_h(s,a)\gets H-h+1$, $V_h(s)\gets H-h+1$, and $V_h^{\mathrm{ref}}(s)\gets H$
\State Set $V_{H+1}(s)\gets0$ and $V_{H+1}^{\mathrm{ref}}(s)\gets0$ for all $s\in\cS$
\For{$k=1,\ldots,K$}
    \State Observe $X_{k,1}$
    \For{$h=1,\ldots,H$}
        \State $A_{k,h}\gets\argmax_{a\in\cA}Q_h(X_{k,h},a)$
        \State Observe $X_{k,h+1}\sim P_h(\cdot\mid X_{k,h},A_{k,h})$
        \State Set $(s,a,s')\gets(X_{k,h},A_{k,h},X_{k,h+1})$
        \State $n_h(s,a)\gets n_h(s,a)+1$, $\check n_h(s,a)\gets\check n_h(s,a)+1$
        \State Update $\check\mu_h,\check\nu_h,\check\sigma_h,\mu_h^{\mathrm{ref}},\sigma_h^{\mathrm{ref}}$ using $(s,a,s')$
        \If{$n_h(s,a)\in\mathcal L$}
            \State Compute $b_h(s,a)$ and $\overline b_h(s,a)$
            \State Update $Q_h(s,a)$ using the stage-wise and reference-advantage estimates
            \State $V_h(s)\gets\max_{a'\in\cA}Q_h(s,a')$
            \State Reset $\check n_h(s,a),\check\mu_h(s,a),\check\nu_h(s,a),\check\sigma_h(s,a)$ to zero
        \EndIf
        \If{$\sum_{a'\in\cA}n_h(s,a')=N_0$}
            \State $V_h^{\mathrm{ref}}(s)\gets V_h(s)$
        \EndIf
    \EndFor
\EndFor
\end{algorithmic}
\end{algorithm}

\section{Further Applications of the Exact Fast-Simulation Framework}
\label{app:further-extensions}

The representative applications in Section~\ref{sec:examples} illustrate our framework in simulator-based, offline, and online RL. All learner-side arithmetic costs in this appendix are measured under the same unit-cost arithmetic convention as in Section~\ref{sec:complexity}. We next discuss two further extensions that highlight additional algorithmic structures compatible with fast Markov-flow simulation. The first considers a round-based single-agent specialization of FedQ-Bernstein~\citep{zheng2023federated}, showing that visit-dependent learning rates, trajectory-dependent moments, and Bernstein confidence bonuses can be implemented without learner-side computation scaling with the round size. The second extends the framework to federated RL, where multiple agents independently collect local experience between learning rounds.

\subsection{Round-Based Q-Learning with Visit-Dependent Learning Rates}
\label{app:visit-dependent-learning-rate}

The visit-dependent learning rate
\[
\alpha_t=\frac{H+1}{H+t}
\]
was introduced in model-free episodic Q-learning by~\citet{jin2018q} and is used by a broad class of subsequent Q-learning methods. A direct implementation updates the Q-value after individual visits and therefore appears to retain learner-side computation linear in the number of trajectories in a batch. We use the single-agent specialization of FedQ-Bernstein~\citep{zheng2023federated} to show that this dependence is not inherent. When $M=1$, the agent and server operations of the federated algorithm collapse into a single round-based learner: each round executes a fixed policy until an event-triggered visitation threshold is reached, accumulates first- and second-moment statistics, and performs round-end Q-updates, using single-visit updates only in the bounded low-count regime and aggregate updates thereafter. Importantly, the Bernstein bonus is also represented through aggregate endpoint quantities rather than a visit-indexed cumulative sum.

\subsubsection{Algorithm Review}

We specialize FedQ-Bernstein~\citep{zheng2023federated} to $M=1$. The learner proceeds in rounds indexed by $k$. At the beginning of round $k$, it maintains $Q^k$, $V^k$, cumulative visit counts $N_h^k(s,a)$, and historical first- and second-moment accumulators. The deterministic greedy policy
\[
\pi_h^k(s)\in\argmax_{a\in\cA}Q_h^k(s,a)
\]
and the value function $V^k$ remain fixed throughout the round.

For each $(s,a,h)$, define the round threshold
\begin{equation}
\label{eq:qbern-threshold}
c_h^k(s,a)=\max\left\{1,\left\lfloor\frac{N_h^k(s,a)}{H(H+1)}\right\rfloor\right\}.
\end{equation}
The learner repeatedly executes $\pi^k$ until, at the end of an episode, the within-round visitation count of some $(s,a,h)$ reaches $c_h^k(s,a)$. Let $C_h^k(s,a)$ denote the number of visits to $(s,a,h)$ during round $k$. For the next states observed after these visits, define
\[
S_{1,h}^k(s,a)=\sum_{i\in\mathcal I_h^k(s,a)}V_{h+1}^k(X_{i,h+1}),\qquad S_{2,h}^k(s,a)=\sum_{i\in\mathcal I_h^k(s,a)}\left(V_{h+1}^k(X_{i,h+1})\right)^2,
\]
where $\mathcal I_h^k(s,a)$ is the set of trajectories in round $k$ that visit $(s,a,h)$. When $C_h^k(s,a)>0$, write
\[
\overline V_{h+1}^k(s,a)=\frac{S_{1,h}^k(s,a)}{C_h^k(s,a)}.
\]

Following FedQ-Bernstein, the learner maintains cumulative first- and second-moment sums
\[
W_{1,k+1}(s,a,h)=W_{1,k}(s,a,h)+S_{2,h}^k(s,a),\qquad W_{2,k+1}(s,a,h)=W_{2,k}(s,a,h)+S_{1,h}^k(s,a),
\]
and forms the empirical variance
\[
W_{k+1}(s,a,h)=\frac{W_{1,k+1}(s,a,h)}{N_h^{k+1}(s,a)}-\left(\frac{W_{2,k+1}(s,a,h)}{N_h^{k+1}(s,a)}\right)^2
\]
whenever $N_h^{k+1}(s,a)>0$.

Let $\iota$ denote the confidence parameter and $c'>0$ the absolute constant used in FedQ-Bernstein. Specializing its Bernstein endpoint confidence quantity to $M=1$, for $N_h^k(s,a)>0$ define
\[\beta_h^{\mathrm{B},k}(s,a)=c'\min\left\{\sqrt{\frac{H\iota}{N_h^k(s,a)}
\left(W_k(s,a,h)+H\right)}+\frac{(\sqrt{H^7SA}+\sqrt{SAH^6})\iota}{N_h^k(s,a)},\sqrt{\frac{H^3\iota}{N_h^k(s,a)}}\right\},\]
and set $\beta_h^{\mathrm{B},k}(s,a)=0$ when $N_h^k(s,a)=0$. Define $\beta_h^{\mathrm{B},k+1}(s,a)$ analogously using the post-round count $N_h^{k+1}(s,a)$ and variance statistic $W_{k+1}(s,a,h)$. Thus, each endpoint quantity is a deterministic function of the corresponding aggregate visit count and empirical variance and can be evaluated using a constant number of scalar arithmetic operations once these quantities are available.

For integers $n_1\leq n_2$, let $\alpha^{\mathrm{c}}(n_1,n_2)=\prod_{t=n_1}^{n_2}(1-\alpha_t)$. If a tuple has $n=N_h^k(s,a)$ visits before round $k$ and receives $c=C_h^k(s,a)$ new visits, define
\begin{equation}
\label{eq:qbern-alpha-agg}
\alpha_{\mathrm{agg}}(n,c)=1-\alpha^{\mathrm{c}}(n+1,n+c).
\end{equation}
The corresponding aggregate Bernstein bonus is
\begin{equation}
\label{eq:qbern-beta-agg}
\widetilde\beta_h^k(s,a)=\beta_h^{\mathrm{B},k+1}(s,a)-\alpha^{\mathrm{c}}(n+1,n+c)\beta_h^{\mathrm{B},k}(s,a).
\end{equation}

As in the original algorithm, let $i_0=2H(H+1)$ and distinguish two regimes. If $n<i_0$,~\Cref{eq:qbern-threshold} gives $c_h^k(s,a)=1$, so a fixed $(s,a,h)$ can be visited at most once during the round. If such a visit occurs and its next state is $s'$, writing $t=n+1$, the corresponding single-visit bonus is
\[
\widetilde b_t(s,a,h)=\frac{\beta_h^{\mathrm{B},k+1}(s,a)-(1-\alpha_t)\beta_h^{\mathrm{B},k}(s,a)}{2\alpha_t},
\]
and the update is
\begin{equation}
\label{eq:qbern-low-update}
Q_h^{k+1}(s,a)=(1-\alpha_t)Q_h^k(s,a)+\alpha_t\left(r_h(s,a)+V_{h+1}^k(s')+\widetilde b_t(s,a,h)\right).
\end{equation}
If $n\geq i_0$, all new visits in the round receive equal weight and the aggregate update becomes
\begin{equation}
\label{eq:qbern-high-update}
Q_h^{k+1}(s,a)=(1-\alpha_{\mathrm{agg}})Q_h^k(s,a)+\alpha_{\mathrm{agg}}\left(r_h(s,a)+\overline V_{h+1}^k(s,a)\right)+\frac{\widetilde\beta_h^k(s,a)}{2}.
\end{equation}
Tuples not visited during the round retain their previous Q-values. Finally,
\[
V_h^{k+1}(s)=\min\left\{H,\max_{a\in\cA}Q_h^{k+1}(s,a)\right\},\qquad \pi_h^{k+1}(s)\in\argmax_{a\in\cA}Q_h^{k+1}(s,a).
\]

Algorithm~\ref{alg:review-qbern} summarizes this single-agent specialization. Since $M=1$, local trajectory collection and global aggregation are performed by the same learner, and no central server or communication operation is required.

\begin{algorithm}[ht]
\caption{Single-Agent Specialization of FedQ-Bernstein}
\label{alg:review-qbern}
\begin{algorithmic}[1]
\State Initialize $Q_h^1(s,a)=V_h^1(s)=H$, $N_h^1(s,a)=0$, and $W_{1,1}(s,a,h)=W_{2,1}(s,a,h)=0$
\State Set $V_{H+1}^k(s)=0$ for all $s\in\cS$ and all $k$
\For{learning rounds $k=1,2,\ldots$}
    \State Set $\pi_h^k(s)\in\argmax_{a\in\cA}Q_h^k(s,a)$
    \State Set $c_h^k(s,a)\gets\max\{1,\lfloor N_h^k(s,a)/(H(H+1))\rfloor\}$ for all $(s,a,h)$
    \State Initialize $C_h^k(s,a)$, $S_{1,h}^k(s,a)$, and $S_{2,h}^k(s,a)$ to zero
    \Repeat
        \State Generate one trajectory under the fixed policy $\pi^k$
        \State Update $C_h^k(s,a)$, $S_{1,h}^k(s,a)$, and $S_{2,h}^k(s,a)$ along the trajectory
    \Until{some $C_h^k(s,a)=c_h^k(s,a)$ at the end of an episode}
    \State Set $N_h^{k+1}(s,a)\gets N_h^k(s,a)+C_h^k(s,a)$ and update $W_{1,k+1}$, $W_{2,k+1}$, and $W_{k+1}$
    \State Set $Q^{k+1}\gets Q^k$
    \For{each $(s,a,h)$ with $C_h^k(s,a)>0$}
        \State Set $n\gets N_h^k(s,a)$ and $c\gets C_h^k(s,a)$
        \If{$n<i_0$}
            \State Apply the single-visit update in~\Cref{eq:qbern-low-update}
        \Else
            \State Compute $\alpha_{\mathrm{agg}}$ and $\widetilde\beta_h^k(s,a)$ from~\Cref{eq:qbern-alpha-agg,eq:qbern-beta-agg}
            \State Apply the aggregate update in~\Cref{eq:qbern-high-update}
        \EndIf
    \EndFor
    \State Set $V_h^{k+1}(s)\gets\min\{H,\max_aQ_h^{k+1}(s,a)\}$ and $\pi_h^{k+1}(s)\in\argmax_aQ_h^{k+1}(s,a)$
\EndFor
\end{algorithmic}
\end{algorithm}
\subsubsection{Application of Our Framework and Computational Cost}

We now apply adaptive Markov-flow simulation to the single-agent FedQ-Bernstein specialization. The round-based learner itself continues according to its original event-triggered rule. For the computational comparison, we consider the first $K$ episodes of this learning process. If the $K$-th episode occurs before the end of a communication round, we truncate the computational accounting at episode $K$ without applying a round-end update that has not yet been triggered.

Let $Z_k$ denote the complete learner state at the beginning of round $k$, including $\pi^k$, $V^k$, the cumulative visit counts, the historical moment accumulators, and the round thresholds. Let $G_j$ denote the Markov flow of the first $j$ trajectories generated during the round, and let $C_h^{(j)}(s,a)$ denote its state-action counts. The round terminates at the first prefix satisfying
\[
g_{\mathrm{QB}}(Z_k,G_j)=\ind\{\exists(s,a,h):C_h^{(j)}(s,a)\geq c_h^k(s,a)\}.
\]
Since the thresholds are determined before the round and each $C_h^{(j)}(s,a)$ is nondecreasing in $j$, this trigger is prefix monotone. Because $\pi^k$ is deterministic, only $C_h^{(j)}(s,\pi_h^k(s))$ can increase for each $(s,h)$, so each trigger evaluation costs $O(HS)$. If $\kappa$ episodes have already been accounted for, Algorithm~\ref{alg:fastforward} is called with candidate size $K-\kappa$. If no round boundary occurs before the remaining horizon, the algorithm returns the full remaining flow and the accounting stops at episode $K$ without a round-end update. Otherwise, adaptive binary refinement locates the exact terminating episode, with the trigger cost dominated by the $O(HS^2)$ deterministic-policy split cost.

All trajectory-dependent statistics required by a completed round update are flow measurable. If $E_h^k(s,a,s')$ denotes the transition counts of the recovered boundary flow, then
\[
S_{1,h}^k(s,a)=\sum_{s'\in\cS}E_h^k(s,a,s')V_{h+1}^k(s'),\qquad S_{2,h}^k(s,a)=\sum_{s'\in\cS}E_h^k(s,a,s')\left(V_{h+1}^k(s')\right)^2.
\]
Because $V^k$ is fixed throughout round $k$, the visitation counts, round means, cumulative first and second moments, and empirical variance statistics are deterministic functions of the pre-round learner state and the boundary flow.

The visit-dependent learning-rate coefficient can also be evaluated without processing the $c$ new visits individually. Since $1-\alpha_t=(t-1)/(H+t)$,
\[
\alpha^{\mathrm{c}}(n+1,n+c)=\prod_{t=n+1}^{n+c}\frac{t-1}{H+t}=\prod_{r=0}^{H}\frac{n+r}{n+c+r}.
\]
Thus,~\Cref{eq:qbern-alpha-agg} can be evaluated using only $H+1$ factors, independently of $c$.

The Bernstein bonus has the same aggregate structure. The boundary flow determines the post-round first and second moments and hence the endpoint quantity $\beta_h^{\mathrm{B},k+1}(s,a)$. \Cref{eq:qbern-beta-agg} then recovers the entire round bonus contribution from the pre-round and post-round endpoint quantities together with $\alpha^{\mathrm{c}}(n+1,n+c)$. No visit-indexed bonus sum needs to be evaluated.

The low-count regime is also compatible with the recovered flow. When $N_h^k(s,a)<i_0$,~\Cref{eq:qbern-threshold} gives $c_h^k(s,a)=1$. Hence, the round contains at most one new visit to a fixed $(s,a,h)$, and if such a visit occurs, its next state is uniquely determined by the transition-count vector $E_h^k(s,a,\cdot)$. The single-visit update in~\Cref{eq:qbern-low-update} can therefore be reproduced exactly. In the high-count regime,~\Cref{eq:qbern-high-update} depends only on the round mean, aggregate learning-rate coefficient, and endpoint Bernstein bonus, all of which are determined by the recovered flow.

All Q-, value-, and policy updates occur only after a round terminates, so the second requirement of the boundary compatibility condition holds automatically. By \Cref{thm:adaptive}, replacing conventional trajectory generation within completed rounds by \Cref{alg:fastforward}, together with the possible terminal truncation at episode $K$, preserves the joint distribution of the completed-round learner states and the aggregate statistics of the terminal partial round up to episode $K$.

We finally compare simulator-side and overall computational costs. Let $R_{\max}(K)$ be a deterministic upper bound on the number of communication-round batches intersecting the first $K$ episodes, with the possible terminal partial batch absorbed into the bound. Conventional episode-by-episode simulation has simulator-side cost $\Theta(HK)$. Maintaining the visit counts and first- and second-moment accumulators along the generated transitions also costs $O(HK)$. Because the within-round policy is deterministic, at most one action per $(s,h)$ receives new visits; the round-end aggregate coefficients therefore cost $O(H^2S)$, while recomputing values and the greedy policy costs $O(HSA)$. Thus, its overall computational cost is
\[
O\left(HK+R_{\max}(K)HS(H+A)\right).
\]

Under adaptive Markov-flow simulation, one forward flow sample and at most $O(\log K)$ conditional splits are required for each round batch intersecting the horizon. Therefore, the simulator-side cost is
\[
O\left(R_{\max}(K)C_{\mathrm{flow}}^{\mathrm{det}}\log K\right)
=
O\left(R_{\max}(K)HS^2\log K\right).
\]
Only the active action $\pi_h^k(s)$ can have nonzero within-round flow for a given $(s,h)$. Hence the first- and second-moment contractions require $O(HS^2)$ arithmetic per completed round, and the aggregate learning-rate coefficients require at most $O(H^2S)$ arithmetic per completed round. The variance, endpoint-bonus, and Q-updates over visited tuples cost at most $O(HS)$, while recomputing the value function and greedy policy costs $O(HSA)$. Thus, the overall cost is
\[
O\left(R_{\max}(K)HS^2\log K+R_{\max}(K)HS(H+A)\right).
\]

The communication-round bound of FedQ-Bernstein is the same as that of FedQ-Hoeffding~\citep{zheng2023federated}. Specializing this deterministic bound to $M=1$ gives
\[
R_{\max}(K)=O\left(H^3SA\log(1+HK)\right).
\]
Consequently, the simulator-side cost becomes
\[
O\left(H^4S^3A\log(1+HK)\log K\right),
\]
while the overall computational cost is
\[
O\left(H^4S^2A\log(1+HK)\left[S\log K+H+A\right]\right).
\]
Thus, conventional simulation has linear dependence on $K$ in both simulator-side and overall computation cost, whereas adaptive Markov-flow simulation reduces both to $O(\log^2K)$ dependence.

\subsection{Federated RL: FedQ-Advantage}
\label{app:fedq-advantage}

Federated RL provides a second extension of the framework. Between communication rounds, multiple agents independently interact with copies of the environment and subsequently transmit aggregate statistics to a central server. This structure is particularly compatible with Markov-flow simulation: if each agent uses a fixed policy during its local exploration batch and its communicated quantities depend only on aggregate transition statistics, the local trajectories need not be materialized individually.

We illustrate this through FedQ-Advantage~\citep{zheng2024federated}. The algorithm separates communication rounds from stage-wise policy updates and uses reference-advantage decomposition for variance reduction. During a communication round, each agent accumulates visitation counts and first- and second-moment statistics of value functions that remain fixed throughout that round. A stage for a fixed $(s,a,h)$ consists of consecutive complete communication rounds, so its statistics are obtained by accumulating the round-wise contributions evaluated with the value functions used when those transitions were collected.

\subsubsection{Algorithm Review}

FedQ-Advantage considers $M$ agents, each interacting with an independent copy of the same finite-horizon MDP. The learning process is divided into communication rounds indexed by $k$. At the beginning of round $k$, the server maintains the current Q-estimate $Q^k$, value function $V^k$, reference value $V^{\mathrm{ref},k}$, and the count and moment accumulators described below. It constructs the deterministic policy
\[
\pi_h^k(s)\in\argmax_{a\in\cA}Q_h^k(s,a)
\]
and broadcasts $\pi^k$, $V^k$, $V^{\mathrm{ref},k}$, and the relevant counters to the agents. Define
\[
V_h^{\mathrm{adv},k}=V_h^k-V_h^{\mathrm{ref},k}.
\]
These functions remain fixed throughout communication round $k$.

For each $(s,a,h)$, let $N_h^k(s,a)$ denote the cumulative number of visits in all completed stages before round $k$, let $n_h^k(s,a)$ denote the size of the previous completed stage, and let $\widetilde n_h^k(s,a)$ denote the number of visits already accumulated in the current stage before round $k$. Let $I_h^{\mathrm{ren},k}(s,a)$ indicate whether round $k$ is the first round of a new stage, with $I_h^{\mathrm{ren},1}(s,a)=1$. The communication threshold is
\begin{equation}
\label{eq:fedq-threshold}
c_h^k(s,a)=
\begin{cases}
\left\lfloor n_h^k(s,a)/(MH)\right\rfloor,
& n_h^k(s,a)>0\ \text{and}\ \widetilde n_h^k(s,a)>(1-1/H)n_h^k(s,a),
\\[1mm]
\max\left\{1,\left\lceil\bigl(n_h^k(s,a)-\widetilde n_h^k(s,a)\bigr)/M\right\rceil\right\},
&\text{otherwise}.
\end{cases}
\end{equation}
FedQ-Advantage allows both synchronization options. When $\mathrm{usyn}=\mathrm{FALSE}$, each agent $m$ independently executes $\pi^k$ until, at the end of an episode, its local visitation count for at least one $(s,a,h)$ reaches $c_h^k(s,a)$, and the server initiates aggregation after every agent has reached one such local threshold. When $\mathrm{usyn}=\mathrm{TRUE}$, we analyze the episode-synchronized specialization in which all agents advance under a common episode clock; the round terminates as soon as any agent reaches one of its local thresholds, at which point all agents transmit their aggregate statistics.

Let $C_{m,h}^k(s,a)$ denote agent $m$'s number of visits to $(s,a,h)$ during round $k$. For a fixed function $f$, define
\begin{equation}
\label{eq:fedq-local-sum}
A_{m,h}^k(s,a;f)=\sum_{i\in\mathcal I_{m,h}^k(s,a)}f(X_{m,i,h+1}).
\end{equation}
In addition to the visit counts, each agent transmits quantities of the form~\Cref{eq:fedq-local-sum} for
\begin{equation}
\label{eq:fedq-five-functions}
f\in\left\{V_{h+1}^{\mathrm{ref},k},V_{h+1}^{\mathrm{adv},k},V_{h+1}^{k},\left(V_{h+1}^{\mathrm{ref},k}\right)^2,\left(V_{h+1}^{\mathrm{adv},k}\right)^2\right\}.
\end{equation}

The server stores all-history reference accumulators $\mu_h^{\mathrm{ref},k},\sigma_h^{\mathrm{ref},k}$ and current-stage accumulators $\mu_h^{\mathrm{adv},k},\mu_h^{\mathrm{val},k},\sigma_h^{\mathrm{adv},k}$, initialized to zero. At the end of round $k$, these are updated pointwise in $(s,a,h)$ as
\begin{align}
(\mu,\sigma)_h^{\mathrm{ref},k+1}(s,a)
&=(\mu,\sigma)_h^{\mathrm{ref},k}(s,a)
+\sum_{m=1}^M\left(
A_{m,h}^k(s,a;V_{h+1}^{\mathrm{ref},k}),
A_{m,h}^k(s,a;(V_{h+1}^{\mathrm{ref},k})^2)
\right),
\label{eq:fedq-ref-moment-update}\\
(\mu^{\mathrm{adv}},\mu^{\mathrm{val}},&\sigma^{\mathrm{adv}})_h^{k+1}(s,a)
=(1-I_h^{\mathrm{ren},k}(s,a))
(\mu^{\mathrm{adv}},\mu^{\mathrm{val}},\sigma^{\mathrm{adv}})_h^{k}(s,a)\notag\\
&+\sum_{m=1}^M\left(
A_{m,h}^k(s,a;V_{h+1}^{\mathrm{adv},k}),
A_{m,h}^k(s,a;V_{h+1}^{k}),
A_{m,h}^k(s,a;(V_{h+1}^{\mathrm{adv},k})^2)
\right).
\label{eq:fedq-stage-moment-update}
\end{align}
Thus, when a stage spans multiple rounds, each round contributes the value or advantage function that was fixed in that round; no historical transition is reevaluated using a later value function.

After receiving the local messages, the server computes
\begin{equation}
\label{eq:fedq-stage-count}
\widehat n_h^{k+1}(s,a)=\widetilde n_h^k(s,a)+\sum_{m=1}^M C_{m,h}^k(s,a).
\end{equation}
A stage renewal occurs when
\begin{equation}
\label{eq:fedq-stage-renewal}
\widehat n_h^{k+1}(s,a)\geq\ind\{n_h^k(s,a)=0\}MH+\left(1+\frac1H\right)n_h^k(s,a).
\end{equation}
Equivalently, with $I_h^{\mathrm{ren},k+1}(s,a)$ denoting the renewal indicator,
\begin{equation}
\label{eq:fedq-count-recursion}
\bigl(n_h^{k+1},\widetilde n_h^{k+1},N_h^{k+1}\bigr)(s,a)=
\begin{cases}
\bigl(\widehat n_h^{k+1}(s,a),0,N_h^k(s,a)+\widehat n_h^{k+1}(s,a)\bigr),
& I_h^{\mathrm{ren},k+1}(s,a)=1,\\
\bigl(n_h^k(s,a),\widehat n_h^{k+1}(s,a),N_h^k(s,a)\bigr),
& I_h^{\mathrm{ren},k+1}(s,a)=0.
\end{cases}
\end{equation}

For a renewed stage, the empirical means entering the two Q-estimates are
\begin{equation*}
\overline V_{h,\mathrm{stage}}^{k+1}(s,a)
=\frac{\mu_h^{\mathrm{val},k+1}(s,a)}{n_h^{k+1}(s,a)},\
\overline V_{h,\mathrm{ref}}^{k+1}(s,a)
=\frac{\mu_h^{\mathrm{ref},k+1}(s,a)}{N_h^{k+1}(s,a)},
\end{equation*}
and
\begin{equation*}
\overline V_{h,\mathrm{adv}}^{k+1}(s,a)
=\frac{\mu_h^{\mathrm{adv},k+1}(s,a)}{n_h^{k+1}(s,a)}.
\end{equation*}
For confidence parameter $p\in(0,1)$, let $\iota=\log(2/p)$. Define
\begin{equation*}
\widehat{\operatorname{Var}}_{h,\mathrm{ref}}^{k+1}(s,a)=\frac{\sigma_h^{\mathrm{ref},k+1}(s,a)}{N_h^{k+1}(s,a)}-\left(\frac{\mu_h^{\mathrm{ref},k+1}(s,a)}{N_h^{k+1}(s,a)}\right)^2
\end{equation*}
and
\begin{equation*}
\widehat{\operatorname{Var}}_{h,\mathrm{adv}}^{k+1}(s,a)=\frac{\sigma_h^{\mathrm{adv},k+1}(s,a)}{n_h^{k+1}(s,a)}-\left(\frac{\mu_h^{\mathrm{adv},k+1}(s,a)}{n_h^{k+1}(s,a)}\right)^2.
\end{equation*}
The corresponding bonus terms are
\begin{align*}
b_h^{k+1,1}(s,a)&=\sqrt{\frac{2H^2\iota}{n_h^{k+1}(s,a)}},\\
b_h^{k+1,2}(s,a)&=2\sqrt{\frac{\widehat{\operatorname{Var}}_{h,\mathrm{ref}}^{k+1}(s,a)\iota}{N_h^{k+1}(s,a)}}+2\sqrt{\frac{\widehat{\operatorname{Var}}_{h,\mathrm{adv}}^{k+1}(s,a)\iota}{n_h^{k+1}(s,a)}} \\
&+10H\left[\left(\frac{\iota}{N_h^{k+1}(s,a)}\right)^{3/4}+\left(\frac{\iota}{n_h^{k+1}(s,a)}\right)^{3/4}+\frac{\iota}{N_h^{k+1}(s,a)}+\frac{\iota}{n_h^{k+1}(s,a)}\right].
\end{align*}
The resulting Hoeffding-type and reference-advantage estimates can be written as
\begin{align}
Q_h^{k+1,1}(s,a)
&=r_h(s,a)+\overline V_{h,\mathrm{stage}}^{k+1}(s,a)+b_h^{k+1,1}(s,a),
\label{eq:fedq-q1}\\
Q_h^{k+1,2}(s,a)
&=r_h(s,a)+\overline V_{h,\mathrm{ref}}^{k+1}(s,a)+\overline V_{h,\mathrm{adv}}^{k+1}(s,a)+b_h^{k+1,2}(s,a).
\label{eq:fedq-q2}
\end{align}
The Q-estimate is updated only when the stage renews:
\begin{equation}
\label{eq:fedq-q-update}
Q_h^{k+1}(s,a)=
\begin{cases}
\min\left\{Q_h^{k+1,1}(s,a),Q_h^{k+1,2}(s,a),Q_h^k(s,a)\right\},
&\text{if the stage renews},
\\
Q_h^k(s,a),
&\text{otherwise}.
\end{cases}
\end{equation}
The server then sets
\[
V_h^{k+1}(s)=\max_{a\in\cA}Q_h^{k+1}(s,a),\qquad \pi_h^{k+1}(s)\in\argmax_{a\in\cA}Q_h^{k+1}(s,a).
\]
The reference value at $(s,h)$ is refreshed at most once, using the updated $V_h^{k+1}(s)$ in the first round for which
\[
\sum_{a\in\cA}N_h^{k+1}(s,a)\geq N_0.
\]
Thus the reference threshold is based on visits from completed stages and does not include the unfinished-stage count $\widetilde n_h^{k+1}$.

Algorithm~\ref{alg:review-fedq-advantage} summarizes FedQ-Advantage under the synchronization conventions described above.
\begin{algorithm}[H]
\caption{FedQ-Advantage}
\label{alg:review-fedq-advantage}
\begin{algorithmic}[1]
\Require Number of agents $M$, synchronization flag $\mathrm{usyn}\in\{\mathrm{TRUE},\mathrm{FALSE}\}$, total interaction-budget target $MT$ with $T\geq H$, reference threshold $N_0$, confidence parameter $p$
\State Initialize $Q^1,V^1,V^{\mathrm{ref},1}$, $N^1,n^1,\widetilde n^1$, $I^{\mathrm{ren},1}$, and all moment accumulators
\State Set $U\gets0$ and $k\gets1$
\While{$U<MT$}
    \State Set $\pi_h^k(s)\in\argmax_{a\in\cA}Q_h^k(s,a)$
    \State Compute $c_h^k(s,a)$ from~\Cref{eq:fedq-threshold}
    \State Broadcast $\pi^k$, $V^k$, $V^{\mathrm{ref},k}$, and the relevant stage counters
    \ForAll{agents $m\in[M]$}
        \State Initialize local counts and the five sums in~\Cref{eq:fedq-five-functions} to zero, and set $\ell_m^k\gets0$
    \EndFor
    \If{$\mathrm{usyn}=\mathrm{FALSE}$}
        \ForAll{agents $m\in[M]$ in parallel}
            \Repeat
                \State Generate one trajectory under the fixed policy $\pi^k$ and set $\ell_m^k\gets\ell_m^k+1$
                \State Update $C_{m,h}^k(s,a)$ and the five local sums
            \Until{some $C_{m,h}^k(s,a)=c_h^k(s,a)$ at the end of an episode}
        \EndFor
    \Else
        \Repeat
            \ForAll{agents $m\in[M]$ in parallel}
                \State Generate one trajectory under the fixed policy $\pi^k$ and set $\ell_m^k\gets\ell_m^k+1$
                \State Update $C_{m,h}^k(s,a)$ and the five local sums
            \EndFor
        \Until{some agent reaches $C_{m,h}^k(s,a)=c_h^k(s,a)$ for some $(s,a,h)$}
    \EndIf
    \State $U\gets U+H\sum_{m=1}^M\ell_m^k$
    \State All agents send their local counts and aggregate sums to the server
    \State Update the global moment accumulators using~\Cref{eq:fedq-ref-moment-update,eq:fedq-stage-moment-update}
    \State Compute $\widehat n_h^{k+1}(s,a)$ using~\Cref{eq:fedq-stage-count}
    \For{each $(s,a,h)$}
        \If{\Cref{eq:fedq-stage-renewal} holds}
            \State Set $I_h^{\mathrm{ren},k+1}(s,a)\gets1$ and update $(n,\widetilde n,N)$ using~\Cref{eq:fedq-count-recursion}
            \State Compute $Q_h^{k+1,1},Q_h^{k+1,2}$ and update $Q_h^{k+1}(s,a)$ using~\Cref{eq:fedq-q1,eq:fedq-q2,eq:fedq-q-update}
        \Else
            \State Set $I_h^{\mathrm{ren},k+1}(s,a)\gets0$ and update $(n,\widetilde n,N)$ using~\Cref{eq:fedq-count-recursion}
            \State $Q_h^{k+1}(s,a)\gets Q_h^k(s,a)$
        \EndIf
    \EndFor
    \State Update $V^{k+1}$ and $\pi^{k+1}$
    \State Refresh $V_h^{\mathrm{ref}}(s)$ from $V_h^{k+1}(s)$ where $\sum_aN_h^{k+1}(s,a)\geq N_0$ is reached for the first time
    \State $k\gets k+1$
\EndWhile
\end{algorithmic}
\end{algorithm}

\subsubsection{Application of Our Framework and Computational Cost}
Let $Z_k$ denote the server state at the beginning of communication round $k$, including $Q^k,V^k,V^{\mathrm{ref},k}$, $N^k,n^k,\widetilde n^k$, the renewal indicators, and all historical moment accumulators. For agent $m$, let $G_{m,j}^k$ denote the Markov flow of its first $j$ trajectories in the round, with state-action and transition counts $C_{m,h}^{k,(j)}$ and $E_{m,h}^{k,(j)}$, respectively. Since $\pi^k$ is fixed and deterministic throughout the round, conditional on $Z_k$, the trajectories generated by each agent follow a fixed deterministic Markov policy.

The local statistics communicated to the server are flow measurable irrespective of the synchronization mechanism. For every $f$ in~\Cref{eq:fedq-five-functions},
\begin{equation}
\label{eq:fedq-flow-stat}
A_{m,h}^k(s,a;f)=\sum_{s'\in\cS}E_{m,h}^k(s,a,s')f(s').
\end{equation}
The corresponding local visitation count is simply $C_{m,h}^k(s,a)$. Historical contributions from earlier rounds are already stored in the accumulators contained in $Z_k$, while the current boundary flow contributes only the round-$k$ contractions in~\Cref{eq:fedq-flow-stat}. Therefore, stages spanning multiple communication rounds remain flow measurable even though the value functions may change between rounds. Once these local statistics have been recovered, the moment updates in~\Cref{eq:fedq-ref-moment-update,eq:fedq-stage-moment-update}, the count recursion in~\Cref{eq:fedq-count-recursion}, the Q-updates in~\Cref{eq:fedq-q1,eq:fedq-q2,eq:fedq-q-update}, and the subsequent value, policy, and reference updates are deterministic functions of the pre-round server state and the local aggregate messages. Moreover, all such updates occur only after local exploration terminates, so the second requirement of the boundary compatibility condition holds automatically.

For agent $m$, the local stopping condition is
\begin{equation}
\label{eq:fedq-trigger}
g_{m,k}(Z_k,G_{m,j}^k)=\ind\left\{\exists(s,a,h):C_{m,h}^{k,(j)}(s,a)\geq c_h^k(s,a)\right\}.
\end{equation}
The thresholds are determined by $Z_k$ before local exploration begins, and all local visitation counts are nondecreasing in $j$, so~\Cref{eq:fedq-trigger} is prefix monotone. At the beginning of round $k$, define
\begin{equation}
\label{eq:fedq-candidate-size}
L_k=1+\sum_{s\in\cS}\left(c_1^k(s,\pi_1^k(s))-1\right).
\end{equation}
This candidate size is guaranteed to contain the local stopping boundary. Indeed, after $L_k$ episodes,
\[
L_k=\sum_{s\in\cS}C_{m,1}^{k,(L_k)}(s,\pi_1^k(s)).
\]
If no local threshold had been reached by then, every term on the right would satisfy $$C_{m,1}^{k,(L_k)}(s,\pi_1^k(s))\leq c_1^k(s,\pi_1^k(s))-1,$$ which would imply $L_k\leq L_k-1$, a contradiction. Thus Algorithm~\ref{alg:fastforward} can be called with candidate size $L_k$ for every complete local round.

For the computational comparison, fix a deterministic budget parameter $T\geq H$ and terminate after the first complete communication round for which the total number of environment steps across all agents reaches $MT$. Thus $T$ specifies an average per-agent budget target; individual agents need not generate the same number of episodes. Every round that is started is completed according to the original local stopping rules. Let $U_k$ denote the total number of environment steps generated by all agents before round $k$. A new round is started only when $U_k<MT$.

The candidate size in~\Cref{eq:fedq-candidate-size} is uniformly controlled before termination. From~\Cref{eq:fedq-threshold}, for every state $s$,
\[
c_1^k(s,\pi_1^k(s))-1\leq \frac{n_1^k(s,\pi_1^k(s))}{M}.
\]
Moreover, $n_1^k(s,\pi_1^k(s))$ is no larger than the cumulative number of visits to the corresponding step-$1$ tuple before round $k$. Since every previous episode contributes exactly one step-$1$ visit,
\[
\sum_{s\in\cS}n_1^k(s,\pi_1^k(s))\leq \frac{U_k}{H}.
\]
Consequently,
\begin{equation}
\label{eq:fedq-candidate-budget-bound}
L_k\leq 1+\frac{U_k}{MH}<1+\frac{T}{H}.
\end{equation}
Thus each agent generates at most $L_k$ episodes in a communication round, under either synchronization mechanism. If $\Delta U_k$ denotes the total number of environment steps generated in round $k$, then
\[
\Delta U_k\leq MH L_k<MT+MH.
\]
Let $U$ be the total number of environment steps when the procedure terminates. The final round starts with $U_k<MT$ and ends with $U\geq MT$, so
\begin{equation}
\label{eq:fedq-total-budget-bound}
MT\leq U<2MT+MH.
\end{equation}
Define
\begin{equation}
\label{eq:fedq-barT}
\overline T=2T+H.
\end{equation}
Then $U<M\overline T$, and $T\geq H$ implies $\overline T\leq3T$. Therefore, conventional trajectory-by-trajectory simulation has simulator-side cost $\Theta(U)=\Theta(MT)$. Maintaining the constant number of local visitation and moment statistics along the generated transitions also costs $O(MT)$. At the end of each communication round, aggregating the local statistics across the $M$ agents costs $O(MHS)$, while checking stage renewals and performing the corresponding Q-, value-, policy-, and reference-function updates costs $O(HSA)$.

\textbf{Without forced synchronization.}
Suppose $\mathrm{usyn}=\mathrm{FALSE}$. Each agent terminates its local exploration independently after the first episode at which one of its local visitation counts reaches the prescribed threshold. Because $\pi^k$ is deterministic, only the active action $\pi_h^k(s)$ can acquire additional local visits for each $(s,h)$, so each local trigger evaluation costs $O(HS)$. Conditional on $Z_k$, the $M$ agents interact with independent copies of the MDP, so Algorithm~\ref{alg:fastforward} can be applied independently to each agent to recover its exact local boundary flow. Together with the flow-measurability argument above, this shows inductively over complete communication rounds that replacing conventional local trajectory generation by adaptive Markov-flow simulation preserves the joint distribution of the federated learning process under the aggregate stopping rule above.

By~\Cref{eq:fedq-candidate-budget-bound}, locating a local stopping boundary requires at most $O(\log(1+T/H))$ conditional refinements. Since the local policy is deterministic, each forward flow sample and conditional split costs $C_{\mathrm{flow}}^{\mathrm{det}}=O(HS^2)$, while the $O(HS)$ local trigger evaluation is dominated. Applying the deterministic communication-round bound of~\citet{zheng2024federated} with the per-agent-scale parameter $\overline T$, define
\begin{equation}
\label{eq:fedq-round-bound}
R_{\max}(\overline T)=SAH^2+4SAH^2(\log H+3)\log\left(1+\frac{\overline T}{SAH^3}\right).
\end{equation}
Since the realized total interaction satisfies $U<M\overline T$, this bounds the number of communication rounds before termination. The simulator-side cost is therefore
\begin{equation}
\label{eq:fedq-fast-cost}
O\left(MR_{\max}(\overline T)HS^2\log\left(1+\frac{T}{H}\right)\right).
\end{equation}
The local contractions in~\Cref{eq:fedq-flow-stat} require $O(MHS^2)$ arithmetic per communication round, while server-side aggregation and the subsequent Q-, value-, policy-, and reference-function updates require $O(HS(M+A))$ additional computation per round. Hence, the overall computational cost is
\begin{equation*}
O\left(MR_{\max}(\overline T)HS^2\log\left(1+\frac{T}{H}\right)+R_{\max}(\overline T)HS(M+A)\right).
\end{equation*}
For comparison, conventional trajectory-by-trajectory simulation has simulator-side cost $\Theta(MT)$ and overall computational cost
\[
O\left(MT+R_{\max}(\overline T)HS(M+A)\right).
\]
Substituting~\Cref{eq:fedq-round-bound} into~\Cref{eq:fedq-fast-cost} and using $\overline T\leq3T$, the simulator-side cost is
\[
O\left(MH^3S^3A\left[1+(\log H+3)\log\left(1+\frac{T}{SAH^3}\right)\right]\log\left(1+\frac{T}{H}\right)\right).
\]
The corresponding overall computational cost is
\[
O\left(H^3S^2A\left[1+(\log H+3)\log\left(1+\frac{T}{SAH^3}\right)\right]\left[MS\log\left(1+\frac{T}{H}\right)+M+A\right]\right).
\]

\textbf{With forced synchronization.}
We next consider $\mathrm{usyn}=\mathrm{TRUE}$ under a synchronized episode clock. In this case, a communication round terminates at the first episode for which any agent reaches one of its local thresholds. The joint stopping condition can therefore be written as
\[
g_k^{\mathrm{sync}}\left(Z_k,G_{1,j}^k,\ldots,G_{M,j}^k\right)
=\ind\left\{\exists m,s,a,h:C_{m,h}^{k,(j)}(s,a)\geq c_h^k(s,a)\right\}.
\]
This trigger is again prefix monotone. Since the common policy is deterministic, evaluating it over the $M$ agents costs $O(MHS)$ per candidate prefix by checking only the active action at each $(s,h)$. The same $L_k$ in~\Cref{eq:fedq-candidate-size} guarantees that the joint stopping boundary occurs no later than episode $L_k$. Adaptive binary refinement can therefore be performed jointly across the $M$ local flows. At each candidate prefix, the $M$ flows are conditionally split at the same episode index; conditional independence of the agents implies that these splits recover the correct joint conditional distribution.

Define the corresponding deterministic communication-round upper bound
\begin{equation}
\label{eq:fedq-round-bound-sync}
R_{\mathrm{sync},\max}(\overline T)=MSAH^2+4MSAH^2(\log H+3)\log\left(1+\frac{\overline T}{SAH^3}\right).
\end{equation}
Each joint refinement costs $O(MC_{\mathrm{flow}}^{\mathrm{det}})=O(MHS^2)$, which dominates the $O(MHS)$ joint trigger evaluation. By~\Cref{eq:fedq-candidate-budget-bound}, at most $O(\log(1+T/H))$ refinements are needed per round. Thus, the simulator-side cost is
\begin{equation*}
O\left(MR_{\mathrm{sync},\max}(\overline T)HS^2\log\left(1+\frac{T}{H}\right)\right),
\end{equation*}
and the overall computational cost is
\begin{equation*}
O\left(MR_{\mathrm{sync},\max}(\overline T)HS^2\log\left(1+\frac{T}{H}\right)+R_{\mathrm{sync},\max}(\overline T)HS(M+A)\right).
\end{equation*}
Substituting~\Cref{eq:fedq-round-bound-sync} and using $\overline T\leq3T$, the simulator-side cost is
\[
O\left(M^2H^3S^3A\left[1+(\log H+3)\log\left(1+\frac{T}{SAH^3}\right)\right]\log\left(1+\frac{T}{H}\right)\right),
\]
while the overall computational cost is
\[
O\left(MH^3S^2A\left[1+(\log H+3)\log\left(1+\frac{T}{SAH^3}\right)\right]\left[MS\log\left(1+\frac{T}{H}\right)+M+A\right]\right).
\]
For comparison, conventional trajectory-by-trajectory simulation again has simulator-side cost $\Theta(MT)$ and overall cost $O(MT+R_{\mathrm{sync},\max}(\overline T)HS(M+A))$. Thus, under the no-forced-synchronization mechanism and the episode-synchronized specialization of forced synchronization, adaptive Markov-flow simulation replaces the linear simulator-side dependence on the average per-agent budget target with polylogarithmic dependence while preserving the corresponding complete-round update semantics.

\subsection{Implications for Other Federated Q-Learning Algorithms}
The two extensions above identify complementary structures that substantially broaden the applicability of our framework. The single-agent specialization of FedQ-Bernstein shows that visit-dependent learning rates, round-wise value moments, and Bernstein confidence bonuses can all be represented through aggregate round statistics and endpoint quantities, while FedQ-Advantage shows that adaptive Markov-flow simulation extends naturally to distributed local exploration. Together, these observations suggest a general route for extending fast simulation to a broader class of federated RL algorithms that organize local interaction into communication rounds and base global updates on flow-measurable aggregate statistics.
\end{document}

%% file: math_commands.tex
\usepackage{amsmath,amsfonts,bm}

\def\eqref#1{equation~\ref{#1}}

\def\1{\bm{1}}

\DeclareMathAlphabet{\mathsfit}{\encodingdefault}{\sfdefault}{m}{sl}
\SetMathAlphabet{\mathsfit}{bold}{\encodingdefault}{\sfdefault}{bx}{n}

\DeclareMathOperator*{\argmax}{arg\,max}